\documentclass[acmsmall]{acmart} 

\usepackage[T1]{fontenc}
\usepackage{bussproofs}

\usepackage{amsmath,amssymb}
\usepackage{mathpartir}
\usepackage{xcolor}
\usepackage{xspace}
\usepackage{xparse}
\usepackage{enumitem}
\usepackage{subcaption}
\usepackage{hyperref}
\usepackage{amsthm}
\usepackage{thmtools}
\usepackage{url}
\usepackage{multirow}
\usepackage{mathtools}    
\usepackage{algorithm}
\usepackage{algpseudocode}

\usepackage{wrapfig}
\usepackage[font=small]{caption}
\usepackage{microtype}
\usepackage{tikz}
\definecolor{softgreen}{RGB}{0,130,0} 

\usepackage{listings}

\definecolor{dafnypurple}{RGB}{136,0,136}
\definecolor{dafnykeyword}{RGB}{153,102,0}
\definecolor{dafnycomment}{RGB}{0,128,0}
\definecolor{dafnynumber}{RGB}{0,102,204}
\definecolor{dafnygreen}{RGB}{0,150,0}
\definecolor{posHighlight}{RGB}{0,110,0}
\definecolor{highlightred}{RGB}{255,0,0}
\definecolor{bordergray}{RGB}{180,180,180}

\EnableBpAbbreviations
\usepackage{graphicx}

\newcommand{\te}[1]{\text{#1}}
\newcommand{\isType}[2]{\te{isType}(#1)(#2)}
\newcommand{\Tool}{SEVerA\xspace}
\definecolor{funcpurple}{RGB}{75,0,130}
\newcommand{\funcIt}[1]{\textcolor{funcpurple}{#1}}

\newcommand{\com}[1]{\textcolor{red}{#1}}

\newcommand{\secRef}{\S}
\newcommand{\alp}{\Sigma}
\newcommand{\alpN}[1]{\Sigma^{#1}}
\newcommand{\llm}[1]{\mathcal{L}_{#1}}
\newcommand{\llmP}[1]{\mathcal{L}^{p}_{#1}}
\newcommand{\prob}[1]{\mathcal{P}({#1})}
\newcommand{\decode}[1]{\mathcal{D}_{#1}}

\newcommand{\vect}[1]{\pmb{#1}}

\newcommand{\dg}{G_D}
\newcommand{\axiom}{\mathcal{A}}
\newcommand{\set}[1]{\{#1\}}
\newcommand{\msf}[1]{\mathsf{#1}}
\newcommand{\mrm}[1]{\mathrm{#1}}

\newcommand{\mcal}[1]{\mathcal{#1}}
\newcommand{\typs}[1]{\tau_{#1}}
\newcommand{\bin}[1]{\mcal{B}_{#1}}
\newcommand{\pre}[1]{{#1}^{\Phi}}
\newcommand{\post}[1]{{#1}^{\Psi}}
\newcommand{\fset}[1]{\mcal{F}_{#1}}
\newcommand{\param}[1]{\theta_{#1}}
\newcommand{\paramset}{\Theta}

\newcommand{\eos}{\langle eos \rangle}
\newcommand{\concat}{\cdot}
\newcommand{\gllm}{FGGM\xspace}
\newcommand{\lang}[1]{L(#1)}
\newcommand{\data}{\mrm{D}}
\newcommand{\reward}{\msf{R}}
\newcommand{\fspace}[2]{\msf{S}(#1, #2)}
\newcommand{\opt}[1]{{#1}^{*}}
\newcommand{\inspec}[1]{\Phi_{#1}}
\newcommand{\outspec}[1]{\Psi_{#1}}
\newcommand{\linspec}[1]{\Phi_{l}}
\newcommand{\planner}{\llm{p}}
\newcommand{\loutspec}[1]{\Psi_{l}}
\newcommand{\prompt}{f_{p}}

\newcommand{\fallback}{f_{d}}
\newcommand{\context}{\msf{C}}
\newcommand{\axiomEncode}[1]{\varphi_{#1}}
\newcommand{\synthesizer}{\msf{S}}

\NewDocumentCommand{\ver}{m m O{\context{}}}{
\mathcal{V}_{#3}[#1, #2]
}

\NewDocumentCommand{\ter}{O{\context{}}}{
\mathcal{T}_{#1}
}

\NewDocumentCommand{\parser}{O{\context{}}}{
\mathcal{P}_{#1}
}
\newcommand{\eval}{\funcIt{check_{ \loutspec{}}}}
\newcommand{\evalb}{\funcIt{check_{ \outspec{}}}}
\newcommand{\evalR}[1]{\funcIt{check_{\axiom, \inspec{},#1}}}
\newcommand{\fggm}{\msf{G}}
\newcommand{\fullParamSet}{\set{\paramset{}}}
\newcommand{\program}{\msf{P}_{\fullParamSet{}}}
\newcommand{\programSub}[1]{\msf{P}_{\{(\theta_1, \dots, \theta_k)/\Theta\}}}
\newcommand{\programOptSub}{\msf{P}_{\{(\opt{\theta_1}, \dots, \opt{\theta_k})/\Theta\}}}
\newcommand{\fggmSet}{\mathcal{G}}
\newcommand{\conformLoss}{\msf{R}_{\linspec{}, \loutspec{}}}
\newcommand{\conformLossParam}[1]{\msf{R}^{#1}_{\linspec{}, \loutspec{}}}
\newcommand{\fpost}{f_{post}}
\newcommand{\ferr}{\funcIt{Err_{\loutspec{}}}}

\usepackage{tcolorbox}
\usepackage{colortbl}

\definecolor{lightgrayrow}{RGB}{235,235,235}
\definecolor{violred}{RGB}{200,30,30}
\definecolor{gaingreen}{RGB}{0,130,0}
\newcommand{\redn}[1]{\textcolor{violred}{#1}}
\newcommand{\greenn}[1]{\textcolor{gaingreen}{#1}}
\newcommand{\dt}[1]{\makebox[0pt][l]{\scriptsize #1}}
\definecolor{lightgraybg}{RGB}{245,245,245}

\definecolor{bracketblue}{RGB}{0,92,170}
\definecolor{bracketred}{RGB}{180,30,30}
\definecolor{kwcolor}{RGB}{0, 92, 170}

\newcommand{\kw}[1]{\textcolor{kwcolor}{\textbf{#1}}}

\newcommand{\typeSig}{\tau}
\newcommand{\typeSet}{\msf{T}}
\newcommand*\circled[1]{\tikz[baseline=(char.base)]{
            \node[shape=circle,fill,inner sep=0.5pt] (char) {\textcolor{white}{#1}};}}

\newcommand{\typeit}[1]{\textcolor{blue}{#1}}

\DeclareMathOperator*{\argmax}{arg\,max}

\newenvironment{proofSketch}
{\begin{proof}[Proof Sketch]}
{\end{proof}}

\newcommand{\insertfigure}[4][0.97\linewidth]{%
    \begin{figure}[t]
        \centering
        \includegraphics[width=#1]{#2}
        \caption{#3}
        \label{#4}
    \end{figure}
}

\definecolor{mygray}{RGB}{195,195,195}

\newcommand\todoFrame[2]{\vspace{.3cm}\noindent\tikz{
\node (contentnode) [draw, color = #1!25, fill=#1!15, text=black, rectangle, outer sep = 0, rounded corners = 1mm, minimum width=\linewidth-1, text width=\linewidth, align=justify, below right] at (0,0) {\noindent #2};
\draw[fill opacity = 1, color=#1, fill=#1] (0,0) rectangle ([xshift=5]contentnode.south west);}
\par}

\newcommand\calloutbox[2]{%
    \par\todoFrame{#1}{%
        \noindent\hspace*{.3cm}%
        \begin{minipage}{\dimexpr\linewidth-.3cm\relax}%
        #2%
        \end{minipage}%
    }\vspace{5pt}
}

\title{\Tool: Verified Self-Evolving Agents with Specification Guidance}

\author{Debangshu Banerjee}
\authornote{These authors contributed equally to this work.}
\email{db21@illinois.edu}
\affiliation{%
  \institution{University of Illinois Urbana-Champaign}
  \country{USA}
}
    
    \author{Changming Xu}
    \authornotemark[1]
    \email{cx23@illinois.edu}
    \affiliation{%
      \institution{University of Illinois Urbana-Champaign}
      \country{USA}
    }
    
    \author{Eugene Ie}
    \email{eugeneie@google.com}
    \affiliation{%
      \institution{Google}
      \country{USA}
    }
    
    \author{Ming Zhang}
    \email{mingzhang@google.com}
    \affiliation{%
      \institution{Google}
      \country{USA}
    }
    
    \author{Daiyi Peng}
    \email{daiyip@google.com}
    \affiliation{%
      \institution{Google}
      \country{USA}
    }
    
    \author{Chu-Cheng Lin}
    \email{kitsing@google.com}
    \affiliation{%
      \institution{Google}
      \country{USA}
    }

    \author{Tarun Suresh}
    \affiliation{%
      \institution{Stanford University}
      \country{USA}
    }
    
    \author{Gagandeep Singh}
    \email{ggnds@illinois.edu}
    \affiliation{%
      \institution{University of Illinois Urbana-Champaign}
      \country{USA}
    }



\begin{document}
\begin{abstract}
Recent advances have demonstrated the effectiveness of self-evolving LLM agents on tasks such as program repair and scientific discovery. In this paradigm, a planner LLM continuously synthesizes agent programs that invoke parametric components, including generative models such as LLMs, and fine-tunes them.
However, existing work has shown that self-evolving agent frameworks can exhibit undesirable behavior on unseen inputs. Because these synthesized programs are often executed autonomously, the lack of formal guarantees that hold for all inputs raises serious reliability and security concerns. To address this gap, we introduce \emph{specification-driven self-evolving agent synthesis}, which combines formal behavioral specifications that the synthesized agents must satisfy on all inputs with best-effort soft objectives that capture task utility. 
We propose Formally Guarded Generative Models (\gllm{}), which allow the planner LLM to bind each generative-model call to a formal output contract. Each \gllm call automatically wraps the underlying parametric generative model within a stateful rejection sampler and a verified fallback, treating model outputs as samples from a proposal distribution. Consequently, every returned output satisfies the specified contract for any input and any parameter setting of the underlying model. Building on \gllm, we present \Tool (\emph{S}elf-\emph{E}volving \textit{Ver}ified \emph{A}gents), a four-stage framework for specification-driven agent synthesis. In the \emph{Ideation} stage, the planner LLM proposes task-specific generative components by constructing FGGMs. In the \emph{Search} stage, the planner LLM synthesizes candidate parametric programs containing multiple \gllm components. In the \emph{Verification} stage, we prove correctness with respect to the hard constraints for all parameter values, thereby reducing the remaining optimization problem to unconstrained learning. Finally, in the \emph{Learning} stage, we apply scalable gradient-based optimization, including GRPO-style fine-tuning for LLMs, to optimize the soft objective while preserving formal correctness. We evaluate \Tool on invariant generation for Dafny programs, policy-compliant agentic tool use ($\tau^2$-bench), constrained symbolic regression, and symbolic mathematical expression synthesis. Across all tasks, \Tool achieves zero constraint violations while simultaneously improving task performance over both unconstrained and state-of-the-art baselines. Our results show that behavioral specifications both guarantee correctness and steer synthesis toward higher-quality agents by pruning bad candidate programs.
\end{abstract}

\maketitle
\section{Introduction}

Recent works have leveraged the code generation capabilities of large language models (LLMs) to automatically synthesize LLM-based agents, represented as programs, from natural language task descriptions \cite{autoAgent1, autoAgent2, autoAgent3, autoAgent4, autoAgent5}. Given a user-defined task, a planner LLM generates a program that invokes parametric models, including LLMs, smaller neural networks, and external tools such as SMT solvers, and executes this program on user inputs to produce outputs. While this paradigm is powerful, it introduces significant safety and reliability concerns. Synthesized programs that achieve impressive performance on training examples repeatedly exhibit undesirable behavior on unseen inputs. Modern self-evolving frameworks \cite{selfEvolve1, selfEvolve2, selfEvolve3, selfEvolve4} further complicate formal safety verification by fine-tuning the parameters of models embedded within the synthesized programs to improve task-specific performance. Consequently, any correctness guarantee established for an agentic program must continue to hold after these parameter updates.

\noindent The absence of formal correctness guarantees has led to failures across a wide range of domains. In program verification, agents cheat by subtly modifying the input program, for example by altering variable initializations, so that the modified program together with the proposed annotations passes verification, thereby inflating task accuracy \cite{dafnyPro}. In code repair, agents delete failing tests instead of fixing the underlying bugs \cite{impossiblebench}. In agentic tool use, unconstrained agents violate domain-specific policies, such as refund eligibility and booking modification rules, in $68$--$76$\% of interactions \cite{agentc}. 
These failures are not isolated incidents. Rather, they are a consequence of evaluating and evolving agents solely using soft performance metrics on a fixed, often small, set of examples. Natural language task descriptions and evaluation on fixed inputs are insufficient to prevent such failures. Instead, we introduce \emph{specification-driven self-evolving agents}, in which every synthesized agentic program must satisfy \emph{formal hard constraints} specified as Hoare-style preconditions and postconditions for all valid inputs. Our goal is to combine the flexibility and performance benefits of self-evolving agentic frameworks while ensuring verified compliance with user-provided specifications on \emph{all} valid inputs.

\noindent \textbf{Key Challenges:} We summarize the challenges below.

\begin{itemize}[leftmargin=7pt]
\item \textbf{Balancing Compliance and Performance:} Any agent synthesis algorithm must ensure formal compliance with specifications while simultaneously improving task-specific performance. Existing work on deductive program synthesis \cite{sygus, semGus, cegis} provides formal guarantees but typically does not optimize task-specific performance objectives. Meanwhile, popular gradient-based optimization methods for LLMs, such as GRPO \cite{grpo}, and self-evolving agents \cite{selfEvolve1, selfEvolve2, selfEvolve3, selfEvolve4} empirically improve output quality but provide no assurance that outputs satisfy specified constraints.
\item \textbf{Binding Model Calls to Formal Contracts:} Generative models (GMs), although powerful, are inherently unreliable. Additionally, generative models are used at different locations within an agentic program, with each invocation serving a different goal.
For example, one call may require generating Python programs, while another may require producing symbolic mathematical expressions conforming to a specific formal grammar. Thus, we require a mechanism to specify and enforce a formal output contract for each \textit{individual} GM call in an agentic program. These contracts can then be utilized to verify that the agentic program satisfies the user-provided specification. Recent work on constrained decoding \cite{syncode} can enforce context-free grammar–based syntactic constraints and limited semantic constraints such as variable naming or simple type checks. However, these works are restricted in the types of constraints they support and require modifying the model’s decoding procedure, which limits their applicability to open-source models. Moreover, constrained decoding strategies can reduce model expressivity and distort the output distribution, producing outputs that satisfy constraints but degrade task performance \cite{grammarAligned, crane}.
\end{itemize}

\noindent\textbf{Our Contributions:} We summarize our contributions below.
\begin{itemize}[leftmargin=7pt]
\item We introduce Formally Guarded Generative Models (FGGM), which allow us to specify and enforce local contracts on model calls using Hoare-style preconditions and postconditions defined by trusted, user-provided oracles. FGGM operates solely on LLM output strings and is compatible with both open-source and closed-source models. The FGGM framework provides formal guarantees on model outputs while also enabling gradient-based parameter optimization to tune the underlying model to conform to the specified contracts.

\item We introduce \Tool: the first self-evolving agent synthesis algorithm with verifiable guarantees. The approach consists of four stages: (a) \emph{Ideation}, where a planner LLM proposes task-specific generative components through FGGM setup; (b) \emph{Search} and (c) \emph{Verify}, which together instantiate \emph{component-based synthesis}, in which the planner LLM samples parametric programs in a verifier-aware language such as Dafny that invoke the proposed components and the language's built-in verifier checks that they satisfy the user-specified specification over all parameter values of the embedded generative-model calls; and (d) \emph{Learn}, where verified parametric programs reduce the constrained learning problem to an unconstrained optimization problem over model parameters, enabling scalable gradient-based training without sacrificing correctness.
\item We prove that \Tool is \emph{sound}: any synthesized agent satisfies the behavioral specification for all inputs and parameter values (Theorem~\ref{thm:soundness}). We further show that, under sufficient conditions, \Tool can construct specification-compliant agents that match or exceed the utility of the underlying GMs, and that FGGMs preserve the expressivity of transformer-based models.
\item We evaluate \Tool on tasks spanning LLM-assisted program verification, symbolic math synthesis, agentic tool use, and constrained symbolic regression. Across all benchmarks, \Tool achieves provably zero constraint violations while outperforming state-of-the-art agents in task performance: 97.0\% verification rate on HumanEvalDafny (vs.\ 86.9\% for the best baseline), 66.0\% accuracy on GSM-Symbolic (vs.\ 44.7\% for the best constrained-decoding method), and 52.6\% pass rate on $\tau^2$-bench's airline domain using Qwen3-8B. Notably, on $\tau^2$-bench airline \Tool even beats Agent-C \cite{agentc} with Claude Sonnet 4.5, a state-of-the-art constrained agent using a frontier model. These results show that behavioral constraints do not merely enforce safety but actively prune the search space of candidate programs and steer synthesis toward higher-quality agents.
\end{itemize}

\section{Background}
\label{sec:background}
\noindent\textbf{Notations and Terminology: }
We use $\alp$ to denote the alphabet, a finite set of characters. All LLMs output finite strings over $\alp$. We use lowercase letters $(x)$ to denote constants, bold letters $(\pmb{x})$ to denote strings, capital letters $(X)$ to denote sets, and $|\vect{x}|$ to denote string length. In the rest of the paper, we treat neural networks as a restricted form of generative model that produces a single output for a fixed input.


\noindent \textbf{LLMs: }
LLMs $\llm{} : \alpN{*} \to \alpN{*}$ are probabilistic string generators that, given a prompt $\vect{p} \in \alpN{*}$ over the alphabet $\alp$, sample an output string $\vect{y} \in \alpN{*}$. Typically, the output length $|\vect{y}|$ is bounded by a fixed user-specified $n$, i.e., $|\vect{y}| \leq n$ or equivalently $\vect{y} \in \alpN{\leq n}$.
Fixed-length generation $\llm{n} : \alpN{*} \to \alpN{\leq n}$ is implemented by iteratively composing two high-level steps, repeated at most $n$ times. Each iteration samples a single character (sometimes referred to as a token) from $\alp$ or the special end-of-sequence symbol $\eos \notin \alp{}$ that marks termination of generation.
The first step is the \textbf{distribution prediction step}, defined as $\llmP{} : \alpN{*} \to \prob{\alp \cup \eos}$, which maps a prefix string $\vect{x} \in \alpN{*}$ to a probability distribution over $\alp \cup \eos$. This distribution captures how likely each character is to appear next. The second step is the \textbf{decoding step}, defined as $\decode{} : \alpN{*} \times \prob{\alp \cup \eos} \to \alp \cup \eos$, which draws the next character using the current prefix $\vect{x}$ and the predicted distribution. If the decoding step produces $\eos$, the generation terminates. Otherwise, the sampled character $c$ is appended to the prefix, and generation continues. Henceforth, we drop the subscript $n$ for simplicity.
\vspace{-0.5cm}
\begin{definition}[LLMs]
For a fixed $n \in \mathbb{N}$, an LLM is a bounded string generator $\llm{n} : \alpN{*} \to \alpN{\leq n}$ that samples an output string $\vect{y} \in \alpN{\leq n}$ for input $\vect{x}$. $\vect{y}$ on $\vect{x}$ is computed recursively as $\vect{y}_0 = \vect{x}$, and $\vect{y}_i = \vect{y}_{i-1} \concat c_i$ if ($\forall j \leq i. c_j \neq\eos)$ else $\vect{y}_{i-1}$ where $i \in [n]$, $c_i = \decode{}(\vect{y}_{i-1}, \llmP{}(\vect{y}_{i-1}))$  and $\vect{x}\concat\vect{y} = \vect{y}_n$.
\end{definition}
\noindent\textbf{Rejection Sampling: }
Rejection sampling is a popular Monte Carlo method for generating samples from a \emph{target distribution} $\pi_t$ when direct sampling from $\pi_t$ is difficult. Instead, it relies on a \emph{proposal distribution} $\pi_p$ from which sampling is easy. The support set $S(\pi_p)$ (see Definition~\ref{def:suppSet}) must cover $S(\pi_t)$, i.e., $S(\pi_t) \subseteq S(\pi_p)$.
Intuitively, $\pi_p$ generates candidate samples, and the rejection sampler decides whether to accept each candidate, rejecting those that are unlikely under $\pi_t$ (details in Appendix~\ref{appen:rejectionSampling}). The rejection sampler guarantees that no accepted sample lies outside $S(\pi_t)$.

\begin{definition}
\label{def:suppSet}
Let $\pi$ be a probability distribution over $\Omega$ with density function $D_{\pi} : \Omega \to \mathbb{R}^+$. The support set $S(\pi)$ of $\pi$ captures all points with positive probability $S(\pi) = \set{x\in \Omega;| D_{\pi}(x) > 0}$.
\end{definition}

\noindent\textbf{Self-evolving LLM agents: }In this work, we focus on LLM agents modeled as programs $f: T_{i} \to T_{o}$ that, on any user input $x \in T_{i}$, compute the output $y = f(x)$. Typically, $f$ is written in a popular imperative language such as Python. The program invokes one or more parametric models (formally defined in \secRef~\ref{sec:searchSpace}), such as LLMs, as well as other tools such as an SMT solver, to compute the output $y$. 
All parametric models $\fset{p}$ and tool calls $\fset{c}$ are provided as a collection of pre-implemented library components. We denote this library by $\fset{} = \fset{p} \cup \fset{c}$ (formal definition in \secRef~\ref{sec:searchSpace}). In the self-evolving paradigm, the program $f$ is not handwritten. Instead, it is generated by a planner LLM based on task-specific instructions provided as a prompt. The high-level goal is to synthesize a program $f$ that improves on-task performance. Given a fixed training dataset of input and optional ground-truth output pairs $\data \subseteq T_i \times T_o$ and a black-box reward oracle $\reward : D \times T_o \to [0, 1]$, the synthesized agent maximizes the aggregated reward $\tfrac{1}{|\data|}\sum_{(x_i, y_i) \in \data} \reward(x_i, y_i, f(x_i))$. For certain tasks, the ground truth $y_i$ may not be available; in such cases, the reward is denoted as $\reward(x_i, \_, f(x_i))$.
Assume that a context-free grammar (CFG) $G$ defines the set of 
syntactically valid programs $f$ as $\lang{G} \subseteq \alpN{*}$. 
Given the library functions $\fset{}$, let $\fspace{G}{\fset{}}$ denote the search space for $f: T_i \to T_o$ with type signature $(T_i \to T_o)$, including all parametric values of the models in $\fset{p}$. Self-evolving agentic frameworks aim to find the optimal solution $\opt{f}$ to the following optimization problem. 
\begin{small}
\begin{equation}
\opt{f} = \argmax_{f \in \fspace{G}{\fset{}}} \tfrac{1}{|\data|} \times {\textstyle \sum_{(x_i, \_) \in \data} \reward(x_i, \_, f(x_i))} \label{eq:unconstrinedLearning}
\end{equation}    
\end{small}

\section{Specification-driven agent synthesis}
 Eq.~\ref{eq:unconstrinedLearning} searches only for the optimal program $\opt{f}$ that maximizes the reward on the training set $\data$. However, it provides no guarantee about how $\opt{f}$ performs on unseen user inputs outside $\data$. In contrast, deductive program synthesis \cite{sygus, cegis} allows program generation to be controlled by formal behavioral specifications. In agent synthesis, we allow users to specify Hoare-style pre and postconditions $\Phi: T_i \to \{T, F\}$, $\Psi: T_i \times T_o \to \set{T, F}$, which any synthesized agent $f: T_i \to T_o$ must comply with on all valid inputs satisfying $\Phi$. Both $(\inspec{}, \outspec{})$ are defined with logical operators over trusted user-provided checkers that are part of the non-parametric library component $\fset{c}$ (see examples below).
 With $(\inspec{}, \outspec{})$, the unconstrained learning problem in Eq.~\ref{eq:unconstrinedLearning} can be reformulated as follows

\begin{small}
\begin{equation}
\opt{f} =
\argmax_{f \in \fspace{G}{\fset{}}}
\overbrace{\tfrac{1}{|\data|}\times {\textstyle \sum_{(x_i, \_) \in \data}}\;\; \reward(x_i, \_, f(x_i))}^{\text{soft learning objective}}
\;\; \text{s.t.} \;\;
\overbrace{\forall x \in T_i.\; \inspec{}(x) \implies \outspec{}(x, f(x))}^{\text{hard formal constraints}}
\label{eq:constrinedLearning}
\end{equation}
\end{small}
\noindent 
\noindent Although our synthesis problem shares similarities with the classic deductive component-based synthesis problem~\cite{componentSynthesis}, there are two key differences. In component-based synthesis, the goal is to construct programs that satisfy a behavioral specification by composing components from a \textit{fixed} user-provided set. In contrast, agent synthesis can propose and fine-tune task-specific generative components, thereby continually expanding the available component library. Moreover, the synthesis algorithm must jointly optimize task reward and specification compliance. As our experiments illustrate, many programs that satisfy the specification can achieve very low reward, making them safe but ineffective for the target task. Next, we illustrate how practical constraints from several application domains can be encoded as behavioral specifications $(\inspec{}, \outspec{})$.

\noindent \textbf{Examples of specifications $(\Phi, \Psi)$: }We show encodings of $(\inspec{}, \outspec{})$ for four different tasks from distinct domains.
 \circled{a} \textbf{Program Verification:} In this task, given a program in a verification-aware language such as Dafny with a predefined specification, the agent is expected to synthesize annotations (e.g., loop invariants, assertions, ranking functions) \cite{dafnyBench, clover, laurel}. Task success depends on whether the annotated program verifies. Recent work \cite{dafnyPro} shows that agents using frontier models (e.g., Claude) may cheat by subtly modifying the input program, such as altering variable initializations, instead of producing correct annotations. They further report that such modifications were not detected by prior string-based checkers, leading to inflated task accuracy. In contrast, formal AST-based diff checkers that syntactically enforce equivalence between the input program and the annotated output serve as a behavioral specification, ensuring the agent cannot cheat even on unseen inputs. \circled{b}
\textbf{Constrained LLM Generation: }Constrained generation captures a class of problems in which LLMs are required to generate strings that follow a formal structure. Typically, these formal structures are defined using formal grammars (regular \cite{dingo} or context-free grammars \cite{syncode, grammarAligned}), or static analyzers such as type checkers \cite{typeConstrained, chopChop}. In our setting, these formal structures define the output specification $\outspec{}$. We consider the task of generating symbolic math expressions from natural language questions, where hard constraints ensure that the synthesized agent always produces at least a structurally valid expression. 
\circled{c} \textbf{Agentic Tool Use: } Conversational LLM agents deployed in customer-service settings must select and invoke API tools to resolve user requests while respecting domain-specific policies, such as refund eligibility, booking modification rules, and authentication-before-access requirements \cite{barres2025tau2}. These temporal and logical policy constraints can be expressed as formal specifications that govern the ordering and content of agent actions.
Agent-C \cite{agentc} provides a domain-specific language for specifying such temporal properties and exposes an incremental validity checker for tool calls. Given the current tool-call execution trace as input, the checker determines whether the proposed sequence of tool calls complies with the specified formal policy. In our formulation, $\outspec{}$ uses the Agent-C checker to define a hard constraint on each tool call, ensuring that synthesized agents never violate domain policies at any step, regardless of the user input.
\circled{d} \textbf{Scientific Discovery:} We consider constrained symbolic regression, where the task is to recover an unknown mathematical formula from observed data. The formal constraints $(\inspec{}, \outspec{})$ encode prior knowledge about the target formula, such as known symbolic bounds \cite{symReg1, shapeConstrained,constantSymReg} or asymptotic behavior \cite{asympoticSymReg}. Any generated formula that violates these hard constraints is invalid, regardless of its empirical fit to the data. Moreover, in the presence of noise, the constraints $(\inspec{}, \outspec{})$ help prevent overfitting, which cannot be addressed solely by optimizing a soft objective \cite{symRegImprove}.
Overall, the hard behavioral specification captures the minimal requirements that must hold on \emph{all} inputs, while the reward function $\reward$ specifies a best-effort task performance metric.

\section{Overview}
\label{sec:overview}
\insertfigure[\linewidth]{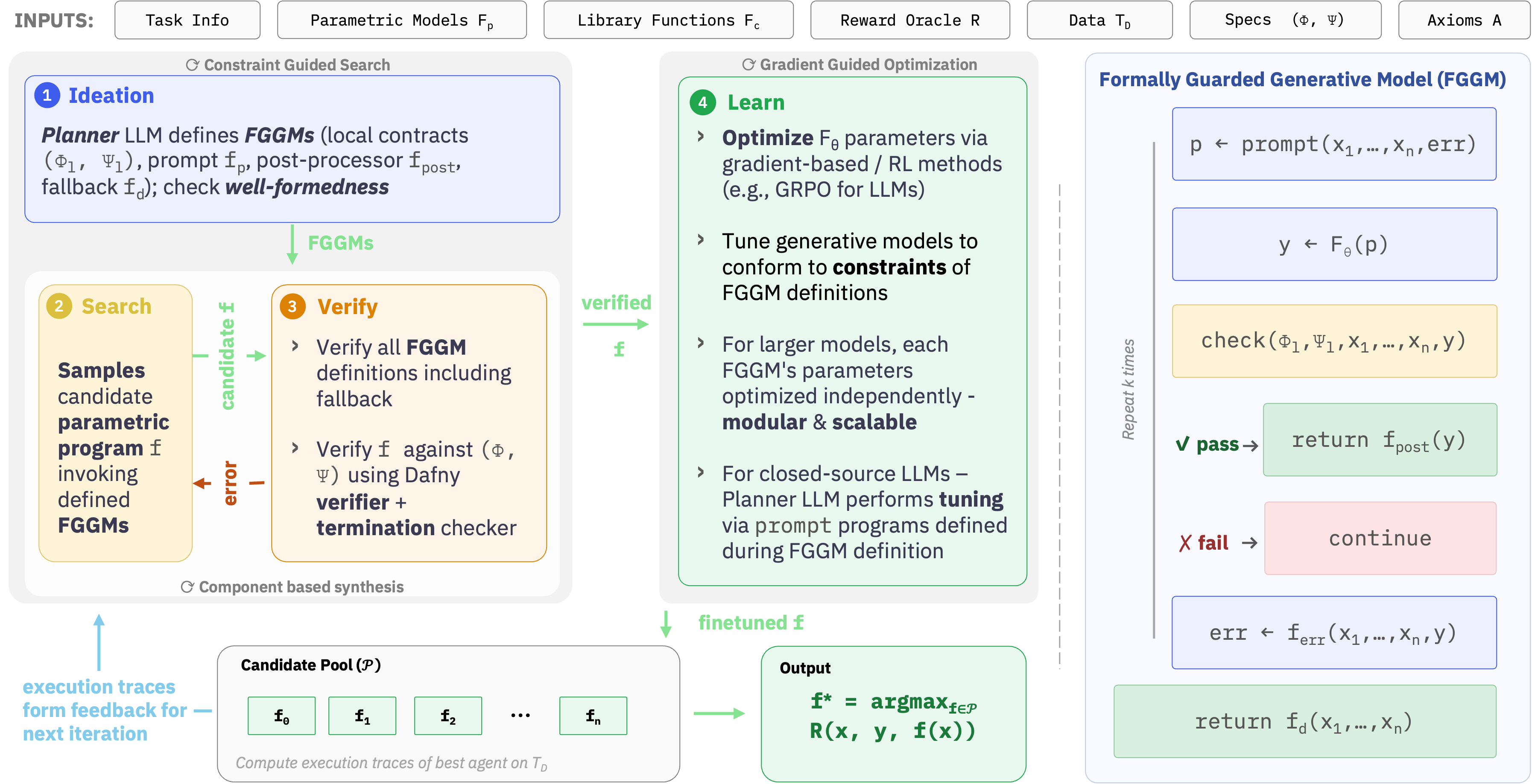}{Overview of \Tool, which operates in four steps. (1) \textbf{Ideation:} given the task information, library functions $\fset{c}$, and parametric generative models (GMs) $\fset{p}$, the planner LLM proposes task-specific FGGMs, each binding a GM call to local contracts $(\linspec{}, \loutspec{})$, a prompting program $\prompt$, a verified fallback $\fallback$, and expands the library components. Steps (2) and (3) form \textbf{component-based synthesis}: (2) \textbf{Search:} the planner emits a parametric agentic program over the expanded library, including parametric FGGMs; (3) \textbf{Verify:} the verifier checks all FGGM definitions and the program against $(\Phi, \Psi)$, accepting valid programs and otherwise returning an error to the planner. (4) \textbf{Learn:} the GM parameters within each FGGM are tuned to improve conformance with the local contracts and to increase reward $\reward{}$ over $\data$. \Tool maintains a pool of verified fine-tuned agents, reuses their traces to seed new candidates via step (1), and returns the highest-$\reward{}$ agent.}{fig:overView}

\noindent Fig.~\ref{fig:overView} gives a high-level overview of \Tool for solving the constrained learning problem in Eq.~\ref{eq:constrinedLearning} over the search space $\fspace{G}{\fset{}}$. Unlike classical component-based synthesis, the presence of parametric components $\fset{p}$ (e.g., LLMs) means that the search space contains infinitely many components, one for each underlying parameter value. Additionally, the reward oracle $\reward{}$ makes finding a specification-compliant program alone insufficient. Instead, the optimal solution requires searching over all specification-compliant candidates and picking the one with the highest reward.

\noindent \Tool systematically breaks down the search procedure into four tractable high-level steps, combining the benefits of both constraint-guided program search from deductive synthesis and unconstrained gradient-based parameter optimization. The first, \emph{Ideation}, step extracts a fixed number of task-specific generative components $\fggmSet$ using models from $\fset{p}$, thereby expanding the set $\fset{c}$. The \emph{Search} and \emph{Verify} steps then perform deductive component-based synthesis to produce parametric programs over $\fggmSet \cup \fset{c}$ that satisfy the specification $(\inspec{}, \outspec{})$ over all parameter values. Finally, the gradient-based optimization step fine-tunes the parametric models to jointly maximize cumulative reward while complying with the formal contracts defined in the corresponding FGGM.

\noindent The proposed Formally Guarded Generative Model (FGGM) serves as the critical link between these stages. During the ideation step, it binds each unreliable GM call to explicit local contracts and a verified non-parametric fallback program, ensuring that these contracts hold regardless of the underlying GM parameters. The synthesis step then leverages these local contracts to prove the correctness of the parametric program with respect to the behavioral specification. Beyond providing correctness guarantees, these contracts also extract local learning objectives by characterizing the expected outputs of each GM call. This enables the final optimization step to improve each GM's conformance to its local contracts, connecting constrained program search with unconstrained learning. 
\noindent Next, we discuss the key steps of \Tool using two representative and easy-to-understand problem instances from scientific discovery and program verification.

\subsection{Example Instantiation of \Tool}
\subsubsection{Constrained Symbolic Regression:}
\noindent In symbolic regression, each problem instance consists of possibly noisy observations of input–output pairs generated by an unknown ground-truth function $f_{gt}$. The goal is to recover this ground-truth function from the training data $\data$, a popular task in scientific discovery and neuro-symbolic program synthesis \cite{llmsr, dreamCoder}. In constrained symbolic regression \cite{symReg1, shapeConstrained}, users encode known properties of $f_{gt}$ beyond what is captured in $\data$ as formal constraints, ensuring that the recovered function $f$ satisfies these constraints. This setup directly follows the constrained learning formulation in Eq.~\ref{eq:constrinedLearning}, where the synthesized agent $\opt{f}$ corresponds to the recovered function.
For illustration, we consider the problem instance with ground-truth function $f_{gt}(x) = \sqrt{1.23 \times \max(x, 0.0)}$, taken from \cite{constantSymReg}.
The dataset $\data$ contains input–output pairs $(x, y) \in \mathbb{R}^2$, where $y = f_{gt}(x) + \epsilon$ is a noisy observation of $f_{gt}(x)$ and $\epsilon \sim \mathcal{N}(0, \sigma)$ denotes additive Gaussian noise. The pointwise reward function $\reward{}$ is defined as the negative of the commonly used Normalized Mean Squared Error (NMSE) (see Eq.~\ref{eq:lossSymReg}).
For this instance, the agent’s type signature $\typeSig$ is $(\mathbb{R} \to \mathbb{R})$, and the behavioral specifications $(\inspec{}, \outspec{})$ define known symbolic bounds on $f_{gt}$ (see Eq.~\ref{eq:constrainedSymReg}) where $\funcIt{sqrt}: \mathbb{R} \to \mathbb{R}$ is a provided library function.  

\begin{small}
    \begin{equation}
        \reward{}(x_i, y_i, f(x_i)) = -\tfrac{(f(x_i) - y_i)^2}{C}\;\;\text{where ${\textstyle C = \sum_{(x_i, y_i) \in \data} y_i^2}$} \label{eq:lossSymReg}
    \end{equation}
    \begin{equation}
        \inspec{}(x):\;(x\geq 0),\; \outspec{}(x, f(x)):\; \left((x \leq 1) \implies (f(x) \geq \funcIt{pow}(x, 0.8))\right) \wedge \left((x \geq 1) \implies (f(x) \geq \funcIt{sqrt}(x))\right)\label{eq:constrainedSymReg}    
    \end{equation}
\end{small}

\noindent\textbf{Search Space: } \Tool defines the search space using a restricted subset of Dafny supporting four basic types $\typeSet = \set{\typeit{int}, \typeit{bool}, \typeit{real},  \typeit{str}}$, arrays of basic types $arr\langle\typeSet\rangle$, conditional blocks, while loops, and function calls. We formalize the search space in \secRef~\ref{sec:searchSpace}. 

\noindent\textbf{User-provided components: }For this task, the library function set $\fset{}$ includes common transcendental functions (e.g., $\funcIt{sin}$, $\funcIt{cos}$, $\funcIt{pow}$, $\funcIt{sqrt}$) as well as small parametric neural networks (NNs), enabling the search space to capture a wide range of neuro-symbolic programs. Users provide type signatures and formal contracts for all non-parametric library functions; the planner LLM defines formal contracts for each parametric model call through the FGGM setup (\secRef~\ref{sec:fggm}). We present representative type signatures and specifications in Eq.~\ref{eq:libSymReg}, with the complete set in Appendix~\ref{appen:symRegLib}.

\begin{small}
    \begin{equation}
        \funcIt{sqrt}(x: \typeit{real}) \to \typeit{real},\; \Phi_{sqrt}(x): (x \geq 0),\; \Psi_{sqrt}(x, \funcIt{sqrt}(x)): (\funcIt{sqrt}(x) \geq 0);\;\; {\textstyle \funcIt{NN^{(1)}_\theta}(x_1: \typeit{real})} \to \typeit{real} \label{eq:libSymReg}
    \end{equation}
    \begin{equation}
        \forall x, d_1, d_2 \in \mathbb{R}. (x \geq 1) \wedge (d_1 \geq d_2)\implies \funcIt{pow}(x, d_1) \geq \funcIt{pow}(x, d_2)\label{eq:axiomSymReg1}
    \end{equation}
    \begin{equation}
        \forall x, d_1, d_2 \in \mathbb{R}. (x \leq 1) \wedge (d_1 \geq d_2)\implies \funcIt{pow}(x, d_1) \leq \funcIt{pow}(x, d_2)  \label{eq:axiomSymReg}
    \end{equation}
    \begin{equation}
        \;\axiomEncode{\axiom} \implies \forall x \in T_i.\big(\inspec{}(x) \implies \outspec{}(x, f(x))\big) \;\;\text{where $\axiomEncode{\axiom{}} = ({\textstyle \bigwedge_{\varphi \in \axiom}\;} \varphi)$ encodes $\axiom{}$}\label{eq:axiomImplication}
    \end{equation}
\end{small}
\noindent \textbf{Axioms:} The axioms $\axiom$ encode known properties of library components in $\fset{}$, including their formal contracts as universally quantified first-order sentences, where library components are treated as uninterpreted. For example, Eq.~\ref{eq:axiomSymReg} captures a property of $\funcIt{pow}$. These axioms assist automated verification as shown in Eq.~\ref{eq:axiomImplication}. We provide the formal syntax for defining $\axiom$ in \secRef~\ref{sec:searchSpace} and the complete axiom set for this task in Appendix~\ref{appen:symRegAxiom}.

\noindent Like other component-based synthesis frameworks \cite{componentSynthesis}, \Tool assumes that all library components have correct contracts, and that they are pure and terminating. Verifying the correctness of the library components and axioms is outside the scope of \Tool.

\subsubsection{LLM-Assisted Automated Verification} This task aims to synthesize annotations, such as inductive loop invariants, ranking functions, and assertions, that enable automatic verification and termination checking of input Dafny programs with pre-encoded input–output specifications \cite{clover, dafnyBench}.
The agent has type signature $(p: \typeit{str}) \to \typeit{str}$, mapping an input Dafny program with pre-encoded specifications, represented as a string, to an annotated Dafny program.  
The point-wise reward function $\reward{}$ evaluates whether the annotated program verifies within a pre-fixed time budget, assigning $0$ to unverified programs and $1$ to verified ones (Eq.~\ref{eq:lossDafny}). Hard constraints $(\inspec{}, \outspec{})$ ensure that, for any parsable input program, the output string remains parsable and does not change the input program beyond adding annotations (Eq.~\ref{eq:specsDafny}).
We use the Dafny built-in verifier to define $\reward{}$, and the Dafny parser together with the AST-based formal diff checker from \cite{dafnyPro} to define $(\inspec{}, \outspec{})$. $\funcIt{noDiff}$ returns True if the second input program is syntactically equivalent to the first (ignoring annotations), and False otherwise. The library function set $\fset{}$ includes all these components.
Although parsing and diff checks could be incorporated into $\reward{}{}$ by penalizing invalid outputs on $\data$, these constraints must hold for \emph{all} possible inputs, not only those in $\data$. Therefore, any candidate agent that fails to satisfy them must be filtered out.  

\noindent\textbf{Library Components: }We use the same search space as constrained symbolic regression. $\fset{}$ includes the verifier with a pre-fixed time budget $\funcIt{verify}$, the Dafny parser $\funcIt{parse}$, the diff-checker $\funcIt{noDiff}$. 
We show the type signature and formal contract of representative functions in Eq.~\ref{eq:libFuncDafny} while providing the rest in Appendix~\ref{appen:dafnyBenchLib}.  
\begin{small}
\begin{align}
&\funcIt{noDiff}(\vect{x}: \typeit{str}, \vect{y}: \typeit{str}) \to \typeit{bool},\; \Phi(\vect{x}, \vect{y}): \funcIt{parse}(\vect{x}),\; \Psi(\vect{x}, \vect{y}, \funcIt{noDiff}(\vect{x}, \vect{y})): T \label{eq:libFuncDafny}\\
& \reward{}(\vect{x}, \_, \vect{y}) = \mathbb{I}(\funcIt{parse}(\vect{y}) \wedge \funcIt{noDiff}(\vect{x}, \vect{y}) \wedge \funcIt{verify}(\vect{y})) \label{eq:lossDafny}\\
& \inspec{}(\vect{x}): \funcIt{parse}(\vect{x});\; \outspec{}(\vect{x}, f(\vect{x})): \funcIt{parse}(f(\vect{x})) \wedge \funcIt{noDiff}(\vect{x}, f(\vect{x})) \label{eq:specsDafny}
\end{align}    
\end{small}
\noindent\textbf{Axioms: }An example of a first-order sentence encoding known properties of $\funcIt{noDiff}$ in Eq.~\ref{eq:axiomDafny}.  $\funcIt{noDiff}$ is both reflexive and transitive w.r.t.\ all parsable Dafny programs. It is easy to see that the same program will always have no difference with itself. We list all the axioms in Appendix~\ref{appen:danfyBenchAxiom}.
\begin{small}
\begin{align}
&\textbf{Reflexivity:}\;\; \forall \vect{x}\in \alpN{*}.\; \funcIt{parse}(\vect{x}) \implies \funcIt{noDiff}(\vect{x}, \vect{x})
\label{eq:axiomDafny}
\end{align}    
\end{small}
\subsection{Formally Guarded Generative Model}
\label{sec:fggm}
The Formally Guarded Generative Model (FGGM) framework allows the planner LLM to bind each unreliable generative model (GM) call to Hoare-style pre- and post-conditions $(\linspec{}, \loutspec{})$, referred to as the local contract of that call. Once an FGGM is verified to be well-formed, its local contracts can be used to prove the correctness of the agentic program with respect to the behavioral specifications $(\inspec{}, \outspec{})$. Our FGGM design ensures the following essential properties.
\circled{1} \emph{Flexibility:} Unlike functions in $\fset{c}$, output contracts cannot be predefined by users because the expected output of a GM depends on the generation goal and therefore varies across calls. In our design, the planner LLM synthesizes custom quantifier-free formulas over terms involving non-parametric functions (e.g., \funcIt{sqrt}) and predicates (e.g., $\funcIt{noDiff}$) from $\fset{c}$ to specify these local contracts.
\circled{2} \emph{Parameter-Independent Correctness:} Provided the defined FGGM is well-formed, the specified contracts hold irrespective of the parameters of the underlying GM. This property is crucial to ensure that gradient-based parameter optimization never breaks the correctness guarantees.
\circled{3} \emph{Local Learning Objective:} Although correctness holds irrespective of parameter values, the local contracts $(\linspec{}, \loutspec{})$ extract the local learning objective for a particular GM call by characterizing its expected output. This objective can guide gradient-based optimization so that, after tuning, the GM’s conformance with $(\linspec{}, \loutspec{})$ improves (\secRef~\ref{subsec:conformanceTune}).
\subsubsection{Construction} The primary building block of each FGGM is a stateful rejection sampler defined by $(\linspec{}, \loutspec{})$ (see Fig.~\ref{fig:rejectionSampler}). The rejection sampler treats the output of the underlying GM as the proposal distribution $\pi_{p}$, while $\loutspec{}$ defines the support set of the target distribution $\pi_{t}$ (Definition~\ref{def:suppSet}). It repeatedly draws samples from the GM and checks their satisfiability ($\eval{}$) with respect to $\loutspec{}$, rejecting any sample that fails the check. For each rejected sample, it collects the evaluations of all atomic formulas and uses the rejected sample together with these evaluations $(\ferr{})$ as feedback for the next stateful sample generation. Since $\loutspec{}$ is quantifier-free and the user-provided components are computable, the checker can evaluate all atomic formulas and is therefore both sound and complete. This setup automatically simulates the widely successful feedback-guided generation paradigm \cite{selfEvolve1, cegis} for any local contract while requiring only output-string access. Benefits of the stateful rejection sampler over the stateless version are shown in \secRef~\ref{sec:stateful} (Table~\ref{tab:stateful}). 

\noindent For each FGGM with local contracts $(\linspec{}, \loutspec{})$ and underlying parametric GM $\llm{\Theta}$, the planner LLM synthesizes two key non-parametric programs using only functions from $\fset{c}$. The first is a fallback program $\fallback$ that satisfies $(\linspec{}, \loutspec{})$ and guards against edge cases where the GM hallucinates or produces specification violating outputs. Without a verified fallback, the rejection sampler could exhaust all samples; the fallback ensures that the contract holds unconditionally for all parameter values. Although $\fallback$ and FGGM share the same Hoare triplet their output on individual concrete points differ. Typically, $\fallback$ lacks the GM's capability, resulting in lower reward, but it prevents the agent from entering unsafe program states during execution by guaranteeing constraint satisfaction. By combining the fallback $\fallback$ with the GM, FGGM retains the correctness guarantees of $\fallback$ (Theorem~\ref{thm:FGGMsoundness}) and the reward gains from the GM (Theorem~\ref{thm:rewardLowerBound}). The second is a prompting program $\prompt$, which constructs the input to the GM. For LLMs, $\prompt$ encodes task-specific natural-language instructions, analogous to prompt-tuning \cite{dspy, lmql}, known to improve GM performance and instruction adherence.
\begin{wrapfigure}{r}{0.4\textwidth}
\vspace{0.4em}
\begin{minipage}{0.4\textwidth}

\begin{tcolorbox}[
colback=lightgraybg,
colframe=lightgraybg,
boxrule=0pt,
arc=2pt,
left=4pt,right=4pt,top=4pt,bottom=4pt
]
\small
\[
\begin{aligned}
&\kw{function}\ \funcIt{id}
(x_1:\typeit{T_1},\dots,x_n:\typeit{T_n})
: (\typeit{T_o}) \\[-0.25em]
&\quad \kw{requires}\ \linspec{}(x_1,\dots,x_n)\ \{ \\[-0.25em]  
&\quad \kw{var}\ p: \typeit{T_{in}}, y: \typeit{T_o}, err: \typeit{str} := ``"; \\[-0.35em]
&\quad \left.
\begin{aligned}
& \kw{for}\ i = 1 \dots K\ \;\{ \\
&\quad p := \funcIt{\prompt}(x_1, \dots, x_n, err); \\
&\quad y := \funcIt{\fpost}(\funcIt{\llm{\Theta}}(p));
\\
&\quad \kw{if}\ 
\funcIt{\eval}(x_1,\dots,x_n,y)\\
&\quad\qquad \kw{return}\ y; \\
&\quad err := \funcIt{\ferr}(x_1, \dots, x_n, y); \\
& \}
\end{aligned}
\quad \right]
\circled{1} \\[-0.25em]
&\quad \left.
\begin{aligned}
&\kw{return}\ 
\fallback(x_1,\dots,x_n, y);
\end{aligned}
\quad\;\;\;\right]
\circled{2} \\[-0.25em]
&\}
\end{aligned}
\]
\end{tcolorbox}
\vspace{-0.9em}
\caption{\small Guarded GM with a stateful rejection sampler (1) and fallback (2).}
\label{fig:rejectionSampler}
\end{minipage}
\end{wrapfigure}

\subsubsection{Examples} We show a couple of example FGGM definitions that the planner LLM suggests: one from symbolic regression and one from Dafny program verification.  The first example from symbolic regression (Eq.~\ref{eq:fallbackSymReg}) with id $\funcIt{boundedParam}$, type signature $(l:\typeit{real}, u:\typeit{real}) \to \typeit{real}$, underlying parametric GM which is a neural network $NN^{(2)}_{\Theta}: \mathbb{R}^2 \to \mathbb{R}$ for this case, $\linspec{}: (l \leq u)$ $\loutspec{}$ (Eq.~\ref{eq:fallbackSymReg}) restricts $\funcIt{boundedParam}(l, u)$ to lie within $[l,u]$. The fallback clamps any violating sample $y$ using library functions $\funcIt{min}$ and $\funcIt{max}$, ensuring contract preservation.

\noindent The second example from Dafny program verification (Eq.~\ref{eq:fallbackDafny}) with id $\funcIt{dafnyAnnotate}$, type signature $(p:\typeit{str}) \to \typeit{str}$, underlying parametric GM which is an LLM $\llm{\Theta}: \alpN{*} \to \alpN{*}$, $\linspec{}: \funcIt{parse}(p)$, $\loutspec{}$ (Eq.~\ref{eq:fallbackDafny}) ensures $\funcIt{dafnyAnnotate}(p)$ always outputs a syntactically valid program and only adds annotations to the input program $p$. If all samples are rejected, the fallback safely returns the original program $p$, guaranteeing $\loutspec{}$ under $\linspec{}$ utilizing the \textit{\textbf{reflexivity axiom}} (Eq.~\ref{eq:axiomDafny}). These examples show how a planner LLM can define FGGMs with fallbacks to avoid incorrect program states for diverse tasks.
\begin{small}
\begin{align}
\loutspec{} &:\ (l \leq f(l,u)) \wedge (f(l,u) \leq u), 
& \fallback &:\ \text{ 
\kw{function}\ \funcIt{$\fallback$}(l:\typeit{real}, u:\typeit{real}, y:\typeit{real})
\{\kw{return}\ $\funcIt{min}(\funcIt{max}(l,y),u)$;\}} \label{eq:fallbackSymReg}\\
\loutspec{} &:\ \funcIt{parse}(p) \wedge \funcIt{noDiff}(p,f(p)), 
& \fallback &:\ \text{ 
\kw{function}\ \funcIt{$\fallback$}(p:\typeit{str}, y:\typeit{str}) $: (\typeit{str})$
\{\kw{return}\ $p$;\}} \label{eq:fallbackDafny}
\end{align}
\end{small}
\subsubsection{Design Considerations.} Alternative approaches for enforcing formal constraints on generative model calls exist, but they fall short of FGGM. Constrained decoding \cite{syncode, grammarAligned, chopChop, crane} modifies the model’s internal decoding procedure, which limits its applicability to open-source models and restricts constraints to syntactic forms such as context-free grammars. Post-hoc output filtering without a fallback provides no guarantees: if all sampled outputs are rejected, the program produces no valid output. Fully verifying an agentic program is computationally intractable, as it depends on large parametric components such as LLMs. FGGM addresses all three limitations: it operates solely on model outputs, making it compatible with closed-source models; it incorporates a verified fallback that guarantees a valid output on every execution; and it decomposes verification into parameter-independent local contracts. Moreover, these local contracts can be utilized during learning to guide parameter tuning, improving conformance, and reducing reliance on the fallback.

\subsubsection{Conformance Tuning}
\label{subsec:conformanceTune}
A FGGM that always uses the fallback satisfies $(\linspec{}, \loutspec{})$ but is practically no better than a non-parametric program $\fallback$; conformance tuning helps avoid this by teaching the GM to satisfy the defined contract directly. The local contracts $(\linspec{}, \loutspec{})$ define the local learning objective to avoid the fallback and assign higher probability to outputs that pass the check $\eval$. The objective in \secRef~\ref{subsec:learningDetails} (Eq.~\ref{eq:conformanceLoss}) maximizes the expected contract satisfaction of GM samples over an input set $\mathbb{P}$. For any $\theta \in \Theta$, the inner term $\conformLoss^{\theta}(p)$ measures whether a sampled output from $\llm{\theta}(p)$ satisfies $\eval(x_1,\dots,x_n,y)$, while $\conformLoss(\theta)$ aggregates this satisfaction across inputs in $\mathbb{P}$.
By optimizing $\theta \in \Theta$ to increase this reward, the GM post-tuning is encouraged to concentrate the probability mass on the support defined by $(\linspec{}, \loutspec{})$ (details in \secRef~\ref{subsec:learningDetails}). Overall conformance tuning improves the acceptance rate of the rejection sampler, reducing reliance on $\fallback$. Practical impact of conformance tuning is shown in \secRef~\ref{sec:ablation} (Table~\ref{tab:ablation}) and \secRef~\ref{sec:fallback} (Table~\ref{tab:fallback}).

\subsubsection{Theoretical Evaluation}
We show that the FGGM design does not reduce the theoretical expressivity of transformer-based models (Theorem~\ref{thm:expressivityFGGM}). Specifically, with an appropriate choice of model parameters and a sufficiently large number of autoregressive decoding steps, commonly referred to as Chain-of-Thought (CoT) reasoning, the underlying generative model can simulate any given Turing machine $M$ \cite{cotExpressivity}. Consequently, FGGM avoids the known expressivity loss associated with naive constrained decoding methods \cite{crane}. Note that an appropriate output processor $\fpost{}$ can remove intermediate CoT reasoning tokens and extract the final output from the raw GM output. Further, based on the loss obtained from the unconstrained learning step, we compute a lower bound on the cumulative reward of FGGM (Theorem~\ref{thm:rewardLowerBound}).
    

\subsection{\Tool}
In this section, we briefly discuss the purpose behind the four stages of \Tool, the key design decisions underlying each stage, and the theoretical properties of the overall framework.
\vspace{-2em}
\subsubsection{Ideation:} The planner LLM defines task-specific generative components through FGGMs, together with their prompting programs $\prompt$, and expands the user-provided component set. This is analogous to the abstraction phase of popular library learning approaches \cite{dreamCoder, houdini}, which learn task-specific abstractions and extend the synthesis DSL. All proposed FGGM definitions are subsequently verified for well-formedness. Specifically, the non-parametric fallback program $\fallback$ is written in a verification-aware language such as Dafny, whose built-in verifier identifies failing program paths and provides them as feedback to the planner LLM for refinement.
Note that direct calls to parametric GMs are disallowed; each must be wrapped within a verified FGGM.
\begin{wrapfigure}{r}{0.46\textwidth}
\vspace{0.5em}
\centering
\begin{subfigure}{\linewidth}
\vspace{-20pt}
\begin{tcolorbox}[
colback=lightgraybg,
colframe=lightgraybg,
boxrule=0pt,
arc=2pt,
left=4pt,right=4pt,top=4pt,bottom=4pt
]
\small
\[
\begin{aligned}
&\kw{function}\ \funcIt{agent}(x:\typeit{real}) 
 : (\typeit{real})\ \{ \\[-0.25em]
&\quad \kw{var}\ a:\typeit{real} := 
\funcIt{boundedParam}(0.0, 1.0); \\[-0.25em]
&\quad \kw{var}\ b:\typeit{real} := 
\funcIt{boundedParam}(0.0, 1.0); \\[-0.25em]
&\quad \kw{var}\ linear\_x:\typeit{real} := a * x; \\[-0.25em]
&\quad \kw{var}\ y:\typeit{real} := linear\_x + b; \\[-0.25em]
&\quad \kw{return}\ y; \\[-0.25em]
&\}
\end{aligned}
\]
\end{tcolorbox}
\vspace{-1em}
\caption{\small Unverified parametric program pruned out.}
\label{fig:unverifiedCandidateSymReg}
\end{subfigure}

\vspace{1.5em}

\begin{subfigure}{\linewidth}
\begin{tcolorbox}[
colback=lightgraybg,
colframe=lightgraybg,
boxrule=0pt,
arc=2pt,
left=4pt,right=4pt,top=4pt,bottom=4pt
]
\small
\[
\begin{aligned}
&\kw{function}\ \funcIt{agent}(x:\typeit{real}) 
: (\typeit{real})\ \{ \\[-0.25em]
&\quad \kw{if}\ (x \leq 0.0)\ \{\kw{return}\ 0.0;\} \\[-0.25em]
&\quad \kw{var}\ a:\typeit{real} := 
\funcIt{boundedParam}(1.0, 1.5); \\[-0.25em]
&\quad \kw{var}\ d:\typeit{real} := 
\funcIt{boundedParam}(0.5, 0.8); \\[-0.25em]
&\quad \kw{assert}\ x \geq 0.0; \\[-0.25em]
&\quad \kw{var}\ pow\_x:\typeit{real} := 
\funcIt{pow}(x, d); \\[-0.25em]
&\quad \kw{assert}\ (d \leq 0.8); \quad \kw{assert}\ (d \geq 0.5);\\[-0.25em]
&\quad \kw{assert}\ (x \leq 1) \implies (pow\_x \geq \funcIt{pow}(x, 0.8)); \\[-0.25em]
&\quad \kw{assert}\ (x \geq 1) \implies (pow\_x \geq \funcIt{pow}(x, 0.5)); \\[-0.25em]
&\quad \kw{assert}\ (x \geq 0) \implies (\funcIt{pow}(x, 0.5) = \funcIt{sqrt}(x)); \\[-0.25em]
&\quad \kw{var}\ y:\typeit{real} := a * pow\_x; \\[-0.25em]
&\quad \kw{return}\ y; \\[-0.25em]
&\}
\end{aligned}
\]
\end{tcolorbox}
\vspace{-1em}
\caption{\small Verified parametric program.}
\label{fig:verifiedCandidateSymReg}
\end{subfigure}

\vspace{-0.5em}
\caption{\small Two example candidate parametric programs generated for symbolic regression.}
\label{fig:candidateAgentsSymReg}
\vspace{-0.7em}
\end{wrapfigure}

\vspace{-0.5em}
\subsubsection{Component-Based Synthesis:} In this step, \Tool synthesizes programs by combining components from the expanded component set. It implements a CEGIS-style loop, with the planner LLM sampling candidate programs in Dafny (search) and the built-in Dafny verifier acting as the checker, providing unverified program paths as counterexamples to the planner LLM for subsequent refinement. Provided that all FGGM components are well-formed and the synthesized program verifies against the user specification, compliance is guaranteed for all choices of generative model parameters.

\subsubsection{Alternative Synthesis Approach:} Both $\fallback$ synthesis and agent synthesis are instances of component-based synthesis. Consequently, \Tool can leverage any existing component-based synthesis algorithm \cite{petrinet, codeHint}. However, classical synthesizers typically make simplifying assumptions, such as using type signatures as proxies for component contracts and behavioral specifications. Moreover, they have two key limitations: (i) they cannot leverage natural-language task descriptions, which provide valuable context for guiding the planner LLM, and (ii) they rely heavily on testing to prune incorrect candidates, which can incur substantial cost by exploring many irrelevant candidates, especially in the presence of multiple GM calls. In contrast, the planner LLM-based search--verify loop provides a simple and practically effective mechanism for generating high-quality candidate programs, as demonstrated by our experiments.

\subsubsection{Parameter Learning:} Since the compliance guarantees hold for all parameter values, the learning step can apply unconstrained gradient-based optimization to tune the GM parameters. Parameter tuning is guided by both the reward oracle $\reward{}$ and the local conformance rewards from each FGGM. We show that the parameters of each FGGM can be optimized independently, ensuring scalability. For closed-source LLMs whose parameters are inaccessible, the learning step is skipped, and improvements rely solely on prompt tuning through $\prompt$. 

\subsubsection{Candidate Pool and Execution Feedback}
\label{sec:agentExamples}
\Tool maintains a pool of verified candidate agentic programs and their fine-tuned parameters (Fig.~\ref{fig:overView}). It greedily selects the candidate with the highest cumulative reward $\reward{}$ and executes it on $\data$ to identify soft counterexamples, i.e., inputs whose reward falls below a user-defined threshold. The execution traces of these counterexamples are then used as feedback for the planner LLM to refine the candidate programs. The number of candidate sampling iterations is capped by the user. Once this limit is reached, \Tool returns the candidate in the pool with the highest cumulative reward $\reward{}$.
Currently, \Tool uses a simple greedy selection strategy, which already performs well in practice (see \secRef~\ref{sec:mainResults}). More sophisticated approaches, such as multi-armed bandits with Thompson sampling for candidate evaluation~\cite{tang2024code}, could further improve the exploration--exploitation trade-off.
\subsubsection{Example} Fig.~\ref{fig:candidateAgentsSymReg} shows two candidate parametric programs sampled by the planner LLM. Each program invokes the FGGM defined in Eq.~\ref{eq:fallbackSymReg} with id $\funcIt{boundedParam}$ from two different locations. Fig.~\ref{fig:unverifiedCandidateSymReg} shows the first parametric program, which fails to verify against the behavioral specification and is therefore pruned by the verifier. This candidate represents the output as a parametric affine function of the input $x$, which does not match the ground truth, illustrating how strict behavioral constraints eliminate poor candidates.
The second candidate verifies for all parameters of the neural networks $\funcIt{NN^{(2)}_{\Theta}}$ wrapped inside the $\funcIt{boundedParam}$ calls. After parameter tuning, it recovers the ground truth with negligible error, as discussed below.
Currently, \Tool does not share parameters across FGGM calls at different locations and instead assigns separate parameter sets to each $\funcIt{boundedParam}$ call. This design allows the system to potentially learn different parameters at different call sites to improve task-specific performance.
Since both $\funcIt{boundedParam}$ calls receive constant inputs, after fine-tuning they return constant values, i.e., $a = 1.11$ and $d = 0.503$. This recovers the ground-truth function $\sqrt{1.23 \times \max(x, 0)}$ with negligible error. Moreover, \Tool can be extended to share parameters across call sites and to cap the number of FGGM calls, both of which can be formally enforced by static resource analysis (see \secRef~\ref{sec:discussion}).
The parametric candidate programs and all FGGM definitions with synthesized prompting programs $\prompt{}$ for the program verification task are included in Appendix~\ref{appen:exampleDafnyAgent}.
\subsubsection{Theoretical Evaluation of \Tool}
Alongside soundness, we establish the theoretical effectiveness of \Tool in balancing specification compliance and reward optimization (Theorem~\ref{thm:sufficientCond}). Assuming that the reward oracle $\reward{}$ penalizes constraint-violating outputs and that there exists a non-parametric fallback program $f_d$ over $\fset{c}$ (recoverable by any classical synthesizer) satisfying the global behavioral specification $(\inspec{}, \outspec{})$, \Tool can always recover a specification-compliant program whose cumulative reward is at least that of the best available parametric model in $\fset{p}$, irrespective of the reward achieved by the fallback program. Furthermore, assuming parameter optimization eventually recovers the ideal model parameters, \Tool can recover the optimal program by the FGGM expressivity result (Theorem~\ref{thm:expressivityFGGM}).


\section{\Tool Details}
\subsection{Search Space} 
\label{sec:searchSpace}
\subsubsection{Program Search Space} We use a restricted subset of the popular verification-aware language Dafny \cite{dafny} to define the syntactically valid candidate programs. In this restricted subset, we allow conditional blocks, while loops, assignment statements, assertion statements, function calls, and annotations including invariant clauses and ranking functions (decreases clauses), along with first-order input–output specifications. We present the BNF grammar $\dg$ capturing this subset in Appendix~\ref{appen:seachSpace}. We allow four basic types $\typeSet = \set{\typeit{bool}, \typeit{int}, \typeit{real}, \typeit{str}}$ and list $arr\langle\typeSet\rangle$, basic built-in arithmetic functions $\bin{a} = \set{\funcIt{+}, \funcIt{-}, \funcIt{\times}, \funcIt{/}}$, boolean functions $\bin{b} = \set{\funcIt{\&\&}, \funcIt{||}, \funcIt{!}, \funcIt{\implies}}$, and relations $\bin{r} = \set{\funcIt{==}, \funcIt{\leq}, \funcIt{\geq}, \funcIt{<}, \funcIt{>}}$ with their usual semantics. We use a planner LLM $\planner{}: \alpN{*} \to \alpN{*}$ to search the program space by sampling candidate programs as strings. We assume that $\planner{}$’s output space can represent all valid Dafny programs $L(\dg) \subseteq \alpN{*}$ and that the type $\typeit{str}$ can represent all strings $\alpN{*}$ over $\alp{}$. \Tool provides built-in predicates $\funcIt{lex_{T}}(x: \typeit{str}) \to \typeit{bool}$ to check whether a string can be parsed as a type $T \in \set{\typeit{bool}, \typeit{int}, \typeit{real}}$. Additionally, \Tool exposes a string processor $\funcIt{split}: \alpN{}\times \typeit{str} \to arr\langle\typeit{str}\rangle$ that splits a string on occurrences of a specified character $c \in \alpN{}$, enabling the definition of $\fpost{}$ in FGGM.  Using Dafny as the language allows us to leverage its built-in type-checker, verifier, and termination checker to check candidate programs against the behavioral specification.
While the framework is not language-specific, using a verification-aware host language avoids re-implementing a custom verifier and termination checker.

\subsubsection{Library Components}
The library $\fset{}$ is partitioned into two disjoint sets, $\fset{} = \fset{c} \cup \fset{p}$, where $\fset{c}$ contains non-parametric components and $\fset{p}$ contains parametric components. Components in $\fset{c}$ have fixed over-approximate semantics specified by contracts, while components in $\fset{p}$ represent families of functions indexed by parameters. Each non-parametric component denotes a pure, terminating function $f_n : T_1 \times \cdots \times T_n \to T_o$. For any inputs $x_i \in T_i$, if $\pre{f_n}(x_1,\dots,x_n)$ holds and $r = f_n(x_1,\dots,x_n)$, then $\post{f_n}(x_1,\dots,x_n,r)$ holds. Axiom set $\axiom{}$ contains the contracts of all $f_n \in \fset{c}$.

\begin{definition}[Non-parametric component]
$f_n \in \fset{c}$ of arity $n$ is a tuple $(\funcIt{id(f_n)}, \typs{f_n}, \pre{f_n}, \post{f_n})$, with identifier $\funcIt{id(f_n)}$, type signature $\typs{f_n} = T_1 \times \cdots \times T_n \to T_o$, and contract $(\pre{f_n}, \post{f_n})$.
\end{definition}

\begin{definition}[Parametric component]
$f^{\paramset{}}_n \in \fset{p}$ of arity $n$ is a tuple $(\funcIt{id(f^{\paramset{}}_n)}, \typs{f^{\paramset{}}_n}, \param{0}, \paramset{})$, with identifier $\funcIt{id(f^{\paramset{}}_n)}$, type signature $\typs{f^{\paramset{}}_n}$, parameter space $\paramset{}$, and default parameter $\param{0} \in \paramset{}$.
\end{definition}

\noindent A parametric component denotes the family of functions $f^{\paramset{}}_n = \set{f^{\param{}}_n \mid \param{} \in \paramset{}}$, where every instance $f^{\param{}}_n$ shares the same type signature $\typs{f^{\paramset{}}_n}$. Although $f^{\paramset{}}_n$ denotes a family of functions, we sometimes abuse notation and write $f^{\paramset{}}_n$ for the default instance $f^{\param{0}}_n$ when the intended parameter value is clear from context.
The input to \Tool is $\mathcal{I} = (\fset{}, \axiom{}, \reward{}, \data{}, \inspec{}, \outspec{}, I)$, where $\fset{}$ is the component library, $\axiom{}$ contains the contracts for components in $\fset{c}$, $\reward{}$ is a pointwise black-box reward oracle, $\data{}$ is the fixed training dataset, $(\inspec{}, \outspec{})$ is the behavioral specification, and $I$ is textual task information.

\noindent The initial synthesis context $\context = (\dg, \fset{c}, \axiom{})$ includes the grammar, non-parametric components, and their contracts. The synthesis context grows during the ideation step with the well-formed FGGM definitions. We use $\parser : \alpN{*} \to \set{T,F}$, $\ver{\inspec{}}{\outspec{}} : \alpN{*} \to \set{T,F}$, and $\ter : \alpN{*} \to \set{T,F}$ to denote the parser, verifier, and termination checker, respectively. The verifier checks candidate programs against $(\inspec{}, \outspec{})$, while the termination checker proves termination. Both procedures run under a fixed timeout and return \textit{false} if they cannot prove the property within the allotted budget.

\subsection{FGGM Details}
\label{sec:keyStepDetails}
\begin{small}
\begin{equation}
\langle \funcIt{FGGM} \rangle ::= 
\langle 
\typeit{id}\rangle \ \text{";"}\ 
\langle \typeit{GMid} \rangle \text{";"}\ 
\underbrace{
\langle\typeit{typeSig}\rangle \ \text{";"}\ 
\langle \typeit{\linspec{}} \rangle \text{";"}\ 
\langle \typeit{\loutspec{}} \rangle \ \text{";"}
}_{\text{signature + formal contract}} 
\underbrace{
\langle \typeit{\prompt} \rangle \text{";"} 
}_{\text{prompt}}
\underbrace{
\langle \typeit{\fallback} \rangle \text{";"}
}_{\text{fallback}}
\langle \typeit{\fpost} \rangle
\label{eq:FGGMDefBNF}
\end{equation}
\end{small}
\subsubsection{FGGM Formal Definition}
\label{sec:FGGMFormal}
Eq.~\ref{eq:FGGMDefBNF} specifies the high-level syntax rules defining an FGGM. The FGGM definition includes:
(a) \textit{id:} is the unique identifier of the created FGGM;
(b) \textit{GMid:} identifies the underlying GM used as the proposal distribution;
(c) \textit{typeSig:} defines the type signature of the FGGM, including input variables with their types and the output type, e.g., $(x_1: \typeit{T_1}, \dots, x_n: \typeit{T_n}) \to \typeit{T_o}$;
(d) \textit{$(\linspec{}, \loutspec{})$:} are quantifier-free formulas whose terms can include non-parametric library functions from $\fset{c}$;
(e) \textit{$\prompt$:} is the program that constructs the GM input from the typed input variables and the task description;
(f) \textit{$\fallback$:} is a non-parametric fallback program that satisfies the local specifications $(\linspec{}, \loutspec{})$; and (g) \textit{$\fpost{}$:} processes the raw GM output and can optionally remove part of the response.
\Tool parses these inputs, validates the definition, and automatically synthesizes a rejection sampler with fallback, as shown in Fig.~\ref{fig:rejectionSampler} where $\llm{\Theta}$ is the parametric GM specified by GMid. The defined FGGM is then invoked in the candidate program via its identifier $\funcIt{id}$. For valid definitions, \Tool synthesizes $\eval$, which checks whether a concrete input--output pair $(x_1, \dots, x_n, y)$ satisfies the proposed contract $(\linspec{}, \loutspec{})$, as elaborated in \secRef~\ref{subsec:FGGMcomputability}.

\noindent We provide the BNF grammar for parsing the FGGM definition in Appendix~\ref{appen:FGGMsyntax}. Next, we discuss correctness checks for a proposed FGGM definition. Provided \Tool validates the definition including the fallback $\fallback$, $\forall \theta \in \Theta. \forall x_1 \in T_1\dots\forall x_n \in T_n. \linspec{}(x_1,\dots, x_n) \implies \loutspec{}(x_1,\dots,x_n, \funcIt{id}(x_1, \dots, x_n))$ holds ensuring formal contracts are preserved \emph{for all parametric values} $\theta \in \Theta$. 

\subsubsection{FGGM Well-formedness}
\label{sec:wellFormedFGGM}
Given an FGGM definition $\fggm{} = (\funcIt{id_{\fggm}}, f^{\paramset{}}_{\fggm{}}, \typs{\fggm{}}, \linspec{}, \loutspec{}, \prompt, \fallback, \fpost)$ following Eq.~\ref{eq:FGGMDefBNF}, this step checks whether the definition is well-formed. Let the parametric GM identified by \textit{GMid} be $\llm{\Theta} : T_{in} \rightarrow T_o$. 
\Tool first checks the well-formedness of the type signature $\typs{\fggm{}} = (T_1 \times \cdots \times T_n) \to T_o$ and the well-formedness of the local contracts $(\linspec{}, \loutspec{})$.
We then verify that all terms in $\linspec{}$ and subsequently in $\loutspec{}$ that involve non-parametric functions from $\fset{c}$ are valid (details in Appendix~\ref{appen:FGGMvalidity}). 
The $\fggm{}$ definition is considered valid provided both $\prompt, \fallback$ and $\fpost$ are type-checked and terminating, and the fallback $\fallback$ satisfies the contract $\ver{\linspec{}}{\loutspec{}}(\fallback)$.  
After the validity check, \Tool constructs  
$\eval : (T_1 \times \cdots \times T_n \times T_o) \to \set{T, F}$  
that checks whether a concrete input–output tuple $(x_1, \dots, x_n, y)$ satisfies the local output contract $\loutspec{}$ (see \secRef~\ref{subsec:FGGMcomputability}).  
We provide the pseudocode of $\eval$ in Appendix~\ref{appen:proofDetails} and show it is sound and complete (Lemma~\ref{lem:checkerComplete}).

\begin{restatable}[Completeness of Checker]{lemma}{checkerComplete}
\label{lem:checkerComplete}
For any valid FGGM $\fggm{}$ with quantifier-free $\loutspec{}$,  
$\forall x_1 \in T_1, \dots, \forall x_n \in T_n, \forall y \in T_o.\;
\linspec{}(x_1, \dots, x_n) \implies \left(\eval(x_1, \dots, x_n, y)
\iff \loutspec{}(x_1, \dots, x_n, y)\right)$.
\end{restatable}

\begin{proofSketch}
We prove this by structural induction on the quantifier-free formula $\loutspec{}$.
\emph{Base case:} If $\loutspec{}$ is an atomic predicate over terms built from computable library functions in $\fset{c}$, then under any concrete substitution $(x_1,\dots,x_n,y)$, $\eval$ evaluates the same computable terms and returns the same boolean value as $\loutspec{}$.
\emph{Inductive step:} Assume the claim holds for formulas $\phi$ and $\psi$. For compound formulas $\neg \phi$, $\phi \wedge \psi$, and $\phi \vee \psi$, $\eval$ recursively evaluates the subformulas and applies the corresponding boolean operator, and returns the result.
Thus, by induction on the structure of $\loutspec{}$, $\eval(x_1,\dots,x_n,y)$ returns true iff $\loutspec{}(x_1,\dots,x_n,y)$ holds whenever $\linspec{}(x_1,\dots,x_n)$ holds. Hence, the checker is complete. Formal details in Appendix~\ref{appen:proofDetails}.
\end{proofSketch}

\noindent For any valid FGGM $\fggm{}$, the rejection sampler  
$\funcIt{id_{\fggm{}}} : T_1 \times \cdots \times T_n \to T_o$  
with fallback synthesized by \Tool (Fig.~\ref{fig:rejectionSampler}) satisfies the defined contract $(\linspec{}, \loutspec{})$ for all well-typed inputs $(x_1, \dots, x_n)$. Crucially, the contract preservation holds over all parameters $\theta \in \Theta$ of the underlying parametric model $f^{\paramset{}}_{\fggm{}}$.  
The function $\funcIt{id_{\fggm{}}}$ terminates on all inputs since the number of samples $K$ is finite and $\prompt$, $\fallback$, $\eval$ and $\fpost$ are all terminating. In Eq.~\ref{eq:inferenceRule}, we present a parameter-independent contract preservation rule for arity $n = 1$, modulo user-provided axioms $\axiom{}$. A proof sketch is given in Theorem~\ref{thm:FGGMsoundness}.
\begin{small}
\begin{align}
\inferrule*[]
{
  (\theta \in \Theta)
  \qquad 
  \Gamma \vdash \prompt: T_1 \to T_{in},\; f^{\paramset{}}_{\fggm{}} : T_{in} \to T_o,\;  \fpost: T_o \to T_o,\; \fallback: T_1 \times T_o \to T_o
  \qquad
  \{\axiomEncode{\axiom{}}\wedge\linspec{}\}\; \fallback \;\{\loutspec{}\}
}
{
  \Gamma \vdash \funcIt{id_{\fggm{}}}[\theta/\Theta] : T_1 \to T_o
  \qquad
  \{\axiomEncode{\axiom{}}\wedge\linspec{}\}\; \funcIt{id_{\fggm{}}}[\theta/\Theta] \;\{\loutspec{}\}
} \label{eq:inferenceRule}
\end{align}
\end{small}
\begin{restatable}[Parameter Independence]{theorem}{localContract}
\label{thm:FGGMsoundness}
For any valid FGGM $\fggm{}$ with $(\linspec{}, \loutspec{})$ and parametric GM $f^{\paramset{}}_{\fggm{}}$,  
$\axiomEncode{\axiom{}}\implies\big(\forall \theta \in \Theta, \forall x_1 \in T_1, \dots, \forall x_n \in T_n.\;
\linspec{}(x_1, \dots, x_n)
\implies
\loutspec{}(x_1, \dots, x_n, \funcIt{id_{\fggm{}}}(x_1, \dots, x_n))\big)$.
\end{restatable}
\begin{proofSketch}
Let $r = \funcIt{id_{\fggm{}}}(x_1, \dots, x_n)$. Based on the rejection sampler with the checker $\eval$, the following condition always holds for the output $r$ (Eq.~\ref{eq:FGGMsound1}). Eq.~\ref{eq:FGGMsound2} follows from the validity of the fallback $\fallback$, and Eq.~\ref{eq:FGGMsound3} follows from the soundness of the checker $\eval$. To simplify the notation, we used $(\forall x_i\in T_i)$ to denote $(\forall x_1 \in T_1, \dots, \forall x_n \in T_n)$ and $y_f$ is the final rejected sample if $\fallback$ is used.
\begin{small}
\begin{align}
&\axiomEncode{\axiom{}}\implies (\forall \theta \in \Theta. (\forall x_i\in T_i). \linspec{}(x_1, \dots,x_n) \implies\eval(x_1,\dots,x_n, r) \vee (r = \fallback(x_1,\dots,x_n,y_f))) \label{eq:FGGMsound1}\\
&\axiomEncode{\axiom{}}\implies\forall \theta \in \Theta. (\forall x_i\in T_i). \linspec{}(x_1, \dots,x_n) \wedge (r = \fallback(x_1,\dots,x_n, y_f)) \implies \loutspec{}(x_1,\dots, x_n, r) \label{eq:FGGMsound2}\\
&\axiomEncode{\axiom{}}\implies(\forall \theta \in \Theta. (\forall x_i\in T_i). (\linspec{}(x_1, \dots,x_n) \wedge \eval(x_1,\dots,x_n, r)) \implies \loutspec{}(x_1,\dots, x_n, r)) \label{eq:FGGMsound3}\\
&\axiomEncode{\axiom{}}\implies(\forall \theta \in \Theta. (\forall x_i\in T_i). \linspec{}(x_1, \dots,x_n) \implies \loutspec{}(x_1, \dots, x_n, r))\;\;\;\; \text{  Using Eq.~{(\ref{eq:FGGMsound1}, \ref{eq:FGGMsound2}, \ref{eq:FGGMsound3}}}) \nonumber
\end{align}    
\end{small}
We provide the detailed formal proof in Appendix~\ref{appen:proofDetails}.
\end{proofSketch}    

\subsubsection{Computation of $\ferr{}$ (Fig.~\ref{fig:rejectionSampler}) on concrete values} 
\label{subsec:FGGMcomputability} 
For any rejected output $y$, $\ferr{}$ returns all failing atomic formulas in $\loutspec{}$ evaluated on $(x_1, \dots, x_n, y)$. Similar to the checker $\eval{}$, since all terms are computable, $\ferr{}$ can exhaustively evaluate the truth value of every atomic formula and return those that fail. The pseudocode for $\ferr{}$ is provided in Appendix~\ref{appen:errPseudocode}.

\subsection{\Tool Details}
\subsubsection{Candidate Program Search}
\begin{wrapfigure}{r}{0.475\textwidth} 
\vspace{-35pt}
\resizebox{\linewidth}{!}{%
\begin{minipage}{\linewidth}
\begin{algorithm}[H]
\footnotesize
\caption{programSearch}
\label{alg:searchVerify}
\begin{algorithmic}[1]
\State \textbf{Input:} $\mathcal{I} =(\fset{},\axiom{},\inspec{},\outspec{}, I)$, planner $\planner{}$, budget $\delta$
\State \textbf{Input:} Previous attempts feedback $I'$ 
\State \textbf{Output:} Verified Parametric Program $\program{}$ else $\bot$ 
\State $i \gets 0$, $\msf{err} \gets \{\}$
\While{$i < \delta$}
    \State $\fggmSet \gets \funcIt{defineFGGM}(\planner{}, I,\fset{},\axiom{}, I', \msf{err})$
    \State $\msf{err} \gets \funcIt{validateFGGM}(\fggmSet{}, \fset{c},\axiom{})$
    \If{$\msf{err} \neq \set{}$}
    \State $i \gets i+1$;\;\;\textbf{continue}
    \EndIf
    \State $\program{}  \gets \funcIt{sampleProgram}(\planner{}, I, \fset{},\axiom{}, I', \fggmSet{}, \msf{err})$
    \State $\context \gets (\dg, \fset{}\cup\fggmSet{}, \axiom{} \cup \axiom{}_{\fggmSet{}})$
    \State $\msf{err} \gets \funcIt{validateProgram}(\program{}, \parser{}, \ter{},\ver{\inspec{}}{\outspec{}})$
    \If{$\msf{err} = \set{}$}
    \State \Return $\program{}$
    \EndIf
    \State $i \gets i + 1$
\EndWhile
\State \Return $\bot$
\end{algorithmic}
\end{algorithm}
\end{minipage}%
}
\vspace{-25pt}
\end{wrapfigure}

Suppose $\fggmSet = \set{\fggm{}_1,\dots,\fggm{}_m}$ denotes the set of well-formed FGGMs produced during the ideation stage. \Tool extends the initial synthesis context $C$ with the synthesized FGGMs and their local contracts. The updated context is $\context' = (\dg, \fset{} \cup \fggmSet{}, \axiom{} \cup \axiom_{\fggmSet{}})$, where $\axiom_{\fggmSet{}}$ contains the local contracts $(\fggm{}_i^{\linspec{}},\fggm{}_i^{\loutspec{}})$ for each $\fggm{}_i \in \fggmSet{}$. \Tool then uses a component-based synthesis algorithm to synthesize a candidate program $\program{}$ in the extended context $\context'$.

\begin{definition}[Synthesizer]
A component-based synthesizer $\synthesizer$ takes as input a context $C = (\dg, \fset{}, \axiom{})$ and a specification $(\inspec{}, \outspec{})$. It either synthesizes a program $\program{}$ that satisfies $(\inspec{}, \outspec{})$ under $C$ (see Eq.~\ref{eq:axiomImplication}) or returns $\bot$.
\end{definition}
Any candidate program $\synthesizer(C', \linspec{}, \loutspec{}) = \program{}$ may contain FGGM calls at distinct program locations. Each FGGM call introduces an optimizable parameter, and $\fullParamSet{} = \set{\paramset{}_1,\dots,\paramset{}_k}$ denotes the set of all optimizable parameters in $\program{}$, where $k$ is the number of FGGM calls.
\begin{restatable}{theorem}{searchVerify}
\label{thm:verificationParamSet}
If $\synthesizer(C', \inspec{}, \outspec{}) = \;\program{} \neq \bot$ with FGGM set $\fggmSet{}$, then $\forall \theta_1 \in \Theta_1, \dots, \forall \theta_k \in \Theta_k. \axiomEncode{\axiom{}} \implies \big(\forall x_1 \in T_1, \dots, \forall x_n \in T_n. \inspec{}(x_1, \dots, x_n) \implies \outspec{}(x_1, \dots, x_n, \programSub{(\theta_1, \dots, \theta_k)}(x_1, \dots, x_n))\big)$.
\end{restatable}
\begin{proofSketch}
Follows from Theorem~\ref{thm:FGGMsoundness} that shows all FGGM contracts in $\axiom{}_{\fggmSet{}}$ holds independent of underlying parameter values. Details in Appendix~\ref{appen:proofDetails}.
\end{proofSketch}

\subsubsection{Synthesizer Implementation}
In \Tool, the component-based synthesizer uses planner-LLM-guided search and verification, as shown in Algorithm~\ref{alg:searchVerify}. Given \Tool's input $\mathcal{I}$ and planner $\planner{}$, the algorithm iteratively proposes candidate solutions within a fixed search budget $\delta$.
At each iteration, the planner synthesizes a set of FGGM definitions $\fggmSet$ (Line~6), which are checked for well-formedness using $\funcIt{validateFGGM}$ (Line~7). It also synthesizes an FGGM-specific prompting function $\prompt{}$. If the FGGM definitions are valid, the planner samples a candidate parametric program $\program{}$ that may invoke these FGGMs (Line~11), and the verification context is extended with their contracts (Line~12).
The sampled program is then checked for syntax, termination, and compliance with the behavioral specification using $\funcIt{validateProgram}$ (Line~13). If verification succeeds, the algorithm returns $\program{}$ (Lines~14--15); otherwise, it continues using error feedback from failing program paths until the search budget is exhausted (Lines~16--18).

\subsubsection{Learner Implementation}
\label{subsec:learningDetails}
The verification step guarantees that $\program{}$ satisfies the behavioral specification $(\inspec{},\outspec{})$ for \emph{all} parameter values (Theorem~\ref{thm:verificationParamSet}). Thus, learning can optimize parameters without violating $(\inspec{}, \outspec{})$. However, solving the general objective in Eq.~\ref{eq:learnObjective} over $\fullParamSet{}$ with a black-box reward $\reward{}$ is practically intractable. Instead, \Tool uses scalable gradient-based optimization to search over $\fullParamSet{}$.
\begin{small}
\begin{align}
(\opt{\theta_1},\dots, \opt{\theta_k}) = \argmax_{(\theta_1,\dots, \theta_k) \in \fullParamSet{}} 
\;\; 
{\textstyle \sum_{(x,y)\in\data} \reward{}(x,y,\programSub{}(x))} 
\label{eq:learnObjective}
\end{align}    
\end{small}

\noindent For each local contract $(\linspec{}, \loutspec{})$, the conformance reward $\conformLoss{}$ encourages parameters to produce outputs that satisfy the contract. Improving conformance increases the success rate of the corresponding rejection sampler, thereby reducing reliance on the static non-parametric fallback program. Since the fallback cannot be tuned and typically has low reward (Table~\ref{tab:fallback}), reducing fallback usage helps optimize the final reward $\reward{}$ more effectively. Our ablation study in \secRef~\ref{sec:ablation} supports this observation.
For a single FGGM with local contract $(\linspec{}, \loutspec{})$ and parameter $\theta$, Eq.~\ref{eq:conformanceLoss} defines the conformance reward over the set of GM inputs, or prompts, $\mathbb{P}_{D}$ collected from all executions on $D$. The inner term $\conformLoss^{\theta}(p)$ measures the fraction of top-$K$ samples $\set{y_1,\dots,y_K}$ from $\llm{\theta}(p)$ that satisfy $\eval$, while $\conformLoss(\theta)$ aggregates this satisfaction across inputs.

\begin{small}
\begin{align}
\underbrace{\conformLoss(\theta) = \tfrac{1}{|\mathbb{P}_D|} \times {\textstyle \sum_{p \in \mathbb{P}_D} \conformLoss^{\theta}(p)}}_{\text{Aggregated constraint satisfaction over $\mathbb{P}_D$}}, \;\; \conformLoss^{\theta}(p) = \underbrace{\textstyle \frac{1}{K}\times\sum_{i=1}^{K} \mathbb{I}(\eval(x_1,\dots,x_n, y_i))}_{\text{ Constraint satisfaction on samples from single GM input $p$}} \label{eq:conformanceLoss}
\end{align}
\end{small}
The augmented objective in Eq.~\ref{eq:augmentedLoss} combines the task reward with penalties for local conformance violations at each FGGM call. Here, $\lambda \in \mathbb{R}^{+}$ is a user-provided constant, and $\epsilon \in \mathbb{R}^{+}$ is a small smoothing constant that avoids taking the logarithm of zero. The logarithm strongly penalizes low reward or any local conformance violation.
\begin{small}
\begin{align}
 {\textstyle \sum_{(x,y)\in\data} \log(\reward{}(x,y,\programSub{}(x))} + \epsilon)\; + \; {\textstyle \sum_{i \in k, (x,y)\in\data} \log(\lambda \times\conformLossParam{i}(\theta_i) + \epsilon)}   \label{eq:augmentedLoss}
\end{align}    
\end{small}

\noindent Jointly optimizing all parameters under Eq.~\ref{eq:augmentedLoss} is feasible for smaller neural networks but becomes expensive for agents with multiple interdependent LLM calls.
To scale learning, we exploit a key structural property: each conformance term $\conformLossParam{i}(\theta_i)$ depends only on the local inputs to the corresponding FGGM and the outputs of its underlying GM, not on parameters of other FGGMs. This independence holds because each FGGM's rejection sampler and $\eval$ checker operate only on its local contract $(\linspec{}^i, \loutspec{}^i)$ and generated outputs. Therefore, the conformance term decomposes across FGGM calls.
For larger models, we approximate the augmented objective by optimizing each FGGM $\funcIt{id_{\theta_i}}$ with parameter $\theta_i$ independently, as shown in Eq.~\ref{eq:decomposedLoss}. The first term includes the task reward $\reward{}$ only when one of the GM's outputs $y_l \in \set{y_1, \dots, y_K}$ matches the agent's final output $y$ on input $x$, where $(x,\_) \in \data$. This modular objective enables parallel training across different $\theta_i$.

\begin{small}
\begin{align}
& \reward{}_a(\theta_i) = \tfrac{1}{|\mathbb{P}_D|} \times {\textstyle \sum_{p} \reward{}_a^{\theta_i}(p)};\;\;\;\;    \reward{}^{\theta_i}_{a}(p_x) = \textstyle \sum_{l}\log(\reward{}(x, \_, y_l)\times \mathbb{I}(y = y_l) + \epsilon)+ \log(\lambda \times \conformLoss^{\theta_i}(p_{x}) + \epsilon)\label{eq:decomposedLoss}  
\end{align}
\end{small}
\noindent For generative models with accessible parameters, the objective in Eq.~\ref{eq:decomposedLoss} can be optimized using reinforcement learning–based fine-tuning methods such as GRPO \cite{grpo} (details in Appendix~\ref{appen:GRPO}). Eq.~\ref{eq:decomposedLoss} defines the reward for an input $p_x$, with the learning objective of maximizing the total reward over all inputs $p_x \in \mathbb{P}$. The reward function combines two signals: (1) the task-specific reward $\reward{}(x, \_, y)$, active when the GM output $y_l$ matches the final agent output $y$, and (2) the conformance checker $\eval$, which provides a bonus for outputs satisfying the local contract. The reward function need not be differentiable w.r.t. $\theta_i$, as long as the outputs $y_l$ remain differentiable w.r.t. $\theta_i$ \cite{grpo}. For closed-source models with inaccessible parameters, the learning step is skipped, and performance improvements rely solely on prompt tuning through the synthesized $\prompt$ programs within each FGGM.
\subsubsection{\Tool}
\noindent Algorithm~\ref{alg:tool} summarizes the overall workflow of \Tool. 
The algorithm maintains a pool $\mathcal{P}$ of verified candidate agents and iteratively improves it through a search–verify–learn loop. 
In each iteration, \Tool first invokes the \funcIt{programSearch} procedure (Algorithm~\ref{alg:searchVerify}) to generate candidate parametric program $\program{}$  that satisfy the behavioral specification $(\inspec{},\outspec{})$. 
If verification succeeds, the program is passed to the \textsc{learn} procedure, which optimizes the parameters of the underlying generative models using the dataset $\data$ and task reward $\reward{}$.  For closed-source models with inaccessible parameters, this step is skipped, and performance improvements rely solely on prompt tuning through the synthesized $\prompt$ programs in each FGGM.
\begin{wrapfigure}{r}{0.51\textwidth}
\vspace{-5pt}
\footnotesize
\captionsetup{type=algorithm}
\resizebox{\linewidth}{!}{%
\begin{minipage}{\linewidth}

\begin{algorithm}[H]
\footnotesize
\caption{ \Tool }
\label{alg:tool}
\begin{algorithmic}[1]
\State \textbf{Input:} $\mathcal I = (\fset{},\axiom{},\reward{},\data{},\inspec{},\outspec{}, I)$, planner $\planner{}$
\State \textbf{Input:} total budget $\Delta$, budget per candidate $\delta$
\State \textbf{Output:} Best agent $\opt{f}$ or $\bot$ if search fails

\State $\mathcal{P} \gets \{\}$ \Comment{Pool of verified candidates}
\State $I' \gets \emptyset$

\While{$\delta \leq \Delta$}
    \State $\program{} \gets \funcIt{{programSearch}}(\fset{},\axiom{},\inspec{},\outspec{}, I, \planner{}$, $\delta, I')$
    \State $\Delta \gets \Delta - \delta$
    \If{$\program{} = \bot$}
        \State \textbf{continue}
    \EndIf
\If{$\set{\Theta} \neq \set{}$}
    \State $(\opt{\theta_1},\dots,\opt{\theta_k}) \gets \funcIt{learn}(\program{}, \reward{}, \data)$
    \State $\mathcal{P} \gets \mathcal{P} \cup \{\programOptSub{}\}$
\Else
    \State $\mathcal{P} \gets \mathcal{P} \cup \{\program{}\}$
\EndIf
    \State $I' \gets \funcIt{collectFeedback}(\program{}, \data)$
\EndWhile

\State $\opt{f} \gets \arg\max_{f \in \mathcal{P}}
\sum_{(x,y)\in\data}\reward{}(x,y,f(x))$

\State \Return $\opt{f}$
\end{algorithmic}
\end{algorithm}
\end{minipage}
}
\vspace{-2.6em}
\end{wrapfigure}

\noindent The resulting tuned agent is added to the candidate pool, and its execution traces on $\data$ are used to construct feedback $I'$ (see Appendix \ref{appen:plannerFeedback}) that guides the planner in subsequent search iterations. 
After the search budget $\Delta$ is exhausted, \Tool returns the agent in the pool that achieves the highest task reward on the dataset. If $\mathcal{P}$ is empty, \Tool returns $\bot$. 
\subsection{Theoretical Results}

\label{subsec:formalProof}
\subsubsection{Soundness} The soundness of \Tool guarantees that any agent program $\opt{f} \neq \bot$ returned satisfies the behavioral specifications $(\inspec{}, \outspec{})$ modulo user provided axioms $\axiom{}$. The \funcIt{programSearch} step generates candidate programs $\program{}$ that satisfy $(\inspec{}, \outspec{})$ for all $(\theta_1, \dots, \theta_k) \in \fullParamSet{}$. Consequently, the parameters $(\opt{\theta_1}, \dots, \opt{\theta_k})$ predicted by the learning step preserve constraint satisfaction. The pool $\mathcal{P}$ is either empty or contains only programs that satisfy the behavioral specifications.

\begin{restatable}[Soundness]{theorem}{soundnessThm}
\label{thm:soundness}
If $\opt{f} \neq \bot$ then, $\axiomEncode{\axiom{}} \implies \forall x_1 \in T_1, \dots, \forall x_n \in T_n .\inspec{}(x_1, \dots, x_n) \implies \outspec{}(x_1, \dots, x_n, \opt{f}(x_1, \dots, x_n))$.
\end{restatable}
\begin{proofSketch}
$(\opt{f} \neq \bot) \implies (\opt{f} \in \mathcal{P})$ and all program $f \in \mathcal{P}$ satisfies $\axiomEncode{\axiom{}} \implies \forall x_1 \in T_1, \dots, \forall x_n \in T_n .\inspec{}(x_1, \dots, x_n) \implies \outspec{}(x_1, \dots, x_n, f(x_1, \dots, x_n))$ following Theorem~\ref{thm:verificationParamSet}. We provide the formal details in Appendix~\ref{appen:proofDetails}.
\end{proofSketch}
\subsubsection{Reward Maximization} In this section, we theoretically evaluate \Tool's strength to maximize task reward $\reward{}$. First, 
we characterize a sufficient condition ensuring the existence of a candidate $f \in \fspace{G}{\fset{}}$ that satisfies $(\inspec{}, \outspec{})$ and achieves no less reward than any generative model $f_n^{\theta} \in \fset{p}$ with any prompting function $\prompt \in \fspace{G}{\fset{c}}$ and any parameter $\param{}$ (Theorem~\ref{thm:sufficientCond}). Secondly, we show that the lower bound of the reward of $f$ grows monotonically with the increase of the modular reward $\reward{}_{a}(\param{i})$ (Theorem~\ref{thm:rewardLowerBound}). This shows improvement in $\reward{}_{a}$ translates to improving task utility.  Finally, we show the proposed FGGM design retains the expressivity of transformer-based models with sufficiently many reasoning steps. Hence, with ideal weights FGGM can simulate any given Turing machine $M$ with sufficiently many auto-regressive steps. 

\begin{definition}[Well-formed Reward]
\label{def:wellFormedReward}
A reward function $\reward{}$ is \emph{well-formed} with respect to $(\inspec{}, \outspec{})$ if it assigns strictly higher reward to constraint-satisfying outputs than constraint-violating ones:
$\forall x \in T_i,\; \forall y, y' \in T_o.\;
(\inspec{}(x) \wedge \neg \outspec{}(x, y) \wedge \outspec{}(x, y'))
\implies \reward{}(x, \_, y') > \reward{}(x, \_, y)$.
\end{definition}


\begin{restatable}[Sufficient Success Condition]{theorem}{sufficientCond}
\label{thm:sufficientCond}
If (i) $\reward{}$ is well-formed (Def.~\ref{def:wellFormedReward}), (ii) $\outspec{}$ is quantifier-free, (iii) there exists a non-parametric program $\fallback \in \fspace{G}{\fset{c}}$ such that $\fallback$ satisfies $(\inspec{}, \outspec{})$.
Then, for any type-correct generative model $f_n^{\theta} : T_i \to T_o$ with parameters $\theta$ and any prompting program $\prompt{}$, there exists a program $f \in \fspace{G}{\fset{}}$ such that: (1) $f$ satisfies $(\inspec{}, \outspec{})$, and (2) $R(f) \geq R(f_n^{\theta})$.
Moreover, if there exists $(x, \_) \in \data$ such that $\neg \outspec{}(x, f_n^{\theta}(x))$, then
$R(f) > R(f_n^{\theta})$ where $R(f) = \sum_{(x, \_) \in \data} \reward{}(x,\_, f(x))$.
\end{restatable}
\begin{proofSketch}  Lemma~\ref{lem:checkerComplete} shows that $\evalb{}$ is complete when $\outspec{}$ is quantifier-free. Define $\fggm{} = (\funcIt{id}, f_n^{\theta}, \typs{}, \linspec{}, \loutspec{}, \prompt, \fallback, \fpost)$ where $(\linspec{}, \loutspec{}) := (\inspec{}, \outspec{})$. Since $\fallback$ satisfies $(\inspec{}, \outspec{})$, the FGGM is valid. Define the program:
\begin{align*}
\kw{function}\ \funcIt{f}(x_1: T_1, \dots, x_n: T_n): T_o\{ \kw{return}\ \funcIt{id}(x_1, \dots, x_n);\}    
\end{align*}

Since $\evalb{}$ is complete, $\funcIt{id}$ returns $f_n^{\theta}(x)$ whenever $\outspec{}(x, f_n^{\theta}(x))$ holds, and otherwise returns $\fallback(x)$, which satisfies $\outspec{}$. For each $(x,\_) \in \data$, $f(x)$ either equals $f_n^{\theta}(x)$ (if valid) or $\fallback(x)$. By assumption (i), such replacement cannot decrease reward, so $R(f) \geq R(f_n^{\theta})$. If $f_n^{\theta}$ violates $\outspec{}$ on some $(x,\_) \in \data$, the improvement is strict. Formal details in Appendix~\ref{appen:proofDetails}.
\end{proofSketch}

\noindent Theorem~\ref{thm:sufficientCond} establishes that the search space $\fspace{G}{\fset{}}$ is rich enough to contain solutions that dominate unconstrained GM calls.

\begin{restatable}[Reward Lower Bound]{theorem}{rewardLowerBound}
\label{thm:rewardLowerBound}
Under the conditions of (ii) and (iii) of Theorem~\ref{thm:sufficientCond} and any (need not be well-formed) $\reward{}$, let $f \in \fspace{G}{\fset{}}$ be the constructed program with FGGM $\fggm{}$ with parameter $\theta \in \Theta$. Let at any training epoch $t$, $\opt{\reward{}_a}(t)$ defines the maximum reward from Eq.~\ref{eq:decomposedLoss} till $t$ with optimal params $\theta_t$. Then:
$\sum_{(x,\_) \in \data} \reward{}(x, \_, f[\theta_t/ \Theta](x))
  \;\geq\;
  \frac{|\data| \times \exp(\opt{\reward{}_a}(t))}{\lambda}$.
\end{restatable}
\begin{proofSketch}
Since $\outspec{}$ is quantifier-free if any of $K$ output of $\fggm{}$ say $y_l$ satisfies $\evalb{}$ then it is the final output, $\mathbb{I}(y = y_l) = 1$ for any $(x, \_) \in \data$. Thus $\log(\lambda \times \reward{}(x, \_, f[\theta_t/ \Theta](x))) \geq \reward{}_a^{\theta_t}(p_x)$, and $\sum_{(x, \_) \in \data} \log(\lambda \times \reward{}(x, \_, f[\theta_t/ \Theta](x))) \geq |\data| \times \opt{\reward{}_a}(t)$; applying the AM--GM inequality yields the result (with $|\mathbb{P}_{\data}| = |\data|$ and $\epsilon$ arbitrarily small). Formal details in Appendix~\ref{appen:proofDetails}.
\end{proofSketch}

\noindent Finally, we show that the FGGM construction does not reduce the theoretical expressivity of the underlying GM. Since the GM inside an FGGM generates without constraints, it can use arbitrarily many autoregressive steps for intermediate reasoning before producing the final answer. By defining $\fpost{}$ with the built-in $\funcIt{split}$ function and an appropriate delimiter, the reasoning tokens are stripped and only the final answer is checked against $\loutspec{}$. This allows the GM to perform chain-of-thought computation of arbitrary polynomial length while the local contract constrains only the extracted output. Consequently, FGGM avoids the expressivity collapse to the DLOGTIME-uniform $TC^0$ class that arises when naive constrained decoding restricts every generated token \cite{crane}.
\begin{restatable}[FGGM Expressivity]{theorem}{expressivityFGGM}
\label{thm:expressivityFGGM}
For any Turing machine $M$ with tape alphabet $\Gamma$ and any valid FGGM $\fggm{}$ with quantifier-free local contract $(\linspec{}, \loutspec{})$, if $M$ computes outputs satisfying $\loutspec{}$ on all inputs satisfying $\linspec{}$, then there exists a constant-depth LLM $\llm{M}$ with log-precision, a delimiter $c \notin \Gamma$, and a post-processor $\fpost{}$ using $\funcIt{split}(c, \cdot)$ such that for any input $\vect{x}$ with $|\vect{x}| = n$, the FGGM with GM $\llm{M}$ returns $M(\vect{x})$ without triggering the fallback, assuming $M$ halts on $\vect{x}$ within $t(n)$ steps where $t(n)$ is polynomial in $n$.
\end{restatable}
\begin{proofSketch}
By existing results on CoT expressivity \cite{cotExpressivity}, there exists a constant-depth LLM that simulates $M$ using $t(n)$ autoregressive steps, producing reasoning tokens followed by $M(\vect{x})$. We insert a fresh delimiter $c$ between the reasoning and the final answer, and define $\fpost{}$ to extract the substring after $c$ via $\funcIt{split}$. The processed output $M(\vect{x})$ satisfies $\loutspec{}$ by assumption, so by completeness of $\eval$ (Lemma~\ref{lem:checkerComplete}), the rejection sampler accepts on the first sample. Formal details in Appendix~\ref{appen:expressivity}.
\end{proofSketch}

\section{Evaluation}
We evaluate \Tool on four tasks spanning scientific discovery, program verification, mathematical reasoning, and agentic tool use. The 
evaluation seeks to answer five questions: whether formal constraints improve safety (RQ1), whether learning remains effective under 
those constraints (RQ2), how conformance tuning with local constraints improves performance (RQ3), how often the verified fallback 
is triggered and whether the gains are attributable to search-space pruning rather than the fallback (RQ4), and whether the stateful 
$\ferr$ feedback loop outperforms stateless rejection sampling (RQ5). We also provide a qualitative analysis of the synthesized agents.

\subsection{Experimental Setup}
\subsubsection{Tasks and Datasets}
We instantiate \Tool on four tasks, each representative of a different kind of constrained agent synthesis problem.

\noindent\textbf{LLM-Assisted Program Verification (DafnyBench).}
The agent receives a Dafny program with pre-encoded input-output specifications and must synthesize annotations that enable verification against the pre-encoded specification inside the input program \cite{dafnyBench}. The behavioral specification of the agent requires the output to remain parsable and equivalent to the input modulo added annotations. The task performance is measured by verification success within a fixed time budget. The underlying LLM is Claude Sonnet 4.5 \cite{claudeSonnet45}, a closed-source model whose weights are not accessible; consequently, the learning step (i.e., parameter tuning) is not applicable to this task.

\noindent\textbf{Symbolic Math Synthesis (GSM-Symbolic).}
GSM-Symbolic is a benchmark of grade-school math word problems generated from symbolic templates \cite{gsmsymbolic}. 
We use it to evaluate the LLM's ability to synthesize correct symbolic mathematical expressions. Behavioral specification of the agent enforces that each generated expression is syntactically valid with respect to a formal grammar, and task performance measures whether the answer is equivalent to the ground-truth expression.

\noindent\textbf{Agentic Tool Use ($\tau^2$-bench).}
$\tau^2$-bench \cite{barres2025tau2} evaluates conversational LLM agents in realistic customer-service scenarios. We use the retail and airline domains. The agent must select and invoke API tools to resolve user requests while respecting domain-specific policies. Unlike the other tasks, the hard constraint and the utility metric operate at different granularities. The hard constraint is enforced at each individual tool call: before every tool call is executed, the Agent-C checker \cite{agentc} verifies that the proposed call complies with temporal policy-compliance rules (e.g., refund eligibility, booking-modification policies) given the trace of prior calls. Task performance (pass rate) is measured over the entire multi-turn interaction trace as the percentage of interactions where the agent successfully resolves the user request.

\noindent\textbf{Constrained Symbolic Regression (SymReg).} 
In symbolic regression, each synthesis instance consists of noisy input–output observations generated by an unknown ground-truth function $f_{gt}$, along with a behavioral specification encoding prior knowledge about $f_{gt}$. The goal is to synthesize a neuro-symbolic program (i.e., the agent) that fits the observed input–output pairs while provably satisfying the behavioral specifications.
We consider a total of $35$ synthesis tasks, each with $2$ different training sets corresponding to varying noise levels $(5\%, 10\%)$, and behavioral specifications that encode symbolic bounds on the ground truth. Each training set includes $600$ noisy samples and $400$ test samples. All synthesis instances are drawn from popular benchmarks~\cite{shapeConstrained, symReg1}.

\subsubsection{Implementation Details}
All experiments are conducted on NVIDIA A100 GPUs (40\,GB).
Across all tasks we use the same high-level pipeline: planner-guided program search with a budget of $\Delta = 10$ iterations, FGGM-based rejection sampling with a fixed sample budget $K=5$, deductive verification with bounded verifier and SMT timeouts, and hyperparameter tuning when model parameters are accessible.
The planner LLM is provided with 3 ground-truth examples as in-context examples (2 for $\tau^2$-bench, due to the long traces); these examples are excluded from all train and test splits. The full planner prompt template is given in Appendix~\ref{appen:plannerPrompt}, and details on how execution feedback is constructed and provided to the planner between iterations are given in Appendix~\ref{appen:plannerFeedback}.
\subsubsection{Learning Implementation}
\label{sec:implement}
Once a program verifies, \Tool uses Dafny's built-in transpiler \cite{dafny2Python} to translate the verified code into Python. Parameter tuning and experiments are then conducted on the transpiled code. \Tool assumes that Dafny’s transpiler is sound for this restricted subset of the language. 

\subsubsection{Metrics}
For each task, we report: (a) the \textit{constraint violation rate} on a held-out test set of unseen inputs, measuring the fraction of outputs that violate the behavioral specification $(\inspec{}, \outspec{})$; (b) \textit{task performance}, measured by mean squared error for SymReg, verification success rate for Dafny, answer accuracy for GSM-Symbolic, and pass rate for $\tau^2$-bench; and (c) \textit{wall-clock time}. 

\subsubsection{Baseline} For all four tasks outside GSM-Symbolic, we consider state-of-the-art (SOTA) hand-crafted agents. For GSM-Symbolic, we use the SOTA constrained decoder \cite{crane}. Some hand-crafted agents (e.g., Agent-C on $\tau^2$-bench) manually enforce hard constraints; however, being static, they achieve lower task performance.
We also consider self-evolving frameworks without constraints, which do not verify the generated programs against behavioral specifications and accept all planner programs provided they are syntactically valid and type-checked. Consequently, such frameworks provide no guarantees about the correctness or safety of the generated code. To ensure fairness, unconstrained self-evolving frameworks are only allowed to use the same set of library functions.

\subsubsection{Verification Success}
For program verification, tool calling, and GSM-Symbolic, \Tool consistently finds a valid program satisfying $(\inspec{}, \outspec{})$ within $10$ attempts. For symbolic regression, it produces a verified program in $33$ out of $35$ cases, whereas the baselines generate at least $11$ cases where the generated program violates the behavioral specification on test inputs.

\renewcommand{\arraystretch}{0.9}
\begin{wraptable}{r}{0.43\linewidth}
\vspace{-12pt}
\centering
\small
\begin{tabular}{@{}lrr@{}}
\toprule
\textbf{Dataset} & \textbf{Total} & \textbf{Train} \\
\midrule
HumanEvalDafny         & 135 & 30 \\
DafnyBench             & 760 & 50 \\
$\tau^2$-bench Retail  & 114 & 15 \\
$\tau^2$-bench Airline &  50 & 10 \\
GSM-Symbolic           & 100 & 50 \\
Symbolic Regression    & 1000 & 600 \\
\bottomrule
\end{tabular}
\caption{Dataset sizes and train/test splits.}
\label{tab:splits}
\vspace{-35pt}
\end{wraptable}
\renewcommand{\arraystretch}{1}
\subsubsection{Dataset splits.}
Table~\ref{tab:splits} summarizes the dataset sizes for each synthesis task.
For each task, a small training set is used for the learning step and the remaining examples form the held-out test set on which all reported metrics are computed.
GSM-Symbolic uses a larger training fraction (50\%) because its learning step performs GRPO-based fine-tuning, which requires more training data than the other tasks.

\subsubsection{Task-specific details.}
For Dafny, the underlying LLM is Claude Sonnet 4.5 \cite{claudeSonnet45}, a closed-source model; because model weights are inaccessible, Dafny experiments use only the search and verification stages of \Tool and do not perform parameter tuning.
For $\tau^2$-bench, parameter tuning is omitted because the long, multi-turn execution traces make gradient-based fine-tuning prohibitively expensive.
For GSM-Symbolic, we use GRPO \cite{grpo} with LoRA adapters \cite{hu2021loralowrankadaptationlarge} on Qwen3-8B \cite{qwen3}. 
Additional hyperparameters (LoRA rank, solver timeouts, etc.)\ are reported in Appendix~\ref{appen:hyperparams}.

\subsection{Main Results}
\label{sec:mainResults}

We begin with the central question motivating \Tool:

\vspace{-5pt}\calloutbox{mygray}{\small{\textbf{RQ1 (Safety).} \emph{What goes wrong without formal constraints, and does \Tool eliminate these failures?}}}\vspace{-5pt}

\noindent We compare unconstrained agent synthesis against \Tool across all four tasks to show that hard behavioral specifications are necessary: without them, agents silently produce invalid outputs on unseen inputs that are undetectable by soft performance metrics.

\subsubsection{LLM-Assisted Program Verification (Dafny)}

\begin{table}[t]
\caption{Dafny program verification results on HumanEvalDafny and DafnyBench.
\emph{Verif.\ \& NoDiff} counts a program as correct only if it both verifies and passes the AST-based diff check against the original; \emph{Verif.} counts verification alone.}
\vspace{-5pt}
\label{tab:dafny}
\centering
\small
\setlength{\tabcolsep}{5pt}%
\begin{tabular}{@{}ll r@{\hspace{5pt}}l r@{\hspace{5pt}}l r@{\hspace{5pt}}l r @{}}
\toprule
\textbf{Dataset} & \textbf{Method}
  & \multicolumn{2}{c}{\textbf{Ver.\ \& NoDiff (\%)}~$\uparrow$}
  & \multicolumn{2}{c}{\textbf{Ver.\ (\%)}~$\uparrow$}
  & \multicolumn{2}{c}{\textbf{Vio.\ (\%)}~$\downarrow$}
  & \textbf{Time (s)} \\
\midrule
\multirow{4}{*}{HumanEvalDafny}
  & LLM (Claude Sonnet 4.5) & \quad\quad73.7 && 76.8 && \redn{8.1} && 9.8 \\
  & DafnyBench baseline     & 86.9 && 87.9 && \redn{4.0} && 16.1 \\
  & \Tool~(w/o constraints) & 84.8 && 88.9 && \redn{5.1} && 15.7 \\
  \rowcolor{lightgrayrow}
  & \Tool                   & \greenn{97.0} & \dt{\greenn{(+10.1)}} & \greenn{97.0} & \dt{\greenn{(+9.1)}} & \greenn{0.0} & \dt{\greenn{(-4.0)}} & 18.2 \\
\midrule
\multirow{4}{*}{DafnyBench}
  & LLM (Claude Sonnet 4.5) & 68.7 && 71.1 && \redn{10.3} && 10.3 \\
  & DafnyBench baseline     & 81.6 && 84.0 && \redn{8.2}  && 20.1 \\
  & \Tool~(w/o constraints) & 79.2 && 84.8 && \redn{7.9}  && 18.4 \\
  \rowcolor{lightgrayrow}
  & \Tool                   & \greenn{89.1} & \dt{\greenn{(+7.5)}} & \greenn{89.1} & \dt{\greenn{(+5.1)}} & \greenn{0.0} & \dt{\greenn{(-8.2)}} & 25.6 \\
\bottomrule
\end{tabular}%

\vspace{-5pt}
\end{table}

Table~\ref{tab:dafny} presents results on two Dafny program-verification benchmarks and offers the clearest illustration of why formal constraints are indispensable.
The gap between the \emph{Verif.} and \emph{Verif.\ \& NoDiff} columns reveals the core safety issue (\textbf{RQ1}): without hard constraints, agents inflate their verification rate by modifying the original program.
The LLM baseline reports 76.8\% verification on HumanEvalDafny, yet only 73.7\% of those outputs actually preserve the input program. Overall, 8.1\% of all outputs are violations that would go undetected by a metric that checks verification alone.
The DafnyBench baseline and \Tool without constraints reduce but do not eliminate this problem (4.0\% and 5.1\% violations, respectively).
Only the full \Tool pipeline, which enforces the \emph{NoDiff} behavioral specification, achieves a 0\% violation rate on both benchmarks while simultaneously attaining the highest Verif.\ \& NoDiff rate (97.0\% on HumanEvalDafny, 89.1\% on DafnyBench).
This confirms that FGGM-based rejection sampling with a verified fallback does not merely filter outputs but actively steers the planner toward higher-quality candidates.
This result is particularly notable because it is achieved without any parameter tuning (the underlying model, Claude Sonnet 4.5, is closed-source), the gains come entirely from making constraints visible to the planner and enforcing them at synthesis time.
The overhead introduced by \Tool is modest. The full pipeline takes 18.2\,s per instance on HumanEvalDafny compared to 9.8\,s for LLM generation, roughly a $1.9\times$ factor that includes Dafny verification and up to $k$ LLM calls. On DafnyBench the factor is $2.5\times$ (25.6\,s vs.\ 10.3\,s). Given that the LLM baseline produces untrustworthy outputs on 8 to 10\% of inputs, we believe this overhead is justified for the formal guarantees it provides.

\subsubsection{Agentic Tool Use ($\tau^2$-bench)}

\begin{table}[t]
\caption{$\tau^2$-bench results with Qwen3-8B on the retail and airline domains.
Higher pass rate is better; lower violation rate is better.}
\label{tab:tau2}
\centering
\small
\begin{tabular}{@{}ll r@{\hspace{5pt}}l r@{\hspace{5pt}}l r @{}}
\toprule
\textbf{Domain} & \textbf{Method}
  & \multicolumn{2}{c}{\textbf{Pass Rate (\%)}~$\uparrow$}
  & \multicolumn{2}{c}{\textbf{Violation (\%)}~$\downarrow$}
  & \textbf{Time (s)} \\
\midrule
\multirow{4}{*}{Retail}
  & LLM (Qwen3-8B)             & 11.3 && \redn{76.3} && 146.6 \\
  & Agent-C (Qwen3-8B)                    & 42.2 && \greenn{0.0} && 234.7 \\
  & \Tool~(w/o constraints)    & 49.4 && \redn{10.3} && 238.3 \\
  \rowcolor{lightgrayrow}
  & \Tool                      & \greenn{53.6} & \dt{\greenn{(+11.4)}} & \greenn{0.0} & \dt{\greenn{(0.0)}} & 212.4 \\
\midrule
\multirow{4}{*}{Airline}
  & LLM (Qwen3-8B)             & 13.2 && \redn{68.4} && 184.8 \\
  & Agent-C (Qwen3-8B)                   & 39.4 && \greenn{0.0} && 272.6 \\
  & \Tool~(w/o constraints)    & 44.7 && \redn{25.5} && 268.4 \\
  \rowcolor{lightgrayrow}
  & \Tool                      & \greenn{52.6} & \dt{\greenn{(+13.2)}} & \greenn{0.0} & \dt{\greenn{(0.0)}} & 241.1 \\
\bottomrule
\end{tabular}

\vspace{-10pt}
\end{table}

Table~\ref{tab:tau2} reports performance on $\tau^2$-bench. The LLM baseline (Qwen3-8B without constraints) violates domain policies on 76.3\% of retail interactions and 68.4\% of airline interactions, yielding pass rates of only 11.3\% and 13.2\%, respectively.
Both Agent-C and \Tool reduce the violation rate to 0\%, confirming that formal policy enforcement is necessary.
\Tool without constraints illustrates an intermediate point: by running the search loop without informing the planner of the policy specification, \Tool~(w/o constraints) improves pass rates to 49.4\% (retail) and 44.7\% (airline), but still incurs 10.3\% and 25.5\% violation rates, respectively. Enforcing constraints eliminates all violations while further improving pass rates to 53.6\% and 52.6\%.
This gap reflects \Tool's search-verify-learn design. By synthesizing agent programs that are verified to satisfy the per-call policy specification, \Tool can explore a wider space of compliant strategies rather than relying solely on runtime constraint checking. The per-call hard constraint ensures that every individual tool call is policy-compliant, while the search loop optimizes the agent's overall strategy to maximize the end-of-trace pass rate. Impressively, \Tool with small Qwen3-8B outperforms Agent-C with Claude Sonnet 4.5 (52.6 vs 47.3) (result from \cite{agentc}).

\noindent \Tool is also faster than Agent-C in both domains (212.4s vs. 234.7s in retail and 241.1s vs. 272.6s in airline). The LLM baseline is fastest because it performs no policy checking, but its outputs are unusable in practice given the violation rates.

\subsubsection{Symbolic Math Synthesis (GSM-Symbolic)}

GSM-Symbolic uses Qwen3-8B as the underlying LLM. Qwen3-8B is open-weight, making parameter tuning feasible. This allows us to ask:

\vspace{-5pt}\calloutbox{mygray}{\small{\textbf{RQ2 (Training Impact).} \emph{Does parameter tuning improve task performance under formal constraints?}}}\vspace{-5pt}

\begin{table}[t]
\caption{GSM-Symbolic results with Qwen3-8B.
Higher accuracy is better; lower violation rate is better. Parenthetical gains are relative to CRANE.}

\vspace{-5pt}
\label{tab:gsm}
\centering
\small
\begin{tabular}{@{}l r@{\hspace{5pt}}l r@{\hspace{5pt}}l r @{}}
\toprule
\textbf{Method}
  & \multicolumn{2}{c}{\textbf{Accuracy (\%)}~$\uparrow$}
  & \multicolumn{2}{c}{\textbf{Violation (\%)}~$\downarrow$}
  & \textbf{Time (s)} \\
\midrule
LLM (Qwen3-8B)                & 38.3 && \redn{10.6} && 10.9 \\
CRANE~\cite{crane}             & 44.7 && \redn{2.1} && 18.4 \\
\Tool~(no parameter tuning)    & \greenn{53.2} & \dt{\greenn{(+8.5)}}  & \greenn{0.0} & \dt{\greenn{(-2.1)}} & 12.8 \\
\rowcolor{lightgrayrow} \Tool~(with parameter tuning)  & \greenn{66.0} & \dt{\greenn{(+21.3)}} & \greenn{0.0} & \dt{\greenn{(-2.1)}} & 11.7 \\
\bottomrule
\end{tabular}

\vspace{-10pt}
\end{table}

\noindent Table~\ref{tab:gsm} compares \Tool against the LLM baseline and CRANE~\cite{crane}, a state-of-the-art constrained-decoding method, on GSM-Symbolic using Qwen3-8B.
The LLM baseline (Qwen3-8B) achieves only 38.3\% accuracy with a 10.6\% constraint violation rate (\textbf{RQ1}), confirming that a vanilla LLM frequently produces outputs that are syntactically or semantically invalid.
CRANE improves accuracy to 44.7\% and reduces violations to 2.1\% by augmenting the output grammar with reasoning tokens. \footnote{CRANE only ensures that at every decoding step, the generated partial output is a valid prefix of the target grammar.}\Tool without parameter tuning already achieves 53.2\% accuracy and eliminates all violations, an 8.5\% accuracy improvement over CRANE with strictly stronger correctness guarantees. This improvement stems from the FGGM mechanism: the rejection sampler with a verified fallback ensures that every output satisfies the formal grammar and the semantic specification, while the synthesized prompting program $\prompt$ guides the underlying LLM toward correct answers.

\noindent We also compare the performance of \Tool with and without parameter tuning (\textbf{RQ2}). We fine-tune Qwen3-8B via GRPO with LoRA adapters under the same formal constraints. This provides a further 12.8\% accuracy gain (from 53.2\% to 66.0\%) while maintaining the zero-violation guarantee. The tuned model produces outputs that more often pass the FGGM checker $\eval$ on the first sample, reducing reliance on the fallback and thereby improving answer quality. Notably, the tuned variant is also faster than the untuned one (11.7\,s vs.\ 12.8\,s), because higher conformance reduces the number of rejection-sampling iterations needed per call.

\subsubsection{Constrained Symbolic Regression}
In constrained symbolic regression, we consider the SOTA genetic programming–based method PySR \cite{pysr}, as well as the existing self-evolving method LLM-SR \cite{llmsr}, which samples parametric programs with optimizable PyTorch parameters and refines them to minimize loss. We denote the noise level by $\epsilon$. As shown in Table~\ref{tab:symreg}, both baselines frequently produce programs that violate the behavioral specifications on the test data, with violations occurring in up to $62.86\%$ of synthesis instances for PySR and up to $34.29\%$ for LLM-SR. For both noise levels, \Tool finds a verified solution in $33$ out of $35$ instances.
Table~\ref{tab:nmse} presents a pairwise comparison between \Tool and the baselines in terms of normalized mean squared error (NMSE) on the test dataset. We compute the average NMSE over only those instances where the baseline-generated programs satisfy the behavioral constraints, pruning instances that correspond to invalid solutions. \Tool achieves significantly lower NMSE values compared to both baselines.
\begin{table}[t]
\centering
\small
\begin{minipage}{0.48\linewidth}
\centering
\caption{\% of synthesized instances that violate behavioral specifications on the test set.}
\label{tab:symreg}
\begin{tabular}{@{}lcc@{}}
\toprule
\textbf{Method}
  & $\epsilon$ (5\%)~$\downarrow$
  & $\epsilon$ (10\%)~$\downarrow$ \\
\midrule
PySR 
  & \redn{62.86\%}
  & \redn{62.86\%} \\
LLM-SR 
  & \redn{31.43\%}
  & \redn{34.29\%} \\
\rowcolor{lightgrayrow} \Tool 
  & \greenn{0.00\%}
  & \greenn{0.00\%} \\
\bottomrule
\end{tabular}
\end{minipage}
\hfill
\begin{minipage}{0.48\linewidth}
\centering
\caption{NMSE values averaged over all instances that do not violate constraints on test data.}
\label{tab:nmse}
\small
\begin{tabular}{lcccc}
\toprule
\multirow{2}{*}{Comparison}
  & \multicolumn{2}{c}{$\epsilon = 0.05$}
  & \multicolumn{2}{c}{$\epsilon = 0.1$} \\
 & {\tiny \Tool} & {\tiny Baseline} & {\tiny \Tool} & {\tiny Baseline} \\
\midrule
vs PySR  
  & 0.0593 & 0.3113
  & 0.0879 & 0.1217 \\
vs LLM-SR 
  & 0.0200 & 0.1580
  & 0.0205 & 0.1583 \\
\bottomrule
\end{tabular}
\end{minipage}
\vspace{-10pt}
\end{table}
\subsection{Constraint Decomposition}
\label{sec:ablation}
The previous section showed that formal constraints improve both safety and performance. We now ask a finer-grained question:

\vspace{-5pt}\calloutbox{mygray}{\small{\textbf{RQ3 (Constraint Decomposition).} \emph{What are the individual contributions of parameter tuning with local FGGM contracts versus global reward $\reward{}$?}}}\vspace{-5pt}

\begin{table}[t]
\caption{Constraint decomposition ablation on GSM-Symbolic (Qwen3-8B).
\emph{Local} = FGGM contracts only; \emph{Global} = behavioral specification only; \emph{Full} = both. Parenthetical gains are relative to the vanilla LLM.}
\label{tab:ablation}
\centering
\small
\begin{tabular}{@{}l r@{\hspace{5pt}}l r@{\hspace{5pt}}l @{}}
\toprule
\textbf{Configuration}
  & \multicolumn{2}{c}{\textbf{Accuracy (\%)}~$\uparrow$}
  & \multicolumn{2}{c}{\textbf{Violation (\%)}~$\downarrow$} \\
\midrule
LLM (Qwen3-8B)                          & 38.3 && \redn{10.6} & \\
\Tool (no parameter tuning)            & \greenn{53.2} & \dt{\greenn{(+14.9)}} & \greenn{0.0} & \dt{\greenn{(-10.6)}} \\
\Tool (local only parameter tuning)    & \greenn{55.3} & \dt{\greenn{(+17.0)}} & \greenn{0.0} & \dt{\greenn{(-10.6)}} \\
\Tool (global only parameter tuning)   & \greenn{61.7} & \dt{\greenn{(+23.4)}} & \greenn{0.0} & \dt{\greenn{(-10.6)}} \\
\rowcolor{lightgrayrow} \Tool (full)                           & \greenn{66.0} & \dt{\greenn{(+27.7)}} & \greenn{0.0} & \dt{\greenn{(-10.6)}} \\
\bottomrule
\end{tabular}

\vspace{-15pt}
\end{table}

\noindent We ablate each component independently on GSM-Symbolic. Table~\ref{tab:ablation} decomposes the 12.8\% gain that parameter tuning provides (from 53.2\% to 66.0\%, as shown in Table~\ref{tab:gsm}) into its two sources.
Starting from \Tool without parameter tuning (53.2\%), adding \emph{local-only} tuning, which optimizes the FGGM conformance reward so that the model's outputs more frequently satisfy per-call contracts, yields a modest improvement to 55.3\%. This 2.1\% gain indicates that parameter tuning alone helps the model produce syntactically valid outputs more reliably, reducing reliance on rejection sampling.
\emph{Global-only} tuning, which optimizes the task reward $\reward{}$ without the FGGM conformance component, provides a substantially larger gain to 61.7\%. This 8.5\% improvement reflects the direct benefit of training the model to produce semantically correct answers.
The full \Tool pipeline combines both and achieves 66.0\%, a 4.3\% gain. This demonstrates that the two training signals are complementary: global tuning improves answer correctness while local conformance tuning ensures outputs satisfy per-call contracts, jointly yielding the best overall accuracy.

\subsection{Fallback Trigger Rate and Impact}
\label{sec:fallback}

\vspace{-5pt}\calloutbox{mygray}{\small{\textbf{RQ4 (Fallback).} \emph{How often does the fallback fire, and how much of \Tool's accuracy gain is attributable to it rather than to search-space pruning?}}}\vspace{-5pt}

\begin{wraptable}{r}{0.52\linewidth}
\vspace{-12pt}
\centering
\small
\setlength{\tabcolsep}{3pt}
\caption{Fallback trigger rate and impact. \emph{Trig.}: fallback produces final output; \emph{FB Acc.}: fallback-only baseline; \emph{Gain$^*$}: gain excluding fallback instances.}
\vspace{-5pt}
\label{tab:fallback}
\begin{tabular}{@{}l r r r r@{}}
\toprule
\textbf{Benchmark}
  & \textbf{Trig.}~$\downarrow$
  & \textbf{FB Acc.}
  & \textbf{Gain}
  & \textbf{Gain$^*$} \\
\midrule
HumanEvalDafny     & 5.1  & 5.1  & \greenn{+10.1} & \greenn{+8.1} \\
DafnyBench         & 7.9  & 29.3 & \greenn{+7.5}  & \greenn{+5.9} \\
$\tau^2$-Retail    & 15.6 & 5.2  & \greenn{+11.4} & \greenn{+10.6} \\
$\tau^2$-Airline   & 21.6 & 5.4  & \greenn{+13.2} & \greenn{+9.7} \\
GSM-Symbolic       & 2.1  & 0.0  & \greenn{+21.3} & \greenn{+21.3} \\
\bottomrule
\end{tabular}
\vspace{-15pt}
\end{wraptable}

\noindent Table~\ref{tab:fallback} reports fallback statistics across all benchmarks. The fallback alone achieves negligible accuracy (0.0--5.4\%; DafnyBench is higher at 29.3\% because returning the unmodified program trivially passes verification on already-correct inputs). Even after excluding fallback instances (\emph{Gain$^*$}), \Tool{} retains substantial gains over the best baseline, confirming that the primary benefit of formal constraints is \emph{search-space pruning}, making shortcuts infeasible forces synthesis toward higher-quality strategies. The fallback fires only when every sampled LLM output violates the specification, serving as a last-resort safety net. Lower trigger rates on GSM-Symbolic (2.1\%) and HumanEvalDafny (5.1\%) also explain the modest wall-clock overhead, since fewer rejection-sampling iterations are needed per call.

\subsection{Stateful vs.\ Stateless Rejection Sampling}
\label{sec:stateful}

\vspace{-5pt}\calloutbox{mygray}{\small{\textbf{RQ5 (Statefulness).} \emph{Does the stateful $\ferr$ feedback loop improve acceptance rate and task performance over a stateless rejection sampler?}}}\vspace{-5pt}

\begin{wraptable}{r}{0.435\linewidth}
\vspace{-12pt}
\centering
\small
\setlength{\tabcolsep}{3pt}
\caption{Stateful vs.\ stateless rejection sampling. \emph{Trig.}: fallback trigger rate; \emph{Perf.}: task performance.}
\vspace{-5pt}
\label{tab:stateful}
\begin{tabular}{@{}ll r r@{}}
\toprule
\textbf{Benchmark} & \textbf{Sampler}
  & \textbf{Trig.}~$\downarrow$
  & \textbf{Perf.}~$\uparrow$ \\
\midrule
\multirow{2}{*}{HumanEvalDafny}
  & Stateless  & 9.1 & 88.8 \\
  & Stateful   & 5.1 & 97.0 \\
\bottomrule
\end{tabular}
\vspace{-8pt}
\end{wraptable}

\noindent The FGGM rejection sampler is \emph{stateful}: on each rejected sample it passes the failing atomic formulas $\ferr$ back to the prompting program so the next draw avoids the same mistake (Fig.~\ref{fig:rejectionSampler}). We ablate this by comparing against a \emph{stateless} variant that redraws independently without $\ferr$ feedback (i.e., $\prompt$ ignores $err$). Table~\ref{tab:stateful} reports fallback trigger rate and task performance for both variants.

\subsection{Qualitative Analysis of Synthesized Programs}
\label{sec:synthesizedAgents}

We examine \emph{what} \Tool synthesizes (full programs in Appendices~\ref{appen:exampleDafnyAgent}, \ref{appen:exampleGSMAgent}, and~\ref{appen:exampleTauAgent}). Exposing the behavioral specification \emph{during} evolution lets \Tool identify each task's specific failure modes and synthesize FGGM components targeted at them, rather than a single generic agent. On Dafny it builds a separate component per failure cause ($\funcIt{diffErrorFGGM}$ for base-program edits, $\funcIt{verifierErrorFGGM}$ for genuine proof failures, and $\funcIt{initialFGGM}$ for lemma-, generic-, or set-specific guidance); on GSM-Symbolic, an escalation ladder of three progressively narrower parse-repair FGGMs ($\funcIt{parseErrorFGGM1}$--$\funcIt{parseErrorFGGM3}$); and on $\tau^2$-bench, a duplicate-action detector ($\funcIt{duplicateFGGM}$) plus a policy-recovery component ($\funcIt{agentCErrorFGGM}$) that inspects the offending tool to supply missing prerequisites. In each case the local contract $\loutspec{}$ turns a failure mode into a synthesis objective, so \Tool specializes to that failure instead of collapsing into a degenerate shortcut.

This error-targeted synthesis is stronger than post-hoc \emph{filtering} of a fixed agent's outputs, which is what the fallback baseline measures and is near-useless on its own (Table~\ref{tab:fallback}, \emph{FB Acc.}\ $0.0$--$5.4\%$). Even after excluding every fallback-triggered instance, the gain over the best baseline remains large ($+8.1$ on HumanEvalDafny, $+10.6$ on $\tau^2$-Retail, $+21.3$ on GSM-Symbolic), so the improvement comes from the specialized program structure rather than from discarding bad outputs. The stateful $\ferr$ feedback (Section~\ref{sec:stateful}) reinforces this by folding the failing atomic formulas of $\loutspec{}$ back into $\prompt$, so successive samples avoid repeating the identified error



\subsection{Discussion}
\label{sec:discussion}
\Tool's runtime overhead is competitive: on $\tau^2$-bench it is \emph{faster} than Agent-C, on Dafny the full pipeline incurs only a modest $1.9$--$2.5\times$ slowdown relative to the LLM baseline, and on GSM-Symbolic just ${\sim}1.1\times$, primarily due to the verification and rejection-sampling steps.
One current limitation is that the formulation is \emph{resource-unaware}. While the behavioral specifications constrain functional correctness, they do not account for computational resources such as LLM calls, token usage, or wall-clock budget. Extending \Tool with resource-aware specifications \cite{resource}, for example by encoding token or call budgets as additional hard constraints within the FGGM framework, is a promising direction for future work.
Another direction is to explore alternative FGGM designs. For example, instead of invoking a fallback when rejection sampling fails, an FGGM component could raise an error that is then handled gracefully by the surrounding program. Conformance tuning would then aim to minimize the frequency of such error cases.

\section{Related Work}

\noindent\textbf{Self-Evolving Agents and Their Risks.}
Recent work has explored automatic agent synthesis and self-improvement, from code-based action unification~\cite{autoAgent1} and open-ended skill libraries~\cite{autoAgent2} to architecture search~\cite{autoAgent3,autoAgent4,autoAgent5}, gradient-free symbolic updates~\cite{selfEvolve1}, co-evolution with world models~\cite{selfEvolve2}, and reinforcement-learning-driven skill refinement~\cite{selfEvolve3,selfEvolve4}.
However, none of these provide formal behavioral guarantees, a gap made urgent by demonstrated risks of coding-agent exploits~\cite{dangerCode1,dangerCode2,dangerCodeArticle}, test-case exploitation~\cite{impossiblebench}, and verification cheating~\cite{dafnyPro}.
\Tool addresses this by imposing and verifying hard formal specifications before deployment.

\noindent\textbf{Runtime Monitoring and Shielding.}
Rather than verifying at synthesis time, runtime approaches enforce safety on observed executions. Shield synthesis from neural policies~\cite{shield1} and temporal-logic runtime checking of tool-call sequences~\cite{agentc} are representative examples, but both are limited to observed traces and cannot guarantee safety on unseen inputs.
\Tool instead verifies the synthesized program for \emph{all} inputs and parameter values at synthesis time. The monitoring checks can be used to define the specification as done with \cite{agentc}.

\noindent\textbf{Neuro-Symbolic Program Synthesis.}
Neuro-symbolic methods, including neural program generation from examples~\cite{neuroSymSynthesis}, type-directed differentiable programming~\cite{houdini}, wake-sleep library learning~\cite{dreamCoder}, vision-language program composition~\cite{viperGPT}, and LLM-guided symbolic regression~\cite{llmsr}, synthesize programs that invoke neural components but lack formal correctness guarantees over the composed system.
\Tool extends this paradigm with FGGMs that bind each model call to a verified local contract, enabling end-to-end verification.

\noindent\textbf{Constrained Decoding.}
Constrained decoding restricts LLM outputs to a formal language through grammar-based~\cite{syncode,grammarAligned,itergen}, diffusion-based~\cite{dingo}, programming-abstraction~\cite{lmql,dspy}, reasoning-preserving~\cite{crane}, type-based~\cite{typeConstrained}, and semantic~\cite{chopChop} techniques.
These approaches require access to the model's decoding internals (precluding closed-source models) and enforce per-token or per-output constraints, not program-level behavioral specifications.
\Tool's FGGM mechanism operates on model outputs via rejection sampling and is model-agnostic.

\noindent\textbf{Deductive Program Synthesis.}
Classical deductive synthesis, including syntax-guided~\cite{sygus}, semantics-guided~\cite{semGus}, counterexample-guided~\cite{cegis}, abstraction-guided~\cite{abstractSynthesis}, and component-based~\cite{componentSynthesis, petrinet, codeHint} approaches, provides correctness guarantees but targets non-parametric programs and does not consider task-specific rewards.

\section{Conclusion}
We presented \Tool, a framework for specification-driven synthesis of self-evolving LLM agents with formal behavioral guarantees. Its core abstraction, the Formally Guarded Generative Model (FGGM), binds each generative-model call to a Hoare-style local contract enforced by a stateful rejection sampler with a verified fallback. Because the fallback discharges the contract for every parameter value, correctness is parameter-independent and survives arbitrary fine-tuning. This lets \Tool split the constrained learning problem into four stages (ideation, search, verification, and learning), so that scalable methods such as GRPO can optimize task reward without putting the specification at risk. We proved that \Tool is sound, gave a sufficient condition under which a verified agent matches or exceeds the reward of the best unconstrained model, and showed that FGGMs retain the expressivity of the underlying transformer. Across program verification, symbolic math synthesis, policy-compliant tool use, and constrained symbolic regression, \Tool attains zero constraint violations on held-out inputs while improving task performance over both unconstrained frameworks and task-specific state of the art. These results suggest that behavioral specifications play a dual role: certifying compliance on all inputs and, by pruning shortcut-laden candidates, steering evolution toward higher-utility agents.

\clearpage
\newpage
\bibliographystyle{ACM-Reference-Format}
\bibliography{references}

@inproceedings{autoAgent1,
author = {Wang, Xingyao and Chen, Yangyi and Yuan, Lifan and Zhang, Yizhe and Li, Yunzhu and Peng, Hao and Ji, Heng},
title = {Executable code actions elicit better LLM agents},
year = {2024},
publisher = {JMLR.org},
booktitle = {Proceedings of the 41st International Conference on Machine Learning},
articleno = {2054},
numpages = {25},
location = {Vienna, Austria},
series = {ICML'24}
}

@misc{autoAgent2,
      title={Voyager: An Open-Ended Embodied Agent with Large Language Models}, 
      author={Guanzhi Wang and Yuqi Xie and Yunfan Jiang and Ajay Mandlekar and Chaowei Xiao and Yuke Zhu and Linxi Fan and Anima Anandkumar},
      year={2023},
      eprint={2305.16291},
      archivePrefix={arXiv},
      primaryClass={cs.AI},
      url={https://arxiv.org/abs/2305.16291}, 
}

@inproceedings{
autoAgent3,
title={Automated Design of Agentic Systems},
author={Shengran Hu and Cong Lu and Jeff Clune},
booktitle={The Thirteenth International Conference on Learning Representations},
year={2025},
url={https://openreview.net/forum?id=t9U3LW7JVX}
}

@inproceedings{viperGPT,
  author={Surís, Dídac and Menon, Sachit and Vondrick, Carl},
  booktitle={2023 IEEE/CVF International Conference on Computer Vision (ICCV)}, 
  title={ViperGPT: Visual Inference via Python Execution for Reasoning}, 
  year={2023},
  volume={},
  number={},
  pages={11854-11864},
  doi={10.1109/ICCV51070.2023.01092}
}

@inproceedings{
autoAgent4,
title={{GPTS}warm: Language Agents as Optimizable Graphs},
author={Mingchen Zhuge and Wenyi Wang and Louis Kirsch and Francesco Faccio and Dmitrii Khizbullin and J{\"u}rgen Schmidhuber},
booktitle={Forty-first International Conference on Machine Learning},
year={2024},
url={https://openreview.net/forum?id=uTC9AFXIhg}
}

@misc{autoAgent5,
      title={FlowReasoner: Reinforcing Query-Level Meta-Agents}, 
      author={Hongcheng Gao and Yue Liu and Yufei He and Longxu Dou and Chao Du and Zhijie Deng and Bryan Hooi and Min Lin and Tianyu Pang},
      year={2025},
      eprint={2504.15257},
      archivePrefix={arXiv},
      primaryClass={cs.AI},
      url={https://arxiv.org/abs/2504.15257}, 
}

@misc{selfEvolve1,
      title={Symbolic Learning Enables Self-Evolving Agents}, 
      author={Wangchunshu Zhou and Yixin Ou and Shengwei Ding and Long Li and Jialong Wu and Tiannan Wang and Jiamin Chen and Shuai Wang and Xiaohua Xu and Ningyu Zhang and Huajun Chen and Yuchen Eleanor Jiang},
      year={2024},
      eprint={2406.18532},
      archivePrefix={arXiv},
      primaryClass={cs.CL},
      url={https://arxiv.org/abs/2406.18532}, 
}

@article{selfEvolve2,
  title={WebEvolver: Enhancing Web Agent Self-Improvement with Coevolving World Model},
  author={Fang, Tianqing and Zhang, Hongming and Zhang, Zhisong and Ma, Kaixin and Yu, Wenhao and Mi, Haitao and Yu, Dong},
  year={2025},
  journal={arXiv preprint arXiv:2504.21024}
}

@misc{selfEvolve3,
      title={Reinforcement Learning for Self-Improving Agent with Skill Library}, 
      author={Jiongxiao Wang and Qiaojing Yan and Yawei Wang and Yijun Tian and Soumya Smruti Mishra and Zhichao Xu and Megha Gandhi and Panpan Xu and Lin Lee Cheong},
      year={2026},
      eprint={2512.17102},
      archivePrefix={arXiv},
      primaryClass={cs.AI},
      url={https://arxiv.org/abs/2512.17102}, 
}

@misc{selfEvolve4,
      title={SEEA-R1: Tree-Structured Reinforcement Fine-Tuning for Self-Evolving Embodied Agents}, 
      author={Wanxin Tian and Shijie Zhang and Kevin Zhang and Xiaowei Chi and Chunkai Fan and Junyu Lu and Yulin Luo and Qiang Zhou and Yiming Zhao and Ning Liu and Siyu Lin and Zhiyuan Qin and Xiaozhu Ju and Shanghang Zhang and Jian Tang},
      year={2025},
      eprint={2506.21669},
      archivePrefix={arXiv},
      primaryClass={cs.AI},
      url={https://arxiv.org/abs/2506.21669}, 
}

@inproceedings{
dangerCode1,
title={RedCode: Risky Code Execution and Generation Benchmark for Code Agents},
author={Chengquan Guo and Xun Liu and Chulin Xie and Andy Zhou and Yi Zeng and Zinan Lin and Dawn Song and Bo Li},
booktitle={The Thirty-eight Conference on Neural Information Processing Systems Datasets and Benchmarks Track},
year={2024},
url={https://openreview.net/forum?id=mAG68wdggA}
}

@misc{dangerCode2,
  title={Takedown: How It's Done in Modern Coding Agent Exploits}, 
  author={Eunkyu Lee and Donghyeon Kim and Wonyoung Kim and Insu Yun},
  year={2025},
  eprint={2509.24240},
  archivePrefix={arXiv},
  primaryClass={cs.CR},
  url={https://arxiv.org/abs/2509.24240}, 
}

@misc{dangerCodeArticle,
  title={Rogue AI agents published passwords and bypassed security protections},
  author={The Guardian},
  year={2026},
  howpublished={News investigation}
}

@INPROCEEDINGS{sygus,
  author={Alur, Rajeev and Bodik, Rastislav and Juniwal, Garvit and Martin, Milo M. K. and Raghothaman, Mukund and Seshia, Sanjit A. and Singh, Rishabh and Solar-Lezama, Armando and Torlak, Emina and Udupa, Abhishek},
  booktitle={2013 Formal Methods in Computer-Aided Design}, 
  title={Syntax-guided synthesis}, 
  year={2013},
  volume={},
  number={},
  pages={1-8},
  doi={10.1109/FMCAD.2013.6679385}
}

@misc{pysr,
      title={Interpretable Machine Learning for Science with PySR and SymbolicRegression.jl}, 
      author={Miles Cranmer},
      year={2023},
      eprint={2305.01582},
      archivePrefix={arXiv},
      primaryClass={astro-ph.IM},
      url={https://arxiv.org/abs/2305.01582}, 
}

@inproceedings{
llmsr,
title={{LLM}-{SR}: Scientific Equation Discovery via Programming with Large Language Models},
author={Parshin Shojaee and Kazem Meidani and Shashank Gupta and Amir Barati Farimani and Chandan K. Reddy},
booktitle={The Thirteenth International Conference on Learning Representations},
year={2025},
url={https://openreview.net/forum?id=m2nmp8P5in}
}

@misc{dafny2Python,
  author       = {{Dafny Language Community}},
  title        = {Integrating Dafny and Python Code},
  howpublished = {\url{https://dafny.org/v3.10.0/DafnyRef/integration-py/IntegrationPython}},
  year         = {2023},
  note         = {Accessed: 2026-03-17}
}

@article{semGus,
author = {Kim, Jinwoo and Hu, Qinheping and D'Antoni, Loris and Reps, Thomas},
title = {Semantics-guided synthesis},
year = {2021},
issue_date = {January 2021},
publisher = {Association for Computing Machinery},
address = {New York, NY, USA},
volume = {5},
number = {POPL},
url = {https://doi.org/10.1145/3434311},
doi = {10.1145/3434311},
journal = {Proc. ACM Program. Lang.},
month = jan,
articleno = {30},
numpages = {32}
}

@article{abstractSynthesis,
author = {Guria, Sankha Narayan and Foster, Jeffrey S. and Van Horn, David},
title = {Absynthe: Abstract Interpretation-Guided Synthesis},
year = {2023},
issue_date = {June 2023},
publisher = {Association for Computing Machinery},
address = {New York, NY, USA},
volume = {7},
number = {PLDI},
url = {https://doi.org/10.1145/3591285},
doi = {10.1145/3591285},
journal = {Proc. ACM Program. Lang.},
month = jun,
articleno = {171},
numpages = {24}
}

@article{cegis,
  author  = {Armando Solar-Lezama},
  title   = {Program Sketching},
  journal = {International Journal on Software Tools for Technology Transfer},
  volume  = {15},
  number  = {5},
  pages   = {475--495},
  year    = {2013},
  doi     = {10.1007/s10009-012-0249-7},
  url     = {https://doi.org/10.1007/s10009-012-0249-7},
  issn    = {1433-2787}
}

@inproceedings{componentSynthesis,
author = {Feng, Yu and Martins, Ruben and Van Geffen, Jacob and Dillig, Isil and Chaudhuri, Swarat},
title = {Component-based synthesis of table consolidation and transformation tasks from examples},
year = {2017},
isbn = {9781450349888},
publisher = {Association for Computing Machinery},
address = {New York, NY, USA},
url = {https://doi.org/10.1145/3062341.3062351},
doi = {10.1145/3062341.3062351},
booktitle = {Proceedings of the 38th ACM SIGPLAN Conference on Programming Language Design and Implementation},
pages = {422–436},
numpages = {15},
location = {Barcelona, Spain},
series = {PLDI 2017}
}

@misc{grpo,
      title={DeepSeekMath: Pushing the Limits of Mathematical Reasoning in Open Language Models}, 
      author={Zhihong Shao and Peiyi Wang and Qihao Zhu and Runxin Xu and Junxiao Song and Xiao Bi and Haowei Zhang and Mingchuan Zhang and Y. K. Li and Y. Wu and Daya Guo},
      year={2024},
      eprint={2402.03300},
      archivePrefix={arXiv},
      primaryClass={cs.CL},
      url={https://arxiv.org/abs/2402.03300}, 
}

@inproceedings{houdini,
author = {Valkov, Lazar and Chaudhari, Dipak and Srivastava, Akash and Sutton, Charles and Chaudhuri, Swarat},
title = {HOUDINI: lifelong learning as program synthesis},
year = {2018},
publisher = {Curran Associates Inc.},
address = {Red Hook, NY, USA},
booktitle = {Proceedings of the 32nd International Conference on Neural Information Processing Systems},
pages = {8701–8712},
numpages = {12},
location = {Montr\'{e}al, Canada},
series = {NIPS'18}
}

@article{syncode,
title={SynCode: {LLM} Generation with Grammar Augmentation},
author={Shubham Ugare and Tarun Suresh and Hangoo Kang and Sasa Misailovic and Gagandeep Singh},
journal={Transactions on Machine Learning Research},
issn={2835-8856},
year={2025},
url={https://openreview.net/forum?id=HiUZtgAPoH},
note={}
}

@inproceedings{grammarAligned,
title={Grammar-Aligned Decoding},
author={Kanghee Park and Jiayu Wang and Taylor Berg-Kirkpatrick and Nadia Polikarpova and Loris D'Antoni},
booktitle={The Thirty-eighth Annual Conference on Neural Information Processing Systems},
year={2024},
url={https://openreview.net/forum?id=5G7ve8E1Lu}
}

@inproceedings{
    dingo,
    title={{DINGO}: Constrained Inference for Diffusion {LLM}s},
    author={Tarun Suresh and Debangshu Banerjee and Shubham Ugare and Sasa Misailovic and Gagandeep Singh},
    booktitle={The Thirty-ninth Annual Conference on Neural Information Processing Systems},
    year={2025},
    url={https://openreview.net/forum?id=KaYMGsnZ4R}
}

@inproceedings{itergen,
title={IterGen: Iterative Semantic-aware Structured {LLM} Generation with Backtracking},
author={Shubham Ugare and Rohan Gumaste and Tarun Suresh and Gagandeep Singh and Sasa Misailovic},
booktitle={The Thirteenth International Conference on Learning Representations},
year={2025},
url={https://openreview.net/forum?id=ac93gRzxxV}
}

@article{symRegImprove,
    author = {Kronberger, G. and de Franca, F. O. and Burlacu, B. and Haider, C. and Kommenda, M.},
    title = {Shape-Constrained Symbolic Regression—Improving Extrapolation
                    with Prior Knowledge},
    journal = {Evolutionary Computation},
    volume = {30},
    number = {1},
    pages = {75-98},
    year = {2022},
    month = {03},
    issn = {1063-6560},
    doi = {10.1162/evco_a_00294},
    url = {https://doi.org/10.1162/evco_a_00294},
    eprint = {https://direct.mit.edu/evco/article-pdf/30/1/75/1995582/evco\_a\_00294.pdf},
}

@misc{agentc,
      title={Enforcing Temporal Constraints for LLM Agents}, 
      author={Adharsh Kamath and Sishen Zhang and Calvin Xu and Shubham Ugare and Gagandeep Singh and Sasa Misailovic},
      year={2025},
      eprint={2512.23738},
      archivePrefix={arXiv},
      primaryClass={cs.PL},
      url={https://arxiv.org/abs/2512.23738}, 
}

@inproceedings{dafny,
author = {Leino, K. Rustan M.},
title = {Dafny: an automatic program verifier for functional correctness},
year = {2010},
isbn = {3642175104},
publisher = {Springer-Verlag},
address = {Berlin, Heidelberg},
booktitle = {Proceedings of the 16th International Conference on Logic for Programming, Artificial Intelligence, and Reasoning},
pages = {348–370},
numpages = {23},
location = {Dakar, Senegal},
series = {LPAR'10}
}

@inproceedings{
constantSymReg,
title={Transformer-based model for symbolic regression via joint supervised learning},
author={Wenqiang Li and Weijun Li and Linjun Sun and Min Wu and Lina Yu and Jingyi Liu and Yanjie Li and Songsong Tian},
booktitle={The Eleventh International Conference on Learning Representations },
year={2023},
url={https://openreview.net/forum?id=ULzyv9M1j5}
}

@article{lmql,
author = {Beurer-Kellner, Luca and Fischer, Marc and Vechev, Martin},
title = {Prompting Is Programming: A Query Language for Large Language Models},
year = {2023},
issue_date = {June 2023},
publisher = {Association for Computing Machinery},
address = {New York, NY, USA},
volume = {7},
number = {PLDI},
url = {https://doi.org/10.1145/3591300},
doi = {10.1145/3591300},
journal = {Proc. ACM Program. Lang.},
month = jun,
articleno = {186},
numpages = {24}
}

@inproceedings{
dspy,
title={{DSP}y: Compiling Declarative Language Model Calls into State-of-the-Art Pipelines},
author={Omar Khattab and Arnav Singhvi and Paridhi Maheshwari and Zhiyuan Zhang and Keshav Santhanam and Sri Vardhamanan A and Saiful Haq and Ashutosh Sharma and Thomas T. Joshi and Hanna Moazam and Heather Miller and Matei Zaharia and Christopher Potts},
booktitle={The Twelfth International Conference on Learning Representations},
year={2024},
url={https://openreview.net/forum?id=sY5N0zY5Od}
}

@inproceedings{symReg1,
author = {B\l{}\k{a}dek, Iwo and Krawiec, Krzysztof},
title = {Solving symbolic regression problems with formal constraints},
year = {2019},
isbn = {9781450361118},
publisher = {Association for Computing Machinery},
address = {New York, NY, USA},
url = {https://doi.org/10.1145/3321707.3321743},
doi = {10.1145/3321707.3321743},
booktitle = {Proceedings of the Genetic and Evolutionary Computation Conference},
pages = {977–984},
numpages = {8},
location = {Prague, Czech Republic},
series = {GECCO '19}
}

@misc{asympoticSymReg,
      title={Neural-Guided Symbolic Regression with Asymptotic Constraints}, 
      author={Li Li and Minjie Fan and Rishabh Singh and Patrick Riley},
      year={2019},
      eprint={1901.07714},
      archivePrefix={arXiv},
      primaryClass={cs.LG},
      url={https://arxiv.org/abs/1901.07714}, 

}

@article{
dafnyBench,
title={DafnyBench: A Benchmark for Formal Software Verification},
author={Chloe R Loughridge and Qinyi Sun and Seth Ahrenbach and Federico Cassano and Chuyue Sun and Ying Sheng and Anish Mudide and Md Rakib Hossain Misu and Nada Amin and Max Tegmark},
journal={Transactions on Machine Learning Research},
issn={2835-8856},
year={2025},
url={https://openreview.net/forum?id=yBgTVWccIx},
note={}
}

@article{laurel,
author = {Mugnier, Eric and Gonzalez, Emmanuel Anaya and Polikarpova, Nadia and Jhala, Ranjit and Yuanyuan, Zhou},
title = {Laurel: Unblocking Automated Verification with Large Language Models},
year = {2025},
issue_date = {April 2025},
publisher = {Association for Computing Machinery},
address = {New York, NY, USA},
volume = {9},
number = {OOPSLA1},
url = {https://doi.org/10.1145/3720499},
doi = {10.1145/3720499},
journal = {Proc. ACM Program. Lang.},
month = apr,
articleno = {134},
numpages = {27}
}

@inproceedings{clover,
author = {Sun, Chuyue and Sheng, Ying and Padon, Oded and Barrett, Clark},
title = {Clover: Closed-Loop Verifiable Code Generation},
year = {2024},
isbn = {978-3-031-65111-3},
publisher = {Springer-Verlag},
address = {Berlin, Heidelberg},
url = {https://doi.org/10.1007/978-3-031-65112-0\_7},
doi = {10.1007/978-3-031-65112-0\_7},
booktitle = {AI Verification: First International Symposium, SAIV 2024, Montreal, QC, Canada, July 22–23, 2024, Proceedings},
pages = {134–155},
numpages = {22},
location = {Montreal, QC, Canada}
}

@inproceedings{dreamCoder,
author = {Ellis, Kevin and Wong, Catherine and Nye, Maxwell and Sabl\'{e}-Meyer, Mathias and Morales, Lucas and Hewitt, Luke and Cary, Luc and Solar-Lezama, Armando and Tenenbaum, Joshua B.},
title = {DreamCoder: bootstrapping inductive program synthesis with wake-sleep library learning},
year = {2021},
isbn = {9781450383912},
publisher = {Association for Computing Machinery},
address = {New York, NY, USA},
url = {https://doi.org/10.1145/3453483.3454080},
doi = {10.1145/3453483.3454080},
booktitle = {Proceedings of the 42nd ACM SIGPLAN International Conference on Programming Language Design and Implementation},
pages = {835–850},
numpages = {16},
location = {Virtual, Canada},
series = {PLDI 2021}
}

@inproceedings{resource,
author = {Knoth, Tristan and Wang, Di and Polikarpova, Nadia and Hoffmann, Jan},
title = {Resource-guided program synthesis},
year = {2019},
isbn = {9781450367127},
publisher = {Association for Computing Machinery},
address = {New York, NY, USA},
url = {https://doi.org/10.1145/3314221.3314602},
doi = {10.1145/3314221.3314602},
booktitle = {Proceedings of the 40th ACM SIGPLAN Conference on Programming Language Design and Implementation},
pages = {253–268},
numpages = {16},
location = {Phoenix, AZ, USA},
series = {PLDI 2019}
}

@inproceedings{codeHint,
author = {Galenson, Joel and Reames, Philip and Bodik, Rastislav and Hartmann, Bj\"{o}rn and Sen, Koushik},
title = {CodeHint: dynamic and interactive synthesis of code snippets},
year = {2014},
isbn = {9781450327565},
publisher = {Association for Computing Machinery},
address = {New York, NY, USA},
url = {https://doi.org/10.1145/2568225.2568250},
doi = {10.1145/2568225.2568250},
booktitle = {Proceedings of the 36th International Conference on Software Engineering},
pages = {653–663},
numpages = {11},
location = {Hyderabad, India},
series = {ICSE 2014}
}

@article{petrinet,
author = {Feng, Yu and Martins, Ruben and Wang, Yuepeng and Dillig, Isil and Reps, Thomas W.},
title = {Component-based synthesis for complex APIs},
year = {2017},
issue_date = {January 2017},
publisher = {Association for Computing Machinery},
address = {New York, NY, USA},
volume = {52},
number = {1},
issn = {0362-1340},
url = {https://doi.org/10.1145/3093333.3009851},
doi = {10.1145/3093333.3009851},
journal = {SIGPLAN Not.},
month = jan,
pages = {599–612},
numpages = {14}
}

@inproceedings{
tang2024code,
title={Code Repair with {LLM}s gives an Exploration-Exploitation Tradeoff},
author={Hao Tang and Keya Hu and Jin Peng Zhou and Si Cheng Zhong and Wei-Long Zheng and Xujie Si and Kevin Ellis},
booktitle={The Thirty-eighth Annual Conference on Neural Information Processing Systems},
year={2024},
url={https://openreview.net/forum?id=o863gX6DxA}
}

@inproceedings{shapeConstrained,
author = {Haider, Christian and Kronberger, Gabriel},
title = {Shape-Constrained Symbolic Regression with NSGA-III},
year = {2022},
isbn = {978-3-031-25311-9},
publisher = {Springer-Verlag},
address = {Berlin, Heidelberg},
url = {https://doi.org/10.1007/978-3-031-25312-6_19},
doi = {10.1007/978-3-031-25312-6_19},
booktitle = {Computer Aided Systems Theory – EUROCAST 2022: 18th International Conference, Las Palmas de Gran Canaria, Spain, February 20–25, 2022, Revised Selected Papers},
pages = {164–172},
numpages = {9},
location = {Las Palmas de Gran Canaria, Spain}
}

@article{typeConstrained,
author = {M\"{u}ndler, Niels and He, Jingxuan and Wang, Hao and Sen, Koushik and Song, Dawn and Vechev, Martin},
title = {Type-Constrained Code Generation with Language Models},
year = {2025},
issue_date = {June 2025},
publisher = {Association for Computing Machinery},
address = {New York, NY, USA},
volume = {9},
number = {PLDI},
url = {https://doi.org/10.1145/3729274},
doi = {10.1145/3729274},
journal = {Proc. ACM Program. Lang.},
month = jun,
articleno = {171},
numpages = {26}
}

@article{chopChop,
author = {Nagy, Shaan and Zhou, Timothy and Polikarpova, Nadia and D'Antoni, Loris},
title = {ChopChop: A Programmable Framework for Semantically Constraining the Output of Language Models},
year = {2026},
issue_date = {January 2026},
publisher = {Association for Computing Machinery},
address = {New York, NY, USA},
volume = {10},
number = {POPL},
url = {https://doi.org/10.1145/3776708},
doi = {10.1145/3776708},
journal = {Proc. ACM Program. Lang.},
month = jan,
articleno = {66},
numpages = {28}
}

@misc{dafnyPro,
      title={DafnyPro: LLM-Assisted Automated Verification for Dafny Programs}, 
      author={Debangshu Banerjee and Olivier Bouissou and Stefan Zetzsche},
      year={2026},
      eprint={2601.05385},
      archivePrefix={arXiv},
      primaryClass={cs.SE},
      url={https://arxiv.org/abs/2601.05385}, 
}

@inproceedings{crane,
title={{CRANE}: Reasoning with constrained {LLM} generation},
author={Debangshu Banerjee and Tarun Suresh and Shubham Ugare and Sasa Misailovic and Gagandeep Singh},
booktitle={Forty-second International Conference on Machine Learning},
year={2025},
url={https://openreview.net/forum?id=wKs9fHYxCV}
}

@misc{neuroSymSynthesis,
      title={Neuro-Symbolic Program Synthesis}, 
      author={Emilio Parisotto and Abdel-rahman Mohamed and Rishabh Singh and Lihong Li and Dengyong Zhou and Pushmeet Kohli},
      year={2016},
      eprint={1611.01855},
      archivePrefix={arXiv},
      primaryClass={cs.AI},
      url={https://arxiv.org/abs/1611.01855}, 
}

@inproceedings{shield1,
author = {Zhu, He and Xiong, Zikang and Magill, Stephen and Jagannathan, Suresh},
title = {An inductive synthesis framework for verifiable reinforcement learning},
year = {2019},
isbn = {9781450367127},
publisher = {Association for Computing Machinery},
address = {New York, NY, USA},
url = {https://doi.org/10.1145/3314221.3314638},
doi = {10.1145/3314221.3314638},
booktitle = {Proceedings of the 40th ACM SIGPLAN Conference on Programming Language Design and Implementation},
pages = {686–701},
numpages = {16},
location = {Phoenix, AZ, USA},
series = {PLDI 2019}
}

@article{impossiblebench,
  title={ImpossibleBench: Measuring LLMs' Propensity of Exploiting Test Cases},
  author={Zhong, Ziqian and Raghunathan, Aditi and Carlini, Nicholas},
  journal={arXiv preprint arXiv:2510.20270},
  year={2025}
}

@misc{barres2025tau2,
      title={$\tau^2$-Bench: Evaluating Conversational Agents in a Dual-Control Environment}, 
      author={Victor Barres and Honghua Dong and Soham Ray and Xujie Si and Karthik Narasimhan},
      year={2025},
      eprint={2506.07982},
      archivePrefix={arXiv},
      primaryClass={cs.AI},
      url={https://arxiv.org/abs/2506.07982}, 
}

@article{gsmsymbolic,
  title={Gsm-symbolic: Understanding the limitations of mathematical reasoning in large language models},
  author={Mirzadeh, Iman and Alizadeh, Keivan and Shahrokhi, Hooman and Tuzel, Oncel and Bengio, Samy and Farajtabar, Mehrdad},
  journal={arXiv preprint arXiv:2410.05229},
  year={2024}
}

@misc{hu2021loralowrankadaptationlarge,
      title={LoRA: Low-Rank Adaptation of Large Language Models}, 
      author={Edward J. Hu and Yelong Shen and Phillip Wallis and Zeyuan Allen-Zhu and Yuanzhi Li and Shean Wang and Lu Wang and Weizhu Chen},
      year={2021},
      eprint={2106.09685},
      archivePrefix={arXiv},
      primaryClass={cs.CL},
      url={https://arxiv.org/abs/2106.09685}, 
}

@misc{claudeSonnet45,
      title={System Card: Claude Sonnet 4.5},
      author={Anthropic},
      year={2025},
      howpublished={\url{https://www-cdn.anthropic.com/963373e433e489a87a10c823c52a0a013e9172dd.pdf}},
      note={Accessed: 2026-03-18}
}

@misc{qwen3,
      title={Qwen3 Technical Report},
      author={Qwen Team},
      year={2025},
      eprint={2505.09388},
      archivePrefix={arXiv},
      primaryClass={cs.CL},
      url={https://arxiv.org/abs/2505.09388}
}

@inproceedings{cotExpressivity,
title={The Expressive Power of Transformers with Chain of Thought},
author={William Merrill and Ashish Sabharwal},
booktitle={The Twelfth International Conference on Learning Representations},
year={2024},
url={https://openreview.net/forum?id=NjNGlPh8Wh}
}

@inproceedings{cotExpressivity2,
title={Chain of Thought Empowers Transformers to Solve Inherently Serial Problems},
author={Zhiyuan Li and Hong Liu and Denny Zhou and Tengyu Ma},
booktitle={The Twelfth International Conference on Learning Representations},
year={2024},
url={https://openreview.net/forum?id=3EWTEy9MTM}
}
\clearpage
\newpage
\appendix
\section{Restricted Dafny Grammar}
\label{appen:seachSpace}
{\small
\[
\begin{array}{rcl}

\langle program \rangle 
& ::= & 
\langle method\_decl \rangle 
\mid 
\langle function\_decl \rangle 
\\[0.3ex]

\langle method\_decl \rangle 
& ::= & 
\textcolor{blue}{\texttt{method}}\ \langle ident \rangle
\texttt{(} \langle params \rangle \texttt{)}
\ \langle specs \rangle \
\texttt{\{} 
    \langle stmt\_list \rangle 
\texttt{\}} 
\\[0.3ex]

\langle function\_decl \rangle 
& ::= & 
\textcolor{blue}{\texttt{function}}\ \langle ident \rangle
\texttt{(} \langle params \rangle \texttt{)} :
\langle type \rangle
\ \langle specs \rangle
\ \langle expr \rangle
\\[0.3ex]

\langle params \rangle 
& ::= & 
\epsilon 
\mid \langle param \rangle 
\mid \langle param \rangle , \langle params \rangle 
\\[0.3ex]

\langle param \rangle 
& ::= & \langle ident \rangle : \langle type \rangle 
\\[0.3ex]

\langle type \rangle 
& ::= & 
\textcolor{blue}{\texttt{int}} 
\mid \textcolor{blue}{\texttt{bool}} 
\mid \textcolor{blue}{\texttt{string}} 
\mid \textcolor{blue}{\texttt{real}} 
\mid \textcolor{blue}{\texttt{arr}}\langle \langle type \rangle \rangle 
\\[0.3ex]

\langle specs \rangle 
& ::= & 
\epsilon 
\mid \langle spec \rangle \langle specs \rangle 
\\[0.3ex]

\langle spec \rangle 
& ::= & 
\textcolor{blue}{\texttt{requires}}\ \langle formula \rangle 
\mid 
\textcolor{blue}{\texttt{ensures}}\ \langle formula \rangle 
\\[0.3ex]

\langle stmt\_list \rangle 
& ::= & 
\epsilon 
\mid \langle stmt \rangle \langle stmt\_list \rangle 
\\[0.3ex]

\langle stmt \rangle 
& ::= & 
\langle assign\_stmt \rangle
\mid \langle if\_stmt \rangle
\mid \langle while\_stmt \rangle
\mid \langle assert\_stmt \rangle
\mid \langle call\_stmt \rangle 
\\[0.3ex]

\langle assign\_stmt \rangle 
& ::= & 
\langle ident \rangle := \langle expr \rangle ; 
\\[0.3ex]

\langle call\_stmt \rangle 
& ::= & 
\langle ident \rangle \texttt{(} \langle args \rangle \texttt{)} ; 
\\[0.3ex]

\langle args \rangle 
& ::= & 
\epsilon
\mid \langle expr \rangle
\mid \langle expr \rangle , \langle args \rangle 
\\[0.3ex]

\langle if\_stmt \rangle 
& ::= & 
\textcolor{blue}{\texttt{if}}\ \texttt{(} \langle formula \rangle \texttt{)}
\texttt{\{} \langle stmt\_list \rangle \texttt{\}} 
\\
& & 
\mid 
\textcolor{blue}{\texttt{if}}\ \texttt{(} \langle formula \rangle \texttt{)}
\texttt{\{} \langle stmt\_list \rangle \texttt{\}}
\ \textcolor{blue}{\texttt{else}}\ 
\texttt{\{} \langle stmt\_list \rangle \texttt{\}} 
\\[0.3ex]

\langle while\_stmt \rangle 
& ::= & 
\textcolor{blue}{\texttt{while}}\ \texttt{(} \langle formula \rangle \texttt{)}
\ \langle loop\_annots \rangle 
\ \texttt{\{} \langle stmt\_list \rangle \texttt{\}} 
\\[0.3ex]

\langle loop\_annots \rangle 
& ::= & 
\epsilon 
\mid \langle loop\_annot \rangle \langle loop\_annots \rangle 
\\[0.3ex]

\langle loop\_annot \rangle 
& ::= & 
\textcolor{blue}{\texttt{invariant}}\ \langle formula \rangle
\mid 
\textcolor{blue}{\texttt{decreases}}\ \langle expr \rangle 
\\[0.3ex]

\langle assert\_stmt \rangle 
& ::= & 
\textcolor{blue}{\texttt{assert}}\ \langle formula \rangle ; 
\\[0.3ex]

\langle expr \rangle 
& ::= & 
\langle aexpr \rangle \mid \langle sexpr \rangle \mid \langle arrexpr \rangle 
\\[0.3ex]

\langle aexpr \rangle
& ::= &
\langle int\_lit \rangle
\mid \langle real\_lit \rangle
\mid \langle ident \rangle
\mid \langle aexpr \rangle \; \langle arith\_op \rangle \; \langle aexpr \rangle
\mid ( \langle aexpr \rangle )
\mid \langle ident \rangle ( \langle args \rangle ) 
\\[0.3ex]

\langle sexpr \rangle 
& ::= & 
\langle string\_lit \rangle
\mid \langle ident \rangle
\mid ( \langle sexpr \rangle )
\mid \langle ident \rangle ( \langle args \rangle ) 
\\[0.3ex]

\langle arrexpr \rangle 
& ::= & 
\langle array\_lit \rangle
\mid \langle ident \rangle
\mid ( \langle arrexpr \rangle )
\mid \langle ident \rangle ( \langle args \rangle ) 
\\[0.3ex]

\langle array\_lit \rangle 
& ::= & 
\texttt{[}\ \langle args \rangle\ \texttt{]} 
\\[0.3ex]

\langle formula \rangle
& ::= &
\langle qf\_formula \rangle
\\ & | &
\langle q\_formula \rangle
\\[0.3ex]

\langle qf\_formula \rangle
& ::= &
\langle expr \rangle \; \langle rel\_op \rangle \; \langle expr \rangle
\\ & | &
! \; \langle qf\_formula \rangle
\\ & | &
( \langle qf\_formula \rangle )
\\ & | &
\langle ident \rangle ( \langle args \rangle )
\\ & | &
\langle bool\_lit \rangle
\\ & | &
\langle qf\_formula \rangle \; ( \land \mid \lor \mid \Rightarrow \mid \Leftrightarrow ) \; \langle qf\_formula \rangle
\\[0.3ex]

\langle q\_formula \rangle
& ::= &
\textcolor{blue}{\texttt{forall}} \; \langle bound\_vars \rangle :: \langle formula \rangle
\\ & | &
\textcolor{blue}{\texttt{exists}} \; \langle bound\_vars \rangle :: \langle formula \rangle
\\[0.3ex]
\langle bound\_vars \rangle 
& ::= & 
\langle ident \rangle : \langle type \rangle
\mid \langle ident \rangle : \langle type \rangle , \langle bound\_vars \rangle 
\\[0.3ex]

\langle arith\_op \rangle 
& ::= & + \mid - \mid * \mid / 
\\[0.3ex]

\langle rel\_op \rangle
& ::= & = \mid \neq \mid < \mid > \mid \leq \mid \geq
\\[0.3ex]

\langle int\_lit \rangle
& ::= & \texttt{0} \mid \texttt{1} \mid \cdots \mid \texttt{9} \text{ (and sequences thereof)}
\\[0.3ex]

\langle real\_lit \rangle
& ::= & \langle int\_lit \rangle \texttt{.} \langle int\_lit \rangle
\\[0.3ex]

\langle bool\_lit \rangle
& ::= & \textcolor{blue}{\texttt{true}} \mid \textcolor{blue}{\texttt{false}}
\\[0.3ex]

\langle string\_lit \rangle
& ::= & "\langle char \rangle^*"
\\[0.3ex]

\langle char \rangle
& ::= & \langle letter \rangle \mid \langle digit \rangle
\\[0.3ex]

\langle ident \rangle
& ::= & \langle letter \rangle \mid \langle letter \rangle \langle ident \rangle
\\[0.3ex]

\langle letter \rangle
& ::= & a \mid b \mid \cdots \mid z \mid A \mid B \mid \cdots \mid Z
\\[0.3ex]

\langle digit \rangle
& ::= & 0 \mid 1 \mid \cdots \mid 9
\end{array}
\]
}
\section{Rejection Sampling Details}
\label{appen:rejectionSampling}
\paragraph{Rejection Sampling Details.}
Given a target distribution $\pi_t$ and a proposal distribution $\pi_p$ with $S(\pi_t) \subseteq S(\pi_p)$, assume there exists $M \geq 1$ such that
$
\forall z \in \Omega, \quad D_{\pi_t}(z) \leq M \cdot D_{\pi_p}(z).
$
The procedure samples $z \sim \pi_p$ and $u \sim \mathrm{Uniform}(0,1)$, and accepts $z$ iff
$
u \leq \frac{D_{\pi_t}(z)}{M \cdot D_{\pi_p}(z)}.
$
Otherwise, the sample is rejected and the process repeats.
Accepted samples follow $\pi_t$. Moreover, any $z \notin S(\pi_t)$ is always rejected since $D_{\pi_t}(z)=0$, ensuring all accepted samples lie in $S(\pi_t)$. The acceptance probability is $1/M$, so efficiency depends on how tightly $\pi_p$ upper-bounds $\pi_t$. For this paper, we assume $M = 1$.
\section{Library Functions for Symbolic Regression}
\label{appen:symRegLib}
\paragraph{Library Function Set $\fset{}$}

\begin{itemize}

\item $\funcIt{abs}(x:\typeit{real}) \rightarrow r:\typeit{real}$  
\textbf{ensures} $r \ge 0$  
\textbf{ensures} $r = x \lor r = -x$

\item $\funcIt{max}(x:\typeit{real},y:\typeit{real}) \rightarrow r:\typeit{real}$  
\textbf{ensures} $(x \le r) \land (y \le r)$  
\textbf{ensures} $(r = x) \lor (r = y)$

\item $\funcIt{min}(x:\typeit{real},y:\typeit{real}) \rightarrow r:\typeit{real}$  
\textbf{ensures} $(x \ge r) \land (y \ge r)$  
\textbf{ensures} $(r = x) \lor (r = y)$

\item $\funcIt{sin}(x:\typeit{real}) \rightarrow r:\typeit{real}$  
\textbf{ensures} $-1 \le r \le 1$  
\textbf{ensures} $(x = 0) \Rightarrow (r = 0)$

\item $\funcIt{cos}(x:\typeit{real}) \rightarrow r:\typeit{real}$  
\textbf{ensures} $-1 \le r \le 1$  
\textbf{ensures} $(x = 0) \Rightarrow (r = 1)$

\item $\funcIt{pi}(x:\typeit{real}) \rightarrow r:\typeit{real}$  
\textbf{requires} $x \ge 0$  
\textbf{ensures} $3.141592653589790 \cdot x \le r$  
\textbf{ensures} $r \le 3.141592653589793 \cdot x$

\item $\funcIt{pow}(x:\typeit{real}, d:\typeit{real}) \rightarrow r:\typeit{real}$  
\textbf{requires} $x \ge 0$  
\textbf{requires} $(x \neq 0) \lor (d \ge 0)$  
\textbf{ensures} $r \ge 0$  
\textbf{ensures} $(x > 0) \Rightarrow (r > 0)$  
\textbf{ensures} $(d = 0) \Rightarrow (r = 1)$  
\textbf{ensures} $(x \ge 1 \land d \ge 0) \Rightarrow (r \ge 1)$  
\textbf{ensures} $(x \ge 1 \land d \le 0) \Rightarrow (r \le 1)$  
\textbf{ensures} $(x \le 1 \land d \ge 0) \Rightarrow (r \le 1)$  
\textbf{ensures} $(x \le 1 \land d \le 0) \Rightarrow (r \ge 1)$


\item $\funcIt{sqrt}(x:\typeit{real}) \rightarrow r:\typeit{real}$  
\textbf{requires} $x \ge 0$  
\textbf{ensures} $r \ge 0$  
\textbf{ensures} $r \cdot r = x$  


\item $\funcIt{exp}(x:\typeit{real}) \rightarrow r:\typeit{real}$  
\textbf{ensures} $r \ge 0$  
\textbf{ensures} $(x \le 0) \Rightarrow (r \le 1)$  
\textbf{ensures} $(x \ge 0) \Rightarrow (r \ge 1)$  
\textbf{ensures} $(x = 1) \Rightarrow (2.71 \le r \le 2.72)$  
\textbf{ensures} $r \ge (1 + x)$

\item $\funcIt{log}(x:\typeit{real}) \rightarrow r:\typeit{real}$
\textbf{requires} $x > 0$  
\textbf{ensures} $(x \ge 1) \Rightarrow (r \ge 0)$  
\textbf{ensures} $(x \le 1) \Rightarrow (r \le 0)$  
\textbf{ensures} $(x = 1) \Rightarrow (r = 0)$  

\item $\funcIt{neural1}(x:\typeit{real}) \rightarrow r:\typeit{real}$
\item $\funcIt{neural2}(x1:\typeit{real}, x2:\typeit{real}) \rightarrow r:\typeit{real}$
\item $\funcIt{neural3}(x1:\typeit{real}, x2:\typeit{real}, x3:\typeit{real}) \rightarrow r:\typeit{real}$

\end{itemize}

\section{Additional Axioms for Symbolic Regression}
\label{appen:symRegAxiom}
\subsection{Axiom Syntax}
\begin{align*}
    \langle \text{axiom} \rangle ::= \langle q\_formula \rangle
\end{align*}
\paragraph{Axiom Set $\axiom{}$}

\begin{itemize}

\item \textbf{Interval Multiplication Rule: }$\forall\, x, y, x_{lb}, x_{ub}, y_{lb}, y_{ub} \in \mathbb{R}.\; (x_{lb} \le x \le x_{ub}) \wedge (y_{lb} \le y \le y_{ub}) \implies$
\[
\funcIt{min}(\funcIt{min}(x_{lb} \times y_{lb},\, x_{lb}\times y_{ub}),\, \funcIt{min}(x_{ub} \times y_{lb},\, x_{ub} \times y_{ub})) \le x \times y \le \funcIt{max}(\funcIt{max}(x_{lb}\times y_{lb},\, x_{lb}\times y_{ub}),\, \funcIt{max}(x_{ub} \times y_{lb},\, x_{ub}\times y_{ub})).
\]

\item $\forall\, x, d_1, d_2 \in \mathbb{R}.\; (0 \le x \le 1) \wedge (d_1 \le d_2) \implies \funcIt{pow}(x,d_1) \ge \funcIt{pow}(x,d_2)$.

\item $\forall\, x, d_1, d_2 \in \mathbb{R}.\; (x \ge 1) \wedge (d_1 \le d_2) \implies \funcIt{pow}(x,d_1) \le \funcIt{pow}(x,d_2)$.

\item $\forall\, x \in \mathbb{R}.\; \funcIt{exp}(x) = \funcIt{pow}(\funcIt{exp}(1),\, x)$.

\item $\forall\, x \in \mathbb{R}.\; (x > 0) \implies \big(\forall\, r \in \mathbb{R}.\; r = \funcIt{log}(x) \iff \funcIt{exp}(r) = x\big)$.
\item $\forall\, x \in \mathbb{R}.\; (x \geq 0) \implies (\funcIt{sqrt}(x) = \funcIt{pow}(x, 0.5))$.
\end{itemize}

\section{Library Functions for Dafny Program Verification}
\label{appen:dafnyBenchLib}
\paragraph{Library Function Set $\fset{}$}

\begin{itemize}

\item $\funcIt{claude}(prompt:\typeit{str}) \rightarrow r:\typeit{str}$

\item $\funcIt{noDiff}(base\_code:\typeit{str}, annotated\_code:\typeit{str}) \rightarrow r:\typeit{bool}$

\item $\funcIt{noDiffWithErrorMsg}(base\_code:\typeit{str}, annotated\_code:\typeit{str}) \rightarrow r:\typeit{str}$

\item $\funcIt{dafnyVerifier}(dafny\_code:\typeit{str}) \rightarrow r:\typeit{bool}$

\item $\funcIt{dafnyVerifierWithErrorMsg}(dafny\_code:\typeit{str}) \rightarrow r:\typeit{str}$

\item $\funcIt{extractDafnyCode}(text:\typeit{str}) \rightarrow r:\typeit{str}$

\item $\funcIt{isSubstring}(sub:\typeit{str}, full:\typeit{str}) \rightarrow r:\typeit{bool}$

\end{itemize}

\section{Additional Axioms for Dafny Program Verification}
\label{appen:danfyBenchAxiom}
\paragraph{Axiom Set $\axiom{}$}

\begin{itemize}

\item \textbf{Reflexivity:} $\forall\, a, b \in \alpN{*}.\; (a = b) \implies \funcIt{noDiff}(a, b)$.

\item \textbf{Transitivity:} $\forall\, a, b, c \in \alpN{*}.\; \funcIt{noDiff}(a, b) \wedge \funcIt{noDiff}(b, c) \implies \funcIt{noDiff}(a, c)$.

\item $\forall\, a, b \in \alpN{*}.\; \funcIt{noDiffWithErrorMsg}(a, b) = \text{``''} \iff \funcIt{noDiff}(a, b)$.

\item $\forall\, c \in \alpN{*}.\; \funcIt{dafnyVerifierWithErrorMsg}(c) = \text{``''} \iff \funcIt{dafnyVerifier}(c)$.

\end{itemize}

\section{Library Functions for $\tau^2$-bench}
\label{appen:tauBenchLib}
\paragraph{Library Function Set $\fset{}$}

\begin{itemize}

\item $\funcIt{qwen}(prompt:\typeit{str}) \rightarrow r:\typeit{str}$

\item $\funcIt{agentCCheckList}(actions:arr\langle\typeit{str}\rangle, trace:\typeit{str}, domain:\typeit{str}) \rightarrow r:\typeit{bool}$

\item $\funcIt{agentCCheckListWithErrorMsg}(actions:arr\langle\typeit{str}\rangle, trace:\typeit{str}, domain:\typeit{str}) \rightarrow r:\typeit{str}$

\item $\funcIt{parseToolCalls}(response:\typeit{str}) \rightarrow r:arr\langle\typeit{str}\rangle$

\item $\funcIt{firstToolName}(actions:arr\langle\typeit{str}\rangle) \rightarrow r:\typeit{str}$

\item $\funcIt{joinActions}(actions:arr\langle\typeit{str}\rangle) \rightarrow r:\typeit{str}$

\item $\funcIt{transferToHuman}(message:\typeit{str}) \rightarrow r:\typeit{str}$

\item $\funcIt{isSubstring}(sub:\typeit{str}, full:\typeit{str}) \rightarrow r:\typeit{bool}$

\end{itemize}

\section{Additional Axioms for $\tau^2$-bench}
\label{appen:tauBenchAxiom}
\paragraph{Axiom Set $\axiom{}$}

\begin{itemize}

\item \textbf{Transfer-to-human satisfiability:} $\forall\, m, t, d \in \alpN{*}.\; \funcIt{agentCCheckList}([\funcIt{transferToHuman}(m)], t, d)$. A singleton action array that transfers to a human agent complies with the Agent-C policy for any trace and domain.

\item \textbf{Empty-action safety:} $\forall\, t, d \in \alpN{*}.\; \funcIt{agentCCheckList}([\,], t, d)$. An empty action array (no tool calls) trivially complies.

\item \textbf{Error-message consistency:} $\forall\, a \in arr\langle\alpN{*}\rangle,\ t, d \in \alpN{*}.\; \funcIt{agentCCheckListWithErrorMsg}(a, t, d) = \text{``''} \iff \funcIt{agentCCheckList}(a, t, d)$.

\end{itemize}

\section{Library Functions for GSM-Symbolic}
\label{appen:gsmSymbolicLib}
\paragraph{Library Function Set $\fset{}$}

\begin{itemize}

\item $\funcIt{qwen}(prompt:\typeit{str}) \rightarrow r:\typeit{str}$

\item $\funcIt{gsmParser}(expression:\typeit{str}) \rightarrow r:\typeit{bool}$

\item $\funcIt{gsmParserWithErrorMsg}(expression:\typeit{str}) \rightarrow r:\typeit{str}$

\item $\funcIt{extractExpression}(text:\typeit{str}) \rightarrow r:\typeit{str}$

\end{itemize}

\section{Additional Axioms for GSM-Symbolic}
\label{appen:gsmSymbolicAxiom}
\paragraph{Axiom Set $\axiom{}$}

\begin{itemize}

\item $\forall\, e \in \alpN{*}.\; (e = \text{``\textless\textless\;\textgreater\textgreater''}) \implies \funcIt{gsmParser}(e)$.

\item $\forall\, e \in \alpN{*}.\; (e = \text{``0''}) \implies \funcIt{gsmParser}(e)$.

\item $\forall\, e \in \alpN{*}.\; \funcIt{gsmParserWithErrorMsg}(e) = \text{``''} \iff \funcIt{gsmParser}(e)$.

\end{itemize}

\section{Example Agents for LLM Assisted Program Verification}
\label{appen:exampleDafnyAgent}

\noindent The following is an example verified agent program for the Dafny program verification task. The agent defines and invokes three FGGMs, $\funcIt{initialFGGM}$ (Fig.~\ref{fig:initialFGGMDef}), $\funcIt{diffErrorFGGM}$ (Fig.~\ref{fig:diffErrorFGGMDef}), and $\funcIt{verifierErrorFGGM}$ (Fig.~\ref{fig:verifierErrorFGGMDef}), each with local contract $\loutspec{}: \funcIt{noDiff}(base\_program, \cdot)$ and a fallback that returns the original program (justified by the reflexivity axiom). All three share $GMid := \funcIt{claude}$ and differ only in the prompting function $\prompt$ and output processor $\fpost$: the first generates an initial annotation with context-specific guidance, the second repairs after a diff-checker failure, and the third repairs after a verification failure with error-specific feedback. In each FGGM, $\fpost$ extracts the annotated program from the fenced code block of the raw model response, while the auto-synthesized error function $\ferr$ (which reports the failing atomic formulas of $\loutspec{}$) is threaded back to the prompt through its $err$ argument, making the rejection sampler (Fig.~\ref{fig:rejectionSampler}) stateful across iterations. The agent program (Fig.~\ref{fig:exampleDafnyAgentNew}) dispatches to the appropriate FGGM on each loop iteration. On HumanEvalDafny, the fallback fires on only 5.1\% of instances (Table~\ref{tab:fallback}), confirming that the rejection sampler with $\ferr$ feedback rarely exhausts all $K$ samples.

\newpage
\begin{figure}[H]
\begin{tcolorbox}[
colback=lightgraybg,
colframe=lightgraybg,
boxrule=0pt,
arc=2pt,
left=4pt,right=4pt,top=4pt,bottom=4pt
]
\scriptsize
\[
\begin{aligned}
&\textbf{FGGM Definition:} \\[-0.25em]
&\quad id := \funcIt{initialFGGM} \\[-0.25em]
&\quad GMid := \funcIt{claude} \\[-0.25em]
&\quad typeSig := (base\_program:\typeit{str}) \to \typeit{str} \\[-0.25em]
&\quad \linspec{} := \kw{true} \\[-0.25em]
&\quad \loutspec{} := \funcIt{noDiff}(base\_program,\ f(base\_program)) \\[-0.25em]
&\quad \prompt := \kw{function}\ \prompt(base\_program:\typeit{str},\ err:\typeit{str}) : (\typeit{str})\ \{ \\[-0.25em]
&\quad\quad \kw{var}\ g:\typeit{str} := \\[-0.25em]
&\quad\quad\quad \text{``You are an expert in Dafny verification. Given the following unannotated Dafny code, ''} \\[-0.25em]
&\quad\quad\quad + \text{``add appropriate loop invariants, assertions, and decreases clauses to make it verify.\textbackslash n\textbackslash n''} \\[-0.25em]
&\quad\quad\quad + \text{``IMPORTANT RULES:\textbackslash n''} \\[-0.25em]
&\quad\quad\quad + \text{``1. Do NOT modify the method signatures, requires clauses, or ensures clauses\textbackslash n''} \\[-0.25em]
&\quad\quad\quad + \text{``2. Only add invariants, assertions, and decreases clauses\textbackslash n''} \\[-0.25em]
&\quad\quad\quad + \text{``3. Do NOT change the logic or control flow\textbackslash n''} \\[-0.25em]
&\quad\quad\quad + \text{``4. Keep invariants MINIMAL - only what's necessary for postconditions\textbackslash n''} \\[-0.25em]
&\quad\quad\quad + \text{``5. Add decreases clauses for termination\textbackslash n''}; \\[0.2em]
&\quad\quad \kw{if}\ (\funcIt{isSubstring}(\text{``lemma ''},\ base\_program))\ \{ \\[-0.25em]
&\quad\quad\quad g := g \\[-0.25em]
&\quad\quad\quad\quad + \text{``6. LEMMAS: If a lemma has a non-trivial postcondition, add explicit proof steps:\textbackslash n''} \\[-0.25em]
&\quad\quad\quad\quad + \text{``\quad - Use assertions to unfold function definitions\textbackslash n''} \\[-0.25em]
&\quad\quad\quad\quad + \text{``\quad - Add case analysis or pattern matching if needed\textbackslash n''} \\[-0.25em]
&\quad\quad\quad\quad + \text{``\quad - Call other lemmas to establish intermediate facts\textbackslash n''} \\[-0.25em]
&\quad\quad\quad\quad + \text{``\quad - Do NOT leave lemma bodies empty unless postcondition is trivial\textbackslash n''};\ \} \\[0.2em]
&\quad\quad \kw{if}\ (\funcIt{isSubstring}(\text{``<''},\ base\_program)\ \wedge\ \funcIt{isSubstring}(\text{``>''},\ base\_program))\ \{ \\[-0.25em]
&\quad\quad\quad g := g \\[-0.25em]
&\quad\quad\quad\quad + \text{``7. GENERICS: When calling generic functions/lemmas, propagate type constraints:\textbackslash n''} \\[-0.25em]
&\quad\quad\quad\quad + \text{``\quad - If a called function requires T(00), declare it in the caller too\textbackslash n''} \\[-0.25em]
&\quad\quad\quad\quad + \text{``\quad - Explicitly instantiate type parameters when needed\textbackslash n''};\ \} \\[0.2em]
&\quad\quad \kw{if}\ (\funcIt{isSubstring}(\text{``set<''},\ base\_program)\ \vee\ \funcIt{isSubstring}(\text{``set ''},\ base\_program))\ \{ \\[-0.25em]
&\quad\quad\quad g := g \\[-0.25em]
&\quad\quad\quad\quad + \text{``8. SETS: Use simple, direct assertions for set reasoning:\textbackslash n''} \\[-0.25em]
&\quad\quad\quad\quad + \text{``\quad - Assert set equality decompositions: A == B + (A - B)\textbackslash n''} \\[-0.25em]
&\quad\quad\quad\quad + \text{``\quad - Use cardinality constraints: |A| == |B| + |C|\textbackslash n''} \\[-0.25em]
&\quad\quad\quad\quad + \text{``\quad - Avoid complex witness variables or case analysis\textbackslash n''} \\[-0.25em]
&\quad\quad\quad\quad + \text{``\quad - Let Dafny's built-in set theory automation work\textbackslash n''};\ \} \\[0.2em]
&\quad\quad \kw{if}\ (\funcIt{isSubstring}(\text{``multiset''},\ base\_program))\ \{ \\[-0.25em]
&\quad\quad\quad g := g \\[-0.25em]
&\quad\quad\quad\quad + \text{``9. MULTISETS: Match the structure of postconditions directly\textbackslash n''} \\[-0.25em]
&\quad\quad\quad\quad + \text{``\quad - Use simple decomposition assertions\textbackslash n''} \\[-0.25em]
&\quad\quad\quad\quad + \text{``\quad - Avoid unnecessary intermediate steps\textbackslash n''};\ \} \\[0.2em]
&\quad\quad \kw{return}\ g + \text{``\textbackslash nBase program:\textbackslash n''} + base\_program \\[-0.25em]
&\quad\quad\quad + \text{``\textbackslash n\textbackslash nReturn ONLY the annotated Dafny code in a \textasciigrave\textasciigrave\textasciigrave dafny code block.''} + err;\ \} \\[0.3em]
&\quad \fallback := \kw{function}\ \fallback(base\_program:\typeit{str},\ y:\typeit{str}) : (\typeit{str}) \\[-0.25em]
&\quad\quad\quad \kw{ensures}\ \funcIt{noDiff}(base\_program,\ \fallback(base\_program, y)) \\[-0.25em]
&\quad \{ \\[-0.25em]
&\quad\quad\quad \kw{return}\ base\_program; \\[-0.25em]
&\quad \} \\[0.3em]
&\quad \fpost := \kw{function}\ \fpost(y_{raw}:\typeit{str}) : (\typeit{str})\ \{ \\[-0.25em]
&\quad\quad\quad \kw{return}\ \funcIt{extractDafnyCode}(y_{raw});\ \}
\end{aligned}
\]
\end{tcolorbox}
\caption{\small FGGM definition for $\funcIt{initialFGGM}$. The prompting function builds context-specific guidance by checking for lemmas, generics, sets, and multisets in the base program, and appends the auto-synthesized rejection feedback $err$ (the failing atomic formulas of $\loutspec{}$) to make the rejection sampler stateful. The output processor $\fpost$ extracts the annotated program from the fenced code block in the raw model response. The fallback returns the original program, satisfying $\loutspec{}$ by the reflexivity axiom. Fallback trigger rate: 5.1\% on HumanEvalDafny (Table~\ref{tab:fallback}).}
\label{fig:initialFGGMDef}
\end{figure}

\newpage
\begin{figure}[H]
\begin{tcolorbox}[
colback=lightgraybg,
colframe=lightgraybg,
boxrule=0pt,
arc=2pt,
left=4pt,right=4pt,top=4pt,bottom=4pt
]
\scriptsize
\[
\begin{aligned}
&\textbf{FGGM Definition:} \\[-0.25em]
&\quad id := \funcIt{diffErrorFGGM} \\[-0.25em]
&\quad GMid := \funcIt{claude} \\[-0.25em]
&\quad typeSig := (base\_program:\typeit{str},\ diff\_error:\typeit{str},\ prev\_attempt:\typeit{str}) \to \typeit{str} \\[-0.25em]
&\quad \linspec{} := \kw{true} \\[-0.25em]
&\quad \loutspec{} := \funcIt{noDiff}(base\_program,\ f(base\_program,\ diff\_error,\ prev\_attempt)) \\[-0.25em]
&\quad \prompt := \kw{function}\ \prompt(base\_program:\typeit{str},\ diff\_error:\typeit{str},\ prev\_attempt:\typeit{str},\ err:\typeit{str}) : (\typeit{str})\ \{ \\[-0.25em]
&\quad\quad \kw{return} \\[-0.25em]
&\quad\quad\quad \text{``The previous Dafny annotation modified the base program logic incorrectly.\textbackslash n\textbackslash n''} \\[-0.25em]
&\quad\quad\quad + \text{``noDiff error:\textbackslash n''} + diff\_error + \text{``\textbackslash n\textbackslash n''} \\[-0.25em]
&\quad\quad\quad + \text{``Base program:\textbackslash n''} + base\_program + \text{``\textbackslash n\textbackslash n''} \\[-0.25em]
&\quad\quad\quad + \text{``Previous attempt:\textbackslash n''} + prev\_attempt + \text{``\textbackslash n\textbackslash n''} \\[-0.25em]
&\quad\quad\quad + \text{``Please add annotations WITHOUT changing:\textbackslash n''} \\[-0.25em]
&\quad\quad\quad + \text{``- Method signatures\textbackslash n''} \\[-0.25em]
&\quad\quad\quad + \text{``- Requires/ensures clauses\textbackslash n''} \\[-0.25em]
&\quad\quad\quad + \text{``- Program logic or control flow\textbackslash n''} \\[-0.25em]
&\quad\quad\quad + \text{``Only add invariants, assertions, and decreases clauses.\textbackslash n''} \\[-0.25em]
&\quad\quad\quad + \text{``Return ONLY the corrected Dafny code in a \textasciigrave\textasciigrave\textasciigrave dafny code block.''} + err;\ \} \\[0.3em]
&\quad \fallback := \kw{function}\ \fallback(base\_program:\typeit{str},\ diff\_error:\typeit{str},\ prev\_attempt:\typeit{str},\ y:\typeit{str}) : (\typeit{str}) \\[-0.25em]
&\quad\quad\quad \kw{ensures}\ \funcIt{noDiff}(base\_program,\ \fallback(base\_program,\ diff\_error,\ prev\_attempt,\ y)) \\[-0.25em]
&\quad \{ \\[-0.25em]
&\quad\quad\quad \kw{return}\ base\_program; \\[-0.25em]
&\quad \} \\[0.3em]
&\quad \fpost := \kw{function}\ \fpost(y_{raw}:\typeit{str}) : (\typeit{str})\ \{ \\[-0.25em]
&\quad\quad\quad \kw{return}\ \funcIt{extractDafnyCode}(y_{raw});\ \}
\end{aligned}
\]
\end{tcolorbox}
\caption{\small FGGM definition for $\funcIt{diffErrorFGGM}$. Invoked when the previous iteration's output failed the diff checker. The prompt includes the diff error and previous attempt, instructing the model to preserve the base program logic, and appends the auto-synthesized rejection feedback $err$ to prevent repeated mistakes. The output processor $\fpost$ extracts the corrected program from the raw model response. Fallback trigger rate: 5.1\% on HumanEvalDafny (Table~\ref{tab:fallback}).}
\label{fig:diffErrorFGGMDef}
\end{figure}

\newpage
\begin{figure}[H]
\begin{tcolorbox}[
colback=lightgraybg,
colframe=lightgraybg,
boxrule=0pt,
arc=2pt,
left=4pt,right=4pt,top=4pt,bottom=4pt
]
\scriptsize
\[
\begin{aligned}
&\textbf{FGGM Definition:} \\[-0.25em]
&\quad id := \funcIt{verifierErrorFGGM} \\[-0.25em]
&\quad GMid := \funcIt{claude} \\[-0.25em]
&\quad typeSig := (base\_program:\typeit{str},\ verify\_error:\typeit{str},\ prev\_attempt:\typeit{str}) \to \typeit{str} \\[-0.25em]
&\quad \linspec{} := \kw{true} \\[-0.25em]
&\quad \loutspec{} := \funcIt{noDiff}(base\_program,\ f(base\_program,\ verify\_error,\ prev\_attempt)) \\[-0.25em]
&\quad \prompt := \kw{function}\ \prompt(base\_program:\typeit{str},\ verify\_error:\typeit{str},\ prev\_attempt:\typeit{str},\ err:\typeit{str}) : (\typeit{str})\ \{ \\[-0.25em]
&\quad\quad \kw{var}\ sg:\typeit{str} := \\[-0.25em]
&\quad\quad\quad \text{``Focus on:\textbackslash n''} \\[-0.25em]
&\quad\quad\quad + \text{``- Strengthening invariants MINIMALLY\textbackslash n''} \\[-0.25em]
&\quad\quad\quad + \text{``- Adding simple decomposition assertions\textbackslash n''} \\[-0.25em]
&\quad\quad\quad + \text{``- Ensuring decreases clauses are correct\textbackslash n''}; \\[0.2em]
&\quad\quad \kw{if}\ (\funcIt{isSubstring}(\text{``timeout''},\ verify\_error)\ \vee\ \funcIt{isSubstring}(\text{``resource''},\ verify\_error))\ \{ \\[-0.25em]
&\quad\quad\quad sg := \text{``TIMEOUT/RESOURCE ISSUE: Break down the proof into smaller steps:\textbackslash n''} \\[-0.25em]
&\quad\quad\quad\quad + \text{``- Add strategic intermediate assertions\textbackslash n''} \\[-0.25em]
&\quad\quad\quad\quad + \text{``- Call helper lemmas to establish sub-goals\textbackslash n''} \\[-0.25em]
&\quad\quad\quad\quad + \text{``- Simplify complex invariants\textbackslash n''};\ \} \\[0.2em]
&\quad\quad \kw{if}\ (\funcIt{isSubstring}(\text{``lemma''},\ verify\_error)\ \vee\ \funcIt{isSubstring}(\text{``postcondition''},\ verify\_error))\ \{ \\[-0.25em]
&\quad\quad\quad sg := \text{``LEMMA PROOF ISSUE: Add explicit proof steps in lemma body:\textbackslash n''} \\[-0.25em]
&\quad\quad\quad\quad + \text{``- Assert intermediate facts that lead to postcondition\textbackslash n''} \\[-0.25em]
&\quad\quad\quad\quad + \text{``- Unfold recursive function definitions\textbackslash n''} \\[-0.25em]
&\quad\quad\quad\quad + \text{``- Use case analysis if needed\textbackslash n''};\ \} \\[0.2em]
&\quad\quad \kw{if}\ (\funcIt{isSubstring}(\text{``invariant''},\ verify\_error))\ \{ \\[-0.25em]
&\quad\quad\quad sg := \text{``INVARIANT ISSUE: Adjust loop invariants:\textbackslash n''} \\[-0.25em]
&\quad\quad\quad\quad + \text{``- Ensure invariants are maintained after each iteration\textbackslash n''} \\[-0.25em]
&\quad\quad\quad\quad + \text{``- Add bounds and relationships incrementally\textbackslash n''} \\[-0.25em]
&\quad\quad\quad\quad + \text{``- Keep invariants minimal but sufficient\textbackslash n''};\ \} \\[0.2em]
&\quad\quad \kw{if}\ (\funcIt{isSubstring}(\text{``type''},\ verify\_error)\ \vee\ \funcIt{isSubstring}(\text{``constraint''},\ verify\_error))\ \{ \\[-0.25em]
&\quad\quad\quad sg := \text{``TYPE CONSTRAINT ISSUE: Check generic type parameters:\textbackslash n''} \\[-0.25em]
&\quad\quad\quad\quad + \text{``- Propagate nonemptiness constraints like T(00)\textbackslash n''} \\[-0.25em]
&\quad\quad\quad\quad + \text{``- Ensure type parameters match between caller and callee\textbackslash n''};\ \} \\[0.2em]
&\quad\quad \kw{return} \\[-0.25em]
&\quad\quad\quad \text{``The previous Dafny annotation preserved the base logic but failed verification.\textbackslash n\textbackslash n''} \\[-0.25em]
&\quad\quad\quad + \text{``Verification error:\textbackslash n''} + verify\_error + \text{``\textbackslash n\textbackslash n''} \\[-0.25em]
&\quad\quad\quad + sg + \text{``\textbackslash n''} \\[-0.25em]
&\quad\quad\quad + \text{``Base program:\textbackslash n''} + base\_program + \text{``\textbackslash n\textbackslash n''} \\[-0.25em]
&\quad\quad\quad + \text{``Previous attempt:\textbackslash n''} + prev\_attempt + \text{``\textbackslash n\textbackslash n''} \\[-0.25em]
&\quad\quad\quad + \text{``Please fix the annotations. Do NOT modify method signatures, requires, or ensures clauses.\textbackslash n''} \\[-0.25em]
&\quad\quad\quad + \text{``Return ONLY the corrected Dafny code in a \textasciigrave\textasciigrave\textasciigrave dafny code block.''} + err;\ \} \\[0.3em]
&\quad \fallback := \kw{function}\ \fallback(base\_program:\typeit{str},\ verify\_error:\typeit{str},\ prev\_attempt:\typeit{str},\ y:\typeit{str}) : (\typeit{str}) \\[-0.25em]
&\quad\quad\quad \kw{ensures}\ \funcIt{noDiff}(base\_program,\ \fallback(base\_program,\ verify\_error,\ prev\_attempt,\ y)) \\[-0.25em]
&\quad \{ \\[-0.25em]
&\quad\quad\quad \kw{return}\ base\_program; \\[-0.25em]
&\quad \} \\[0.3em]
&\quad \fpost := \kw{function}\ \fpost(y_{raw}:\typeit{str}) : (\typeit{str})\ \{ \\[-0.25em]
&\quad\quad\quad \kw{return}\ \funcIt{extractDafnyCode}(y_{raw});\ \}
\end{aligned}
\]
\end{tcolorbox}
\caption{\small FGGM definition for $\funcIt{verifierErrorFGGM}$. Invoked when the previous iteration passed the diff checker but failed verification. The prompting function analyzes the verification error to provide targeted guidance (timeout, lemma, invariant, or type issues) and appends the auto-synthesized rejection feedback $err$, enabling the stateful rejection sampler to avoid repeated failures. The output processor $\fpost$ extracts the corrected program from the raw model response. Fallback trigger rate: 5.1\% on HumanEvalDafny (Table~\ref{tab:fallback}).}
\label{fig:verifierErrorFGGMDef}
\end{figure}

\newpage
\begin{figure}[H]
\begin{tcolorbox}[
colback=lightgraybg,
colframe=lightgraybg,
boxrule=0pt,
arc=2pt,
left=4pt,right=4pt,top=4pt,bottom=4pt
]
\scriptsize
\[
\begin{aligned}
&\kw{method}\ \funcIt{agent}(base\_program:\typeit{str})\ \kw{returns}\ (r:\typeit{str}) \\[-0.25em]
&\quad \kw{ensures}\ \funcIt{noDiff}(base\_program, r) \\[-0.25em]
&\{ \\[-0.25em]
&\quad \kw{var}\ max\_attempts:\typeit{int} := 5; \\[-0.25em]
&\quad \kw{var}\ attempt:\typeit{int} := 0; \\[-0.25em]
&\quad \kw{var}\ best:\typeit{str} := base\_program; \\[-0.25em]
&\quad \kw{var}\ best\_verified:\typeit{bool} := \kw{false}; \\[0.3em]
&\quad \kw{var}\ prev\_diff\_err:\typeit{str} := \text{``''}; \\[-0.25em]
&\quad \kw{var}\ prev\_ver\_err:\typeit{str} := \text{``''}; \\[-0.25em]
&\quad \kw{var}\ prev\_code:\typeit{str} := \text{``''}; \\[0.3em]
&\quad \kw{while}\ (attempt < max\_attempts) \\[-0.25em]
&\quad\quad \kw{invariant}\ \funcIt{noDiff}(base\_program, best) \\[-0.25em]
&\quad\quad \kw{decreases}\ max\_attempts - attempt \\[-0.25em]
&\quad \{ \\[-0.25em]
&\quad\quad \kw{var}\ y:\typeit{str}; \\[0.2em]
&\quad\quad \kw{if}\ (attempt = 0)\ \{ \\[-0.25em]
&\quad\quad\quad y := \funcIt{initialFGGM}(base\_program); \\[-0.25em]
&\quad\quad \}\ \kw{else}\ \{ \\[-0.25em]
&\quad\quad\quad \kw{if}\ (prev\_diff\_err \neq \text{``''})\ \{ \\[-0.25em]
&\quad\quad\quad\quad y := \funcIt{diffErrorFGGM}(base\_program,\ prev\_diff\_err,\ prev\_code); \\[-0.25em]
&\quad\quad\quad \}\ \kw{else}\ \{ \\[-0.25em]
&\quad\quad\quad\quad y := \funcIt{verifierErrorFGGM}(base\_program,\ prev\_ver\_err,\ prev\_code); \\[-0.25em]
&\quad\quad\quad \} \\[-0.25em]
&\quad\quad \} \\[0.2em]
&\quad\quad prev\_diff\_err := \text{``''}; \\[-0.25em]
&\quad\quad prev\_ver\_err := \text{``''}; \\[-0.25em]
&\quad\quad prev\_code := y; \\[0.2em]
&\quad\quad \kw{var}\ diff\_err:\typeit{str} := \funcIt{noDiffWithErrorMsg}(base\_program,\ y); \\[0.3em]
&\quad\quad \kw{if}\ (diff\_err = \text{``''})\ \{ \\[-0.25em]
&\quad\quad\quad \kw{var}\ v\_err:\typeit{str} := \funcIt{dafnyVerifierWithErrorMsg}(y); \\[0.2em]
&\quad\quad\quad \kw{if}\ (v\_err = \text{``''})\ \{ \\[-0.25em]
&\quad\quad\quad\quad r := y;\ \kw{return}; \\[-0.25em]
&\quad\quad\quad \}\ \kw{else}\ \{ \\[-0.25em]
&\quad\quad\quad\quad \kw{if}\ (\neg best\_verified)\ \{\ best := y;\ \} \\[-0.25em]
&\quad\quad\quad\quad prev\_ver\_err := v\_err; \\[-0.25em]
&\quad\quad\quad \} \\[-0.25em]
&\quad\quad \}\ \kw{else}\ \{ \\[-0.25em]
&\quad\quad\quad prev\_diff\_err := diff\_err; \\[-0.25em]
&\quad\quad \} \\[-0.25em]
&\quad\quad attempt := attempt + 1; \\[-0.25em]
&\quad \} \\[-0.25em]
&\quad r := best;\ \kw{return}; \\[-0.25em]
&\}
\end{aligned}
\]
\end{tcolorbox}
\caption{\small Verified agent program for Dafny annotation synthesis. On the first iteration, $\funcIt{initialFGGM}$ (Fig.~\ref{fig:initialFGGMDef}) generates an initial annotation. On subsequent iterations, the agent dispatches to $\funcIt{diffErrorFGGM}$ (Fig.~\ref{fig:diffErrorFGGMDef}) or $\funcIt{verifierErrorFGGM}$ (Fig.~\ref{fig:verifierErrorFGGMDef}) depending on the previous error type. Each FGGM guarantees $\funcIt{noDiff}(base\_program, y)$ via a rejection sampler with $K=5$ and stateful $\ferr$ feedback, maintaining the loop invariant. The fallback fires on only 5.1\% of instances (HumanEvalDafny) and 7.9\% (DafnyBench), confirming that the $\ferr$-guided sampler produces conforming outputs in most cases. Upon exhausting all attempts, the agent returns $best$, which satisfies the postcondition by the loop invariant.}
\label{fig:exampleDafnyAgentNew}
\end{figure}

\newpage
\section{Example Agent for GSM-Symbolic}
\label{appen:exampleGSMAgent}

\noindent The following is an example verified agent program for the GSM-Symbolic math synthesis task. The agent defines and invokes four FGGMs, $\funcIt{initialFGGM}$ (Fig.~\ref{fig:gsmInitialFGGMDef}), $\funcIt{parseErrorFGGM1}$ (Fig.~\ref{fig:gsmParseError1Def}), $\funcIt{parseErrorFGGM2}$ (Fig.~\ref{fig:gsmParseError2Def}), and $\funcIt{parseErrorFGGM3}$ (Fig.~\ref{fig:gsmParseError3Def}), each with local contract $\loutspec{}: \funcIt{gsmParser}(\cdot)$ enforcing syntactic validity of the output expression and a fallback that returns ``0'' (a trivially parsable expression). All four share $GMid := \funcIt{qwen}$ and differ in the prompting function $\prompt$ and output processor $\fpost$: the first generates an initial math expression, while the remaining three provide increasingly focused repair strategies for parse errors. In each FGGM, $\fpost$ extracts the expression wrapped in \texttt{<{}< >{}>} from the raw model response, and the auto-synthesized stateful error function $\ferr$ (which reports the failing atomic formulas of $\loutspec{}$, i.e., parse-failure details) is threaded back to the prompt through its $err$ argument, enabling the rejection sampler (Fig.~\ref{fig:rejectionSampler}) to avoid repeating the same syntactic mistakes. The agent program (Fig.~\ref{fig:gsmAgentProgram}) first attempts generation via $\funcIt{initialFGGM}$, then dispatches to the appropriate parse-error FGGM on subsequent iterations. On GSM-Symbolic, the fallback fires on only 2.1\% of instances (Table~\ref{tab:fallback}), confirming that the $\ferr$-guided sampler produces syntactically valid expressions in nearly all cases.

\begin{figure}[H]
\begin{tcolorbox}[
colback=lightgraybg,
colframe=lightgraybg,
boxrule=0pt,
arc=2pt,
left=4pt,right=4pt,top=4pt,bottom=4pt
]
\scriptsize
\[
\begin{aligned}
&\textbf{FGGM Definition:} \\[-0.25em]
&\quad id := \funcIt{initialFGGM} \\[-0.25em]
&\quad GMid := \funcIt{qwen} \\[-0.25em]
&\quad typeSig := (question:\typeit{str},\ variable\_types:\typeit{str}) \to \typeit{str} \\[-0.25em]
&\quad \linspec{} := \kw{true} \\[-0.25em]
&\quad \loutspec{} := \funcIt{gsmParser}(f(question,\ variable\_types)) \\[-0.25em]
&\quad \prompt := \kw{function}\ \prompt(question:\typeit{str},\ variable\_types:\typeit{str},\ err:\typeit{str}) : (\typeit{str})\ \{ \\[-0.25em]
&\quad\quad \kw{var}\ p:\typeit{str} := \\[-0.25em]
&\quad\quad\quad \text{``You are a math expert solving grade-school word problems.\textbackslash n\textbackslash n''} \\[-0.25em]
&\quad\quad\quad + \text{``CRITICAL RULES:\textbackslash n''} \\[-0.25em]
&\quad\quad\quad + \text{``1. Use ONLY the variable names provided in `Variable types'\textbackslash n''} \\[-0.25em]
&\quad\quad\quad + \text{``2. ONLY use operators: +, -, *, /, //, \%, **, int(), parentheses, numbers\textbackslash n''} \\[-0.25em]
&\quad\quad\quad + \text{``3. NO functions like round(), abs(), pow(), etc.\textbackslash n''} \\[-0.25em]
&\quad\quad\quad + \text{``4. Use // for integer division, / for regular division\textbackslash n''} \\[-0.25em]
&\quad\quad\quad + \text{``5. Return ONLY the mathematical expression\textbackslash n\textbackslash n''}; \\[0.2em]
&\quad\quad p := p + \text{``Question: ''} + question + \text{``\textbackslash n\textbackslash n''} \\[-0.25em]
&\quad\quad\quad + \text{``Variable types: ''} + variable\_types + \text{``\textbackslash n\textbackslash n''} \\[-0.25em]
&\quad\quad\quad + \text{``Provide your answer as a single expression wrapped in <{}< >{}>.''};\\[-0.25em]
&\quad\quad \kw{return}\ p + err;\ \} \\[0.3em]
&\quad \fallback := \kw{function}\ \fallback(question:\typeit{str},\ variable\_types:\typeit{str},\ y:\typeit{str}) : (\typeit{str}) \\[-0.25em]
&\quad\quad\quad \kw{ensures}\ \funcIt{gsmParser}(\fallback(question,\ variable\_types,\ y)) \\[-0.25em]
&\quad \{ \\[-0.25em]
&\quad\quad\quad \kw{return}\ \text{``0''};\\[-0.25em]
&\quad \} \\[0.3em]
&\quad \fpost := \kw{function}\ \fpost(y_{raw}:\typeit{str}) : (\typeit{str})\ \{ \\[-0.25em]
&\quad\quad\quad \kw{return}\ \funcIt{extractExpression}(y_{raw});\ \}
\end{aligned}
\]
\end{tcolorbox}
\caption{\small FGGM definition for $\funcIt{initialFGGM}$ (GSM-Symbolic). The prompting function provides rules and examples for generating valid math expressions and appends the auto-synthesized rejection feedback $err$, enabling the rejection sampler to avoid repeating syntactic mistakes. The output processor $\fpost$ extracts the expression wrapped in \texttt{<{}< >{}>} from the raw model response. The fallback returns ``0'', a trivially valid expression. Fallback trigger rate: 2.1\% (Table~\ref{tab:fallback}).}
\label{fig:gsmInitialFGGMDef}
\end{figure}

\begin{figure}[H]
\begin{tcolorbox}[
colback=lightgraybg,
colframe=lightgraybg,
boxrule=0pt,
arc=2pt,
left=4pt,right=4pt,top=4pt,bottom=4pt
]
\scriptsize
\[
\begin{aligned}
&\textbf{FGGM Definition:} \\[-0.25em]
&\quad id := \funcIt{parseErrorFGGM1} \\[-0.25em]
&\quad GMid := \funcIt{qwen} \\[-0.25em]
&\quad typeSig := (question:\typeit{str},\ variable\_types:\typeit{str},\ parse\_error:\typeit{str},\ prev\_expr:\typeit{str}) \to \typeit{str} \\[-0.25em]
&\quad \linspec{} := \kw{true} \\[-0.25em]
&\quad \loutspec{} := \funcIt{gsmParser}(f(question,\ variable\_types,\ parse\_error,\ prev\_expr)) \\[-0.25em]
&\quad \prompt := \kw{function}\ \prompt(question:\typeit{str},\ variable\_types:\typeit{str}, \\[-0.25em]
&\quad\quad\quad\quad parse\_error:\typeit{str},\ prev\_expr:\typeit{str},\ err:\typeit{str}) : (\typeit{str})\ \{ \\[-0.25em]
&\quad\quad \kw{var}\ p:\typeit{str} := \\[-0.25em]
&\quad\quad\quad \text{``Fix this expression. Use ONLY variable names from the list.\textbackslash n\textbackslash n''} \\[-0.25em]
&\quad\quad\quad + \text{``Failed expression: ''} + prev\_expr + \text{``\textbackslash n''} \\[-0.25em]
&\quad\quad\quad + \text{``Error: ''} + parse\_error + \text{``\textbackslash n\textbackslash n''} \\[-0.25em]
&\quad\quad\quad + \text{``Variable types: ''} + variable\_types + \text{``\textbackslash n''} \\[-0.25em]
&\quad\quad\quad + \text{``Question: ''} + question + \text{``\textbackslash n\textbackslash n''} \\[-0.25em]
&\quad\quad\quad + \text{``Rules:\textbackslash n''} \\[-0.25em]
&\quad\quad\quad + \text{``- Use exact variable names from Variable types (case-sensitive)\textbackslash n''} \\[-0.25em]
&\quad\quad\quad + \text{``- Only operators: +, -, *, /, //, \%, **, int(), parentheses\textbackslash n''} \\[-0.25em]
&\quad\quad\quad + \text{``- NO round(), abs(), pow() or other functions\textbackslash n\textbackslash n''} \\[-0.25em]
&\quad\quad\quad + \text{``Return corrected expression in <{}< >{}>.''};\\[-0.25em]
&\quad\quad \kw{return}\ p + err;\ \} \\[0.3em]
&\quad \fallback := \kw{function}\ \fallback(question:\typeit{str},\ variable\_types:\typeit{str}, \\[-0.25em]
&\quad\quad\quad\quad parse\_error:\typeit{str},\ prev\_expr:\typeit{str},\ y:\typeit{str}) : (\typeit{str}) \\[-0.25em]
&\quad\quad\quad \kw{ensures}\ \funcIt{gsmParser}(\fallback(question,\ variable\_types,\ parse\_error,\ prev\_expr,\ y)) \\[-0.25em]
&\quad \{ \\[-0.25em]
&\quad\quad\quad \kw{return}\ \text{``0''};\\[-0.25em]
&\quad \} \\[0.3em]
&\quad \fpost := \kw{function}\ \fpost(y_{raw}:\typeit{str}) : (\typeit{str})\ \{ \\[-0.25em]
&\quad\quad\quad \kw{return}\ \funcIt{extractExpression}(y_{raw});\ \}
\end{aligned}
\]
\end{tcolorbox}
\caption{\small FGGM definition for $\funcIt{parseErrorFGGM1}$ (GSM-Symbolic). Invoked on the first parse-error recovery attempt. The prompt provides the failed expression and error details, and appends the auto-synthesized rejection feedback $err$ that accumulates parse errors across samples. The output processor $\fpost$ extracts the corrected expression from the raw model response.}
\label{fig:gsmParseError1Def}
\end{figure}

\begin{figure}[H]
\begin{tcolorbox}[
colback=lightgraybg,
colframe=lightgraybg,
boxrule=0pt,
arc=2pt,
left=4pt,right=4pt,top=4pt,bottom=4pt
]
\scriptsize
\[
\begin{aligned}
&\textbf{FGGM Definition:} \\[-0.25em]
&\quad id := \funcIt{parseErrorFGGM2} \\[-0.25em]
&\quad GMid := \funcIt{qwen} \\[-0.25em]
&\quad typeSig := (question:\typeit{str},\ variable\_types:\typeit{str},\ parse\_error:\typeit{str},\ prev\_expr:\typeit{str}) \to \typeit{str} \\[-0.25em]
&\quad \linspec{} := \kw{true} \\[-0.25em]
&\quad \loutspec{} := \funcIt{gsmParser}(f(question,\ variable\_types,\ parse\_error,\ prev\_expr)) \\[-0.25em]
&\quad \prompt := \kw{function}\ \prompt(question:\typeit{str},\ variable\_types:\typeit{str}, \\[-0.25em]
&\quad\quad\quad\quad parse\_error:\typeit{str},\ prev\_expr:\typeit{str},\ err:\typeit{str}) : (\typeit{str})\ \{ \\[-0.25em]
&\quad\quad \kw{var}\ p:\typeit{str} := \\[-0.25em]
&\quad\quad\quad \text{``The expression has errors. Fix it carefully.\textbackslash n\textbackslash n''} \\[-0.25em]
&\quad\quad\quad + \text{``Failed: ''} + prev\_expr + \text{``\textbackslash n''} \\[-0.25em]
&\quad\quad\quad + \text{``Error: ''} + parse\_error + \text{``\textbackslash n''} \\[-0.25em]
&\quad\quad\quad + \text{``Variables: ''} + variable\_types + \text{``\textbackslash n''} \\[-0.25em]
&\quad\quad\quad + \text{``Question: ''} + question + \text{``\textbackslash n\textbackslash n''} \\[-0.25em]
&\quad\quad\quad + \text{``Check:\textbackslash n''} \\[-0.25em]
&\quad\quad\quad + \text{``1. Are all variable names exactly as listed?\textbackslash n''} \\[-0.25em]
&\quad\quad\quad + \text{``2. Are you using only: +, -, *, /, //, \%, **, int(), ()?\textbackslash n''} \\[-0.25em]
&\quad\quad\quad + \text{``3. No round(), abs(), or other functions?\textbackslash n\textbackslash n''} \\[-0.25em]
&\quad\quad\quad + \text{``Corrected expression in <{}< >{}>:''};\\[-0.25em]
&\quad\quad \kw{return}\ p + err;\ \} \\[0.3em]
&\quad \fallback := \kw{function}\ \fallback(question:\typeit{str},\ variable\_types:\typeit{str}, \\[-0.25em]
&\quad\quad\quad\quad parse\_error:\typeit{str},\ prev\_expr:\typeit{str},\ y:\typeit{str}) : (\typeit{str}) \\[-0.25em]
&\quad\quad\quad \kw{ensures}\ \funcIt{gsmParser}(\fallback(question,\ variable\_types,\ parse\_error,\ prev\_expr,\ y)) \\[-0.25em]
&\quad \{ \\[-0.25em]
&\quad\quad\quad \kw{return}\ \text{``0''};\\[-0.25em]
&\quad \} \\[0.3em]
&\quad \fpost := \kw{function}\ \fpost(y_{raw}:\typeit{str}) : (\typeit{str})\ \{ \\[-0.25em]
&\quad\quad\quad \kw{return}\ \funcIt{extractExpression}(y_{raw});\ \}
\end{aligned}
\]
\end{tcolorbox}
\caption{\small FGGM definition for $\funcIt{parseErrorFGGM2}$ (GSM-Symbolic). Invoked on the second parse-error recovery attempt with a checklist-style prompt that appends the auto-synthesized rejection feedback $err$. The output processor $\fpost$ extracts the corrected expression from the raw model response.}
\label{fig:gsmParseError2Def}
\end{figure}

\begin{figure}[H]
\begin{tcolorbox}[
colback=lightgraybg,
colframe=lightgraybg,
boxrule=0pt,
arc=2pt,
left=4pt,right=4pt,top=4pt,bottom=4pt
]
\scriptsize
\[
\begin{aligned}
&\textbf{FGGM Definition:} \\[-0.25em]
&\quad id := \funcIt{parseErrorFGGM3} \\[-0.25em]
&\quad GMid := \funcIt{qwen} \\[-0.25em]
&\quad typeSig := (question:\typeit{str},\ variable\_types:\typeit{str},\ parse\_error:\typeit{str},\ prev\_expr:\typeit{str}) \to \typeit{str} \\[-0.25em]
&\quad \linspec{} := \kw{true} \\[-0.25em]
&\quad \loutspec{} := \funcIt{gsmParser}(f(question,\ variable\_types,\ parse\_error,\ prev\_expr)) \\[-0.25em]
&\quad \prompt := \kw{function}\ \prompt(question:\typeit{str},\ variable\_types:\typeit{str}, \\[-0.25em]
&\quad\quad\quad\quad parse\_error:\typeit{str},\ prev\_expr:\typeit{str},\ err:\typeit{str}) : (\typeit{str})\ \{ \\[-0.25em]
&\quad\quad \kw{return}\ \text{``Fix to use only: +, -, *, /, //, \%, **, int(), parentheses ''} \\[-0.25em]
&\quad\quad\quad + \text{``and variables from: ''} + variable\_types + \text{``\textbackslash n\textbackslash n''} \\[-0.25em]
&\quad\quad\quad + \text{``Expression: ''} + prev\_expr + \text{``\textbackslash n''} \\[-0.25em]
&\quad\quad\quad + \text{``Error: ''} + parse\_error + \text{``\textbackslash n\textbackslash n''} \\[-0.25em]
&\quad\quad\quad + \text{``Corrected in <{}< >{}>:''} + err;\ \} \\[0.3em]
&\quad \fallback := \kw{function}\ \fallback(question:\typeit{str},\ variable\_types:\typeit{str}, \\[-0.25em]
&\quad\quad\quad\quad parse\_error:\typeit{str},\ prev\_expr:\typeit{str},\ y:\typeit{str}) : (\typeit{str}) \\[-0.25em]
&\quad\quad\quad \kw{ensures}\ \funcIt{gsmParser}(\fallback(question,\ variable\_types,\ parse\_error,\ prev\_expr,\ y)) \\[-0.25em]
&\quad \{ \\[-0.25em]
&\quad\quad\quad \kw{return}\ \text{``0''};\\[-0.25em]
&\quad \} \\[0.3em]
&\quad \fpost := \kw{function}\ \fpost(y_{raw}:\typeit{str}) : (\typeit{str})\ \{ \\[-0.25em]
&\quad\quad\quad \kw{return}\ \funcIt{extractExpression}(y_{raw});\ \}
\end{aligned}
\]
\end{tcolorbox}
\caption{\small FGGM definition for $\funcIt{parseErrorFGGM3}$ (GSM-Symbolic). A minimal last-resort prompt used on the third and subsequent parse-error attempts, appending the auto-synthesized rejection feedback $err$. The output processor $\fpost$ extracts the corrected expression from the raw model response.}
\label{fig:gsmParseError3Def}
\end{figure}

\newpage
\begin{figure}[H]
\begin{tcolorbox}[
colback=lightgraybg,
colframe=lightgraybg,
boxrule=0pt,
arc=2pt,
left=4pt,right=4pt,top=4pt,bottom=4pt
]
\scriptsize
\[
\begin{aligned}
&\kw{method}\ \funcIt{agent}(question:\typeit{str},\ variable\_types:\typeit{str})\ \kw{returns}\ (r:\typeit{str}) \\[-0.25em]
&\quad \kw{ensures}\ \funcIt{gsmParser}(r) \\[-0.25em]
&\{ \\[-0.25em]
&\quad \kw{var}\ max\_attempts:\typeit{int} := 7; \\[-0.25em]
&\quad \kw{var}\ attempt:\typeit{int} := 0; \\[-0.25em]
&\quad \kw{var}\ result:\typeit{str} := \text{``0''}; \\[0.3em]
&\quad \text{// --- initial generation ($\fpost$ extracts the expression) ---} \\[-0.25em]
&\quad result := \funcIt{initialFGGM}(question,\ variable\_types); \\[0.3em]
&\quad \text{// --- repair loop ---} \\[-0.25em]
&\quad \kw{while}\ (attempt < max\_attempts) \\[-0.25em]
&\quad\quad \kw{invariant}\ 0 \leq attempt \leq max\_attempts \\[-0.25em]
&\quad\quad \kw{decreases}\ max\_attempts - attempt \\[-0.25em]
&\quad \{ \\[-0.25em]
&\quad\quad \kw{if}\ (result = \text{``''})\ \{\ result := \text{``0''};\ \kw{break};\ \} \\[0.2em]
&\quad\quad \kw{var}\ parse\_error:\typeit{str} := \funcIt{gsmParserWithErrorMsg}(result); \\[0.2em]
&\quad\quad \kw{if}\ (parse\_error = \text{``''})\ \{ \\[-0.25em]
&\quad\quad\quad r := result;\ \kw{return}; \\[-0.25em]
&\quad\quad \} \\[0.2em]
&\quad\quad \text{// --- FGGM dispatch based on attempt ---} \\[-0.25em]
&\quad\quad attempt := attempt + 1; \\[-0.25em]
&\quad\quad \kw{var}\ fix\_response:\typeit{str}; \\[0.2em]
&\quad\quad \kw{if}\ (attempt = 1)\ \{ \\[-0.25em]
&\quad\quad\quad fix\_response := \funcIt{parseErrorFGGM1}(question,\ variable\_types,\ parse\_error,\ result); \\[-0.25em]
&\quad\quad \}\ \kw{else if}\ (attempt = 2)\ \{ \\[-0.25em]
&\quad\quad\quad fix\_response := \funcIt{parseErrorFGGM2}(question,\ variable\_types,\ parse\_error,\ result); \\[-0.25em]
&\quad\quad \}\ \kw{else}\ \{ \\[-0.25em]
&\quad\quad\quad fix\_response := \funcIt{parseErrorFGGM3}(question,\ variable\_types,\ parse\_error,\ result); \\[-0.25em]
&\quad\quad \} \\[0.2em]
&\quad\quad \kw{var}\ new\_result:\typeit{str} := fix\_response; \\[0.2em]
&\quad\quad \kw{if}\ (new\_result \neq \text{``''}\ \wedge\ new\_result \neq result)\ \{ \\[-0.25em]
&\quad\quad\quad result := new\_result; \\[-0.25em]
&\quad\quad \}\ \kw{else}\ \{ \\[-0.25em]
&\quad\quad\quad \kw{if}\ (attempt \geq max\_attempts - 1)\ \{\ \kw{break};\ \} \\[-0.25em]
&\quad\quad \} \\[-0.25em]
&\quad \} \\[0.3em]
&\quad \text{// --- final validation ---} \\[-0.25em]
&\quad \kw{if}\ (result \neq \text{``''}\ \wedge\ \funcIt{gsmParser}(result))\ \{ \\[-0.25em]
&\quad\quad r := result;\ \kw{return}; \\[-0.25em]
&\quad \} \\[-0.25em]
&\quad r := \text{``0''};\ \kw{return}; \\[-0.25em]
&\}
\end{aligned}
\]
\end{tcolorbox}
\caption{\small Verified agent program for GSM-Symbolic. The agent first generates an expression via $\funcIt{initialFGGM}$, then enters a repair loop dispatching to $\funcIt{parseErrorFGGM1}$/$\funcIt{parseErrorFGGM2}$/$\funcIt{parseErrorFGGM3}$ based on the attempt number. Each FGGM guarantees $\funcIt{gsmParser}(\cdot)$ via a rejection sampler with $K=5$ and stateful $\ferr$ feedback. The fallback fires on only 2.1\% of instances (Table~\ref{tab:fallback}).}
\label{fig:gsmAgentProgram}
\end{figure}

\newpage
\section{Example Agent for $\tau^2$-bench}
\label{appen:exampleTauAgent}

\noindent The following is an example verified agent program for the $\tau^2$-bench multi-turn tool-use task. The agent defines and invokes three FGGMs, $\funcIt{initialFGGM}$ (Fig.~\ref{fig:tauInitialFGGMDef}), $\funcIt{duplicateFGGM}$ (Fig.~\ref{fig:tauDuplicateFGGMDef}), and $\funcIt{agentCErrorFGGM}$ (Fig.~\ref{fig:tauAgentCErrorFGGMDef}), each of which returns an \emph{array} of tool-call actions of type $arr\langle\typeit{str}\rangle$ using the built-in array primitive, rather than a single action. Each FGGM has local contract $\loutspec{}:\funcIt{agentCCheckList}(\cdot,\ trace,\ domain)$, a sequence-level checker ensuring that every action in the generated array satisfies the Agent-C temporal policy given the trace extended by the preceding actions. Here $trace$ and $domain$ denote the running tool-call trace and the domain identifier (the $tool\_trace$ and $domain$ parameters of the agent method in Fig.~\ref{fig:tauAgentProgram}); they are ambient contract inputs threaded to every FGGM by the framework, so the $typeSig$ lists only the task-specific arguments. The fallback returns a singleton array that transfers to a human agent, which trivially satisfies the policy. All three share $GMid := \funcIt{qwen}$ and differ in prompting function $\prompt$ and output processor $\fpost$: the first generates an initial action array, the second handles duplicate action-array detection, and the third recovers from Agent-C constraint violations with targeted guidance. In each FGGM, $\fpost$ parses the ordered array of tool calls from the raw model response, while the auto-synthesized stateful error function $\ferr$ (which reports the failing atomic formulas of $\loutspec{}$, i.e., policy-violation details) is threaded back to the prompt through its $err$ argument, enabling the rejection sampler (Fig.~\ref{fig:rejectionSampler}) to avoid repeating the same policy-violating action arrays. The agent program (Fig.~\ref{fig:tauAgentProgram}) dispatches to the appropriate FGGM based on duplicate detection and constraint-check results. On $\tau^2$-bench, the fallback fires on 15.6\% (retail) and 21.6\% (airline) of instances (Table~\ref{tab:fallback}).

\begin{figure}[H]
\begin{tcolorbox}[
colback=lightgraybg,
colframe=lightgraybg,
boxrule=0pt,
arc=2pt,
left=4pt,right=4pt,top=4pt,bottom=4pt
]
\scriptsize
\[
\begin{aligned}
&\textbf{FGGM Definition:} \\[-0.25em]
&\quad id := \funcIt{initialFGGM} \\[-0.25em]
&\quad GMid := \funcIt{qwen} \\[-0.25em]
&\quad typeSig := (conversation:\typeit{str},\ policy:\typeit{str},\ tools:\typeit{str}) \to arr\langle\typeit{str}\rangle \\[-0.25em]
&\quad \linspec{} := \kw{true} \\[-0.25em]
&\quad \loutspec{} := \funcIt{agentCCheckList}(f(\dots),\ trace,\ domain) \\[-0.25em]
&\quad \prompt := \kw{function}\ \prompt(conversation:\typeit{str},\ policy:\typeit{str},\ tools:\typeit{str},\ err:\typeit{str}) : (\typeit{str})\ \{ \\[-0.25em]
&\quad\quad \kw{var}\ sys:\typeit{str} := policy \\[-0.25em]
&\quad\quad\quad + \text{``\textbackslash n\textbackslash nCRITICAL INSTRUCTIONS:\textbackslash n''} \\[-0.25em]
&\quad\quad\quad + \text{``1. When user confirms an action, EXECUTE it immediately with the correct tool call\textbackslash n''} \\[-0.25em]
&\quad\quad\quad + \text{``2. For reservation updates: call update\_reservation\_flights with NEW flights\textbackslash n''} \\[-0.25em]
&\quad\quad\quad + \text{``3. For basic economy changes: cancel\_reservation then book\_reservation\textbackslash n''} \\[-0.25em]
&\quad\quad\quad + \text{``4. Always communicate specific numbers: prices, refund amounts, IDs\textbackslash n''} \\[-0.25em]
&\quad\quad\quad + \text{``5. When no direct flights exist, try search\_onestop\_flight\textbackslash n''} \\[-0.25em]
&\quad\quad\quad + \text{``6. Only transfer to human after genuinely attempting required actions\textbackslash n''} \\[-0.25em]
&\quad\quad\quad + \text{``7. Emit the full ordered array of tool-call actions to take this turn\textbackslash n''}; \\[-0.25em]
&\quad\quad \kw{return}\ \text{``[system: ''} + sys + \text{``] ''} + conversation + err;\ \} \\[0.3em]
&\quad \fallback := \kw{function}\ \fallback(conversation:\typeit{str},\ policy:\typeit{str},\ tools:\typeit{str},\ y:arr\langle\typeit{str}\rangle) : (arr\langle\typeit{str}\rangle) \\[-0.25em]
&\quad\quad\quad \kw{ensures}\ \funcIt{agentCCheckList}(\fallback(\dots),\ trace,\ domain) \\[-0.25em]
&\quad \{ \\[-0.25em]
&\quad\quad\quad \kw{return}\ [\funcIt{transferToHuman}(\text{``Unable to process request''})];\\[-0.25em]
&\quad \} \\[0.3em]
&\quad \fpost := \kw{function}\ \fpost(y_{raw}:\typeit{str}) : (arr\langle\typeit{str}\rangle)\ \{ \\[-0.25em]
&\quad\quad\quad \kw{return}\ \funcIt{parseToolCalls}(y_{raw});\ \}
\end{aligned}
\]
\end{tcolorbox}
\caption{\small FGGM definition for $\funcIt{initialFGGM}$ ($\tau^2$-bench). The FGGM emits an \emph{array} of tool-call actions ($arr\langle\typeit{str}\rangle$), and $\loutspec{}$ uses the sequence-level checker $\funcIt{agentCCheckList}$ that verifies every action complies with the Agent-C policy given the trace extended by the preceding actions. The prompting function provides domain-specific instructions and appends the auto-synthesized rejection feedback $err$, enabling the rejection sampler to avoid repeating non-compliant action arrays. The output processor $\fpost$ parses the ordered array of tool calls from the raw model response, and the fallback returns a singleton array that transfers to a human agent. Fallback trigger rate: 15.6\% retail, 21.6\% airline (Table~\ref{tab:fallback}).}
\label{fig:tauInitialFGGMDef}
\end{figure}

\begin{figure}[H]
\begin{tcolorbox}[
colback=lightgraybg,
colframe=lightgraybg,
boxrule=0pt,
arc=2pt,
left=4pt,right=4pt,top=4pt,bottom=4pt
]
\scriptsize
\[
\begin{aligned}
&\textbf{FGGM Definition:} \\[-0.25em]
&\quad id := \funcIt{duplicateFGGM} \\[-0.25em]
&\quad GMid := \funcIt{qwen} \\[-0.25em]
&\quad typeSig := (messages:\typeit{str},\ prev\_response:\typeit{str},\ tool\_name:\typeit{str},\ tools:\typeit{str}) \to arr\langle\typeit{str}\rangle \\[-0.25em]
&\quad \linspec{} := \kw{true} \\[-0.25em]
&\quad \loutspec{} := \funcIt{agentCCheckList}(f(\dots),\ trace,\ domain) \\[-0.25em]
&\quad \prompt := \kw{function}\ \prompt(messages:\typeit{str},\ prev\_response:\typeit{str}, \\[-0.25em]
&\quad\quad\quad\quad tool\_name:\typeit{str},\ tools:\typeit{str},\ err:\typeit{str}) : (\typeit{str})\ \{ \\[-0.25em]
&\quad\quad \kw{var}\ feedback:\typeit{str} := \\[-0.25em]
&\quad\quad\quad \text{``You are repeating the same tool call. Try a DIFFERENT approach:\textbackslash n''} \\[-0.25em]
&\quad\quad\quad + \text{``1. If you have user confirmation, EXECUTE the action tool NOW.\textbackslash n''} \\[-0.25em]
&\quad\quad\quad + \text{``2. If missing prerequisites: call get\_user\_details or get\_reservation\_details ONCE.\textbackslash n''} \\[-0.25em]
&\quad\quad\quad + \text{``3. If no direct flights found, try search\_onestop\_flight.\textbackslash n''} \\[-0.25em]
&\quad\quad\quad + \text{``4. For basic economy: cancel\_reservation + book\_reservation.\textbackslash n''} \\[-0.25em]
&\quad\quad\quad + \text{``5. Stop gathering information and START EXECUTING the confirmed action.''}; \\[0.2em]
&\quad\quad \kw{return}\ messages + \text{`` + [assistant: ''} + prev\_response \\[-0.25em]
&\quad\quad\quad + \text{``, user: ''} + feedback + \text{``]''} + err;\ \} \\[0.3em]
&\quad \fallback := \kw{function}\ \fallback(messages:\typeit{str},\ prev\_response:\typeit{str}, \\[-0.25em]
&\quad\quad\quad\quad tool\_name:\typeit{str},\ tools:\typeit{str},\ y:arr\langle\typeit{str}\rangle) : (arr\langle\typeit{str}\rangle) \\[-0.25em]
&\quad\quad\quad \kw{ensures}\ \funcIt{agentCCheckList}(\fallback(\dots),\ trace,\ domain) \\[-0.25em]
&\quad \{ \\[-0.25em]
&\quad\quad\quad \kw{return}\ [\funcIt{transferToHuman}(\text{``Unable to satisfy prerequisites. Tool: ''} + tool\_name)];\\[-0.25em]
&\quad \} \\[0.3em]
&\quad \fpost := \kw{function}\ \fpost(y_{raw}:\typeit{str}) : (arr\langle\typeit{str}\rangle)\ \{ \\[-0.25em]
&\quad\quad\quad \kw{return}\ \funcIt{parseToolCalls}(y_{raw});\ \}
\end{aligned}
\]
\end{tcolorbox}
\caption{\small FGGM definition for $\funcIt{duplicateFGGM}$ ($\tau^2$-bench). Invoked when the generated action array duplicates the previous one. The prompt instructs the model to try a different approach and appends the auto-synthesized rejection feedback $err$ to prevent repeated non-compliant action arrays. The FGGM returns an array of tool-call actions ($arr\langle\typeit{str}\rangle$); $\fpost$ parses this array from the raw model response.}
\label{fig:tauDuplicateFGGMDef}
\end{figure}

\newpage
\begin{figure}[H]
\begin{tcolorbox}[
colback=lightgraybg,
colframe=lightgraybg,
boxrule=0pt,
arc=2pt,
left=4pt,right=4pt,top=4pt,bottom=4pt
]
\scriptsize
\[
\begin{aligned}
&\textbf{FGGM Definition:} \\[-0.25em]
&\quad id := \funcIt{agentCErrorFGGM} \\[-0.25em]
&\quad GMid := \funcIt{qwen} \\[-0.25em]
&\quad typeSig := (messages:\typeit{str},\ prev\_response:\typeit{str},\ agentc\_error:\typeit{str}, \\[-0.25em]
&\quad\quad\quad\quad tool\_name:\typeit{str},\ tools:\typeit{str}) \to arr\langle\typeit{str}\rangle \\[-0.25em]
&\quad \linspec{} := \kw{true} \\[-0.25em]
&\quad \loutspec{} := \funcIt{agentCCheckList}(f(\dots),\ trace,\ domain) \\[-0.25em]
&\quad \prompt := \kw{function}\ \prompt(messages:\typeit{str},\ prev\_response:\typeit{str},\ agentc\_error:\typeit{str}, \\[-0.25em]
&\quad\quad\quad\quad tool\_name:\typeit{str},\ tools:\typeit{str},\ err:\typeit{str}) : (\typeit{str})\ \{ \\[-0.25em]
&\quad\quad \kw{var}\ guidance:\typeit{str} := \text{``Tool call `''} + tool\_name \\[-0.25em]
&\quad\quad\quad + \text{``' failed agentC constraint: ''} + agentc\_error + \text{``\textbackslash n\textbackslash n''}; \\[0.2em]
&\quad\quad \kw{if}\ (\funcIt{isSubstring}(\text{``update\_reservation''},\ tool\_name))\ \{ \\[-0.25em]
&\quad\quad\quad guidance := guidance \\[-0.25em]
&\quad\quad\quad\quad + \text{``To update a reservation you need:\textbackslash n''} \\[-0.25em]
&\quad\quad\quad\quad + \text{``1. Call get\_reservation\_details first to get current state.\textbackslash n''} \\[-0.25em]
&\quad\quad\quad\quad + \text{``2. Check if cabin is basic\_economy -- if yes, cancel then book.\textbackslash n''} \\[-0.25em]
&\quad\quad\quad\quad + \text{``3. For flight changes: search for new flights first.\textbackslash n''} \\[-0.25em]
&\quad\quad\quad\quad + \text{``Try calling the missing prerequisite tool now.\textbackslash n''};\ \} \\[0.2em]
&\quad\quad \kw{else if}\ (\funcIt{isSubstring}(\text{``cancel\_reservation''},\ tool\_name))\ \{ \\[-0.25em]
&\quad\quad\quad guidance := guidance \\[-0.25em]
&\quad\quad\quad\quad + \text{``To cancel: call get\_reservation\_details first to verify eligibility.\textbackslash n''} \\[-0.25em]
&\quad\quad\quad\quad + \text{``Verify cancellation rules are satisfied.\textbackslash n''};\ \} \\[0.2em]
&\quad\quad \kw{else if}\ (\funcIt{isSubstring}(\text{``book\_reservation''},\ tool\_name))\ \{ \\[-0.25em]
&\quad\quad\quad guidance := guidance \\[-0.25em]
&\quad\quad\quad\quad + \text{``To book: call get\_user\_details for payment methods,\textbackslash n''} \\[-0.25em]
&\quad\quad\quad\quad + \text{``search for flights, ensure all passenger details.\textbackslash n''};\ \} \\[0.2em]
&\quad\quad \kw{else}\ \{ \\[-0.25em]
&\quad\quad\quad guidance := guidance \\[-0.25em]
&\quad\quad\quad\quad + \text{``Call prerequisite tools first: get\_user\_details, get\_reservation\_details, or search flights.\textbackslash n''};\ \} \\[0.2em]
&\quad\quad \kw{return}\ messages + \text{`` + [assistant: ''} + prev\_response \\[-0.25em]
&\quad\quad\quad + \text{``, user: ''} + guidance + \text{``]''} + err;\ \} \\[0.3em]
&\quad \fallback := \kw{function}\ \fallback(messages:\typeit{str},\ prev\_response:\typeit{str},\ agentc\_error:\typeit{str}, \\[-0.25em]
&\quad\quad\quad\quad tool\_name:\typeit{str},\ tools:\typeit{str},\ y:arr\langle\typeit{str}\rangle) : (arr\langle\typeit{str}\rangle) \\[-0.25em]
&\quad\quad\quad \kw{ensures}\ \funcIt{agentCCheckList}(\fallback(\dots),\ trace,\ domain) \\[-0.25em]
&\quad \{ \\[-0.25em]
&\quad\quad\quad \kw{return}\ [\funcIt{transferToHuman}(\text{``Unable to satisfy prerequisites. Error: ''} + agentc\_error)];\\[-0.25em]
&\quad \} \\[0.3em]
&\quad \fpost := \kw{function}\ \fpost(y_{raw}:\typeit{str}) : (arr\langle\typeit{str}\rangle)\ \{ \\[-0.25em]
&\quad\quad\quad \kw{return}\ \funcIt{parseToolCalls}(y_{raw});\ \}
\end{aligned}
\]
\end{tcolorbox}
\caption{\small FGGM definition for $\funcIt{agentCErrorFGGM}$ ($\tau^2$-bench). Invoked when the generated action array fails the Agent-C policy check. The prompting function analyzes the offending tool name to provide targeted recovery guidance and appends the auto-synthesized rejection feedback $err$, enabling the stateful rejection sampler to progressively narrow toward compliant action arrays. The FGGM returns an array of tool-call actions ($arr\langle\typeit{str}\rangle$); $\fpost$ parses this array from the raw model response.}
\label{fig:tauAgentCErrorFGGMDef}
\end{figure}

\newpage
\begin{figure}[H]
\begin{tcolorbox}[
colback=lightgraybg,
colframe=lightgraybg,
boxrule=0pt,
arc=2pt,
left=4pt,right=4pt,top=4pt,bottom=4pt
]
\scriptsize
\[
\begin{aligned}
&\kw{method}\ \funcIt{agent}(conversation:\typeit{str},\ tool\_trace:\typeit{str},\ available\_tools:\typeit{str}, \\[-0.25em]
&\quad\quad policy:\typeit{str},\ domain:\typeit{str},\ tools:\typeit{str})\ \kw{returns}\ (r:arr\langle\typeit{str}\rangle) \\[-0.25em]
&\quad \kw{ensures}\ \funcIt{agentCCheckList}(r,\ tool\_trace,\ domain) \\[-0.25em]
&\{ \\[-0.25em]
&\quad \kw{var}\ max\_retries:\typeit{int} := 5; \\[-0.25em]
&\quad \kw{var}\ retry\_count:\typeit{int} := 0; \\[-0.25em]
&\quad \kw{var}\ prev\_actions:arr\langle\typeit{str}\rangle := []; \\[-0.25em]
&\quad \kw{var}\ prev\_response:\typeit{str} := \text{``''}; \\[0.3em]
&\quad \text{// --- initial FGGM call: generate an array of actions ---} \\[-0.25em]
&\quad \kw{var}\ actions:arr\langle\typeit{str}\rangle := \funcIt{initialFGGM}(conversation,\ policy,\ tools); \\[0.3em]
&\quad \text{// --- main loop ---} \\[-0.25em]
&\quad \kw{while}\ (retry\_count < max\_retries) \\[-0.25em]
&\quad\quad \kw{decreases}\ max\_retries - retry\_count \\[-0.25em]
&\quad \{ \\[-0.25em]
&\quad\quad \text{// --- duplicate detection over the whole action array ---} \\[-0.25em]
&\quad\quad \kw{if}\ (actions = prev\_actions)\ \{ \\[-0.25em]
&\quad\quad\quad retry\_count := retry\_count + 1; \\[-0.25em]
&\quad\quad\quad \kw{if}\ (retry\_count \geq max\_retries)\ \{ \\[-0.25em]
&\quad\quad\quad\quad r := [\funcIt{transferToHuman}(\text{``Unable to proceed.''})];\ \kw{return}; \\[-0.25em]
&\quad\quad\quad \} \\[-0.25em]
&\quad\quad\quad actions := \funcIt{duplicateFGGM}(conversation,\ prev\_response,\ \funcIt{firstToolName}(actions),\ tools); \\[-0.25em]
&\quad\quad\quad \kw{continue}; \\[-0.25em]
&\quad\quad \} \\[0.2em]
&\quad\quad prev\_actions := actions; \\[-0.25em]
&\quad\quad prev\_response := \funcIt{joinActions}(actions); \\[0.2em]
&\quad\quad \text{// --- sequence-level agentC constraint check ---} \\[-0.25em]
&\quad\quad \kw{var}\ error\_msg:\typeit{str} := \\[-0.25em]
&\quad\quad\quad \funcIt{agentCCheckListWithErrorMsg}(actions,\ tool\_trace,\ domain); \\[0.2em]
&\quad\quad \kw{if}\ (error\_msg = \text{``''})\ \{ \\[-0.25em]
&\quad\quad\quad r := actions;\ \kw{return}; \\[-0.25em]
&\quad\quad \}\ \kw{else}\ \{ \\[-0.25em]
&\quad\quad\quad retry\_count := retry\_count + 1; \\[-0.25em]
&\quad\quad\quad \kw{if}\ (retry\_count \geq max\_retries)\ \{ \\[-0.25em]
&\quad\quad\quad\quad r := [\funcIt{transferToHuman}(\text{``Constraint failures. Error: ''} + error\_msg)];\ \kw{return}; \\[-0.25em]
&\quad\quad\quad \} \\[-0.25em]
&\quad\quad\quad actions := \funcIt{agentCErrorFGGM}(conversation,\ prev\_response,\ error\_msg, \\[-0.25em]
&\quad\quad\quad\quad \funcIt{firstToolName}(actions),\ tools); \\[-0.25em]
&\quad\quad\quad \kw{continue}; \\[-0.25em]
&\quad\quad \} \\[-0.25em]
&\quad \} \\[0.3em]
&\quad \text{// --- exhausted retries ---} \\[-0.25em]
&\quad r := [\funcIt{transferToHuman}(\text{``Unable to process request after multiple attempts''})];\ \kw{return}; \\[-0.25em]
&\}
\end{aligned}
\]
\end{tcolorbox}
\caption{\small Verified agent program for $\tau^2$-bench. The agent generates an \emph{array} of tool-call actions via $\funcIt{initialFGGM}$, then enters a loop that detects duplicate action arrays (dispatching to $\funcIt{duplicateFGGM}$) and checks sequence-level Agent-C policy compliance with $\funcIt{agentCCheckListWithErrorMsg}$ (dispatching to $\funcIt{agentCErrorFGGM}$ on failure). Each FGGM returns a policy-compliant action array ($arr\langle\typeit{str}\rangle$) via a rejection sampler with $K=5$ and stateful $\ferr$ feedback, so the returned array satisfies the postcondition $\funcIt{agentCCheckList}(r,\ tool\_trace,\ domain)$. The fallback returns a singleton transfer-to-human array and fires on 15.6\% (retail) and 21.6\% (airline) of instances (Table~\ref{tab:fallback}).}
\label{fig:tauAgentProgram}
\end{figure}

\section{FGGM Syntax}
\label{appen:FGGMsyntax}
\noindent The syntax of a first-order guarded generative model (FGGM) can be described using the following BNF grammar. We use the production rules $(\langle \text{ident} \rangle, \langle \text{type} \rangle, \langle \text{program} \rangle, \langle qf\_formula \rangle)$ from Appendix~\ref{appen:seachSpace}.
\begin{align*}
\langle id \rangle       &::= \langle \text{ident} \rangle \\
\langle GMid \rangle     &::= \langle \text{ident} \rangle \\
\langle typeSig \rangle  &::= "(" \langle \text{typedVars} \rangle ")" "\to" \langle \text{type} \rangle \\[0.5ex]
\langle \text{typedVars} \rangle &::= \langle \text{typedVar} \rangle 
                                   \;\mid\; \langle \text{typedVar} \rangle , \langle \text{typedVars} \rangle \\[0.5ex]
\langle \text{typedVar} \rangle  &::= \langle \text{ident} \rangle : \langle \text{type} \rangle \\[1ex]
\langle \linspec{} \rangle  &::= \texttt{"requires"} \ \langle qf\_formula \rangle \\[0.5ex]
\langle \loutspec{} \rangle &::= \texttt{"ensures"} \ \langle qf\_formula \rangle \\[0.5ex]
\langle \prompt{} \rangle   &::= \langle \text{program} \rangle \\[0.5ex]
\langle \fallback{} \rangle &::= \langle \text{program} \rangle \\[0.5ex]
\langle \fpost \rangle &::= \langle \text{program} \rangle
\end{align*}
\section{FGGM Validity Check}
\label{appen:FGGMvalidity}
Algorithm~\ref{alg:validateFGGM} checks whether an FGGM definition is valid. It verifies that the type signature and local contracts are syntactically well-formed and that all terms using library functions satisfy their input specifications. It then checks that the prompting program and fallback program are type-correct and terminating. Finally, the fallback program is verified using the deductive verifier to ensure it satisfies the local contract for all inputs. The FGGM definition is accepted only if no errors are found.
\begin{algorithm}[H]
\footnotesize
\caption{validateFGGM}
\label{alg:validateFGGM}
\begin{algorithmic}[1]
\State \textbf{Input:} FGGM definition $\fggm{}=(id,GMid,typeSig,\linspec{},\loutspec{},\prompt,\fallback,\fpost)$
\State \textbf{Input:} Library functions $\fset{c}$, axioms $\axiom{}$
\State \textbf{Output:} $\{ \}$ if valid else set of errors $\msf{err}$

\State $\msf{err} \gets \{\}$

\If{$\neg\funcIt{WellTyped}(typeSig)$}
    \State $\msf{err} \gets \msf{err} \cup \{\text{Invalid type signature}\}$
\EndIf

\If{$\neg\funcIt{WellFormedFormula}(\linspec{})$}
    \State $\msf{err} \gets \msf{err} \cup \{\text{Invalid input contract}\}$
\EndIf

\If{$\neg\funcIt{WellFormedFormula}(\loutspec{})$}
    \State $\msf{err} \gets \msf{err} \cup \{\text{Invalid output contract}\}$
\EndIf

\If{$\neg\funcIt{CheckTerms}(\linspec{},\fset{c})$}
    \State $\msf{err} \gets \msf{err} \cup \{\text{Invalid library terms in }\linspec{}\}$
\EndIf

\If{$\neg\funcIt{CheckTerms}(\loutspec{},\fset{c})$}
    \State $\msf{err} \gets \msf{err} \cup \{\text{Invalid library terms in }\loutspec{}\}$
\EndIf

\If{$\neg\funcIt{TypeCheck}(\prompt)$ \textbf{or} $\neg\funcIt{Terminates}(\prompt)$}
    \State $\msf{err} \gets \msf{err} \cup \{\text{Invalid prompting program}\}$
\EndIf

\If{$\neg\funcIt{TypeCheck}(\fallback)$ \textbf{or} $\neg\funcIt{Terminates}(\fallback)$}
    \State $\msf{err} \gets \msf{err} \cup \{\text{Invalid fallback program}\}$
\EndIf

\If{$\neg\funcIt{Verify}(\forall x_1\dots x_n,y.\;\linspec{}(x_1\dots x_n) \Rightarrow \loutspec{}(x_1\dots x_n,\fallback(x_1\dots x_n,y)))$}
    \State $\msf{err} \gets \msf{err} \cup \{\text{Fallback violates contract}\}$
\EndIf

\State \Return $\msf{err}$
\end{algorithmic}
\end{algorithm}

Algorithm~\ref{alg:evalChecker} checks whether a concrete input–output tuple satisfies the FGGM output contract. It substitutes the concrete values into the quantifier-free specification $\loutspec{}$ and directly evaluates the resulting formula by recursively computing all terms using the library functions.
\begin{algorithm}[H]
\footnotesize
\caption{$\eval$ Contract Checker}
\label{alg:evalChecker}
\begin{algorithmic}[1]
\State \textbf{Input:} Concrete values $(x_1,\dots,x_n,y)$
\State \textbf{Input:} Local contracts $(\linspec{},\loutspec{})$, axioms $\axiom{}$
\State \textbf{Output:} $T$ if contract satisfied else $F$

\State $\phi \gets \funcIt{Substitute}(\loutspec{},\{x_1,\dots,x_n,y\})$

\If{$\phi = \neg \phi'$}
    \State \Return $\neg \evalR{\phi'}(x_1, \dots, x_n, y)$
\ElsIf{$\phi = \phi' \vee \psi$}
    \State \Return $\evalR{\phi'}(x_1, \dots, x_n, y) \vee \evalR{\psi}(x_1, \dots, x_n, y)$ \Comment{Recursively evaluate on subformulas}
\ElsIf{$\phi = \phi' \wedge \psi$}
    \State \Return $\evalR{\phi'}(x_1, \dots, x_n, y) \wedge \evalR{\psi}(x_1, \dots, x_n, y)$
\ElsIf{$\funcIt{isAtomic}(\phi)$}
    \State \Return $\phi(x_1,\dots, x_n, y)$ \Comment{Compute the atomic predicate $\phi$ value on $(x_1, \dots, x_n, y)$: $T$ or $F$}
\EndIf

\end{algorithmic}
\end{algorithm}

Algorithm~\ref{alg:errChecker} computes the error feedback $\ferr{}$ for a rejected output $y$. It traverses the quantifier-free formula $\loutspec{}$ in the same manner as $\eval$ (Algorithm~\ref{alg:evalChecker}) but, instead of returning a single boolean, collects all atomic formulas that evaluate to $F$ on the concrete values $(x_1, \dots, x_n, y)$. The returned set is used by the prompting program $\prompt$ to construct targeted feedback for the next sample.
\label{appen:errPseudocode}
\begin{algorithm}[H]
\footnotesize
\caption{$\ferr{}$ Error Feedback}
\label{alg:errChecker}
\begin{algorithmic}[1]
\State \textbf{Input:} Concrete values $(x_1,\dots,x_n,y)$
\State \textbf{Input:} Local contracts $(\linspec{},\loutspec{})$
\State \textbf{Output:} Set of failing atomic formulas

\State $\phi \gets \funcIt{Substitute}(\loutspec{},\{x_1,\dots,x_n,y\})$
\State \Return $\funcIt{CollectFailing}(\phi, x_1, \dots, x_n, y)$

\Statex
\Function{CollectFailing}{$\phi, x_1, \dots, x_n, y$}
\If{$\funcIt{isAtomic}(\phi)$}
    \If{$\phi(x_1,\dots,x_n,y) = F$}
        \State\Return $\{\phi\}$ \Comment{Atomic formula fails: return it}
    \Else
        \State\Return $\{\}$
    \EndIf
\ElsIf{$\phi = \neg \phi'$}
    \If{$\evalR{\phi'}(x_1,\dots,x_n,y) = T$}
        \State\Return $\{\phi\}$ \Comment{Negation fails: return the negated formula}
    \Else
        \State\Return $\{\}$
    \EndIf
\ElsIf{$\phi = \phi' \wedge \psi$}
    \State\Return $\funcIt{CollectFailing}(\phi', x_1, \dots, x_n, y) \cup \funcIt{CollectFailing}(\psi, x_1, \dots, x_n, y)$
\ElsIf{$\phi = \phi' \vee \psi$}
    \If{$\evalR{\phi'}(x_1,\dots,x_n,y) = F\ \wedge\ \evalR{\psi}(x_1,\dots,x_n,y) = F$}
        \State\Return $\funcIt{CollectFailing}(\phi', x_1, \dots, x_n, y) \cup \funcIt{CollectFailing}(\psi, x_1, \dots, x_n, y)$
    \Else
        \State\Return $\{\}$ \Comment{Disjunction holds: no failing formulas}
    \EndIf
\EndIf
\EndFunction
\end{algorithmic}
\end{algorithm}
\section{Background on GRPO}
\label{appen:GRPO}
\noindent\textbf{Group Relative Policy Optimization (GRPO): }
GRPO is a reinforcement learning–based fine-tuning method for generative models that optimizes a policy by comparing outputs \emph{relative to other samples from the same input}, rather than relying on an explicit value function. Let $\pi_{\theta}(y \mid p)$ denote the parametric policy for an FGGM $\funcIt{id_{\theta_i}}$, where $p \in \mathbb{P}$ is the input prompt and $y_l \sim \pi_{\theta}(\cdot \mid p)$ is a sampled output.
For each input $p$, GRPO samples a \emph{group} of $G$ candidate outputs $\{y^{(1)}, \dots, y^{(G)}\}$ from the current policy. Each sample is assigned a scalar reward $r^{(j)} = \reward_a(p, y^{(j)})$ (the per-sample augmented reward of Eq.~\ref{eq:decomposedLoss}, defined below). Instead of using absolute rewards directly, GRPO computes \emph{relative advantages} within the group:
\begin{small}
\begin{align}
A^{(j)} = \frac{r^{(j)} - \mu_p}{\sigma_p}, \quad
\mu_p = \tfrac{1}{G} \sum_{j=1}^G r^{(j)}, \quad
\sigma_p = \sqrt{\tfrac{1}{G} \sum_{j=1}^G (r^{(j)} - \mu_p)^2 + \xi}
\end{align}
\end{small}
Here $\mu_p$ and $\sigma_p$ are the mean and standard deviation of rewards within the group, and $\xi > 0$ is a small numerical-stability constant.
\begin{small}
\begin{align}
\max_{\theta} \;\; \mathbb{E}_{p \sim \mathbb{P}} \left[
\tfrac{1}{G} \sum_{j=1}^G 
\min\Big(
\rho^{(j)} A^{(j)}, \;
\text{clip}(\rho^{(j)}, 1-\epsilon_{\mathrm{clip}}, 1+\epsilon_{\mathrm{clip}}) A^{(j)}
\Big)
\right]
\end{align}
\end{small}
where $\rho^{(j)} = \frac{\pi_{\theta}(y^{(j)} \mid p)}{\pi_{\theta_{\text{old}}}(y^{(j)} \mid p)}$ is the importance ratio and $\epsilon_{\mathrm{clip}}$ is the PPO clip range. This clipped objective, similar to PPO, ensures stable policy updates.
In our setting, each FGGM $\funcIt{id_{\theta_i}}$ defines a policy $\pi_{\theta_i}(y_l \mid p)$ over outputs $y_l$ for input $p$, and the per-sample reward $r^{(j)}$ above is the augmented reward of Eq.~\ref{eq:decomposedLoss}, restated here:
\begin{small}
\begin{align}
\reward_a(p, y_l) = \log(\reward{}(x, \_, y)\times \mathbb{I}(y = y_l)+ \epsilon) + \log(\lambda \times \mathbb{I}(\eval(p, y_l)) + \epsilon)
\end{align}
\end{small}
\noindent This reward combines two signals. The task term $\reward{}(x, \_, y)\times\mathbb{I}(y = y_l)$ rewards outputs $y_l$ that determine the final agent output $y$ and achieve high task reward; for intermediate outputs ($y_l \neq y$) it vanishes, ensuring proper credit assignment. The conformance term $\mathbb{I}(\eval(p, y_l))$, scaled by $\lambda$, rewards outputs that satisfy the local contract $(\linspec{}, \loutspec{})$, thereby increasing rejection-sampling acceptance. The constant $\epsilon > 0$ is the same log-smoothing constant as in Eq.~\ref{eq:decomposedLoss}. Because GRPO only requires sampling from $\pi_{\theta_i}$ and scalar rewards, the non-differentiable checker $\eval$ poses no obstacle. Overall, GRPO optimizes each FGGM policy $\pi_{\theta_i}$ to increase the likelihood of outputs that (i) minimize downstream task loss when relevant and (ii) satisfy local contracts, aligning the learned generative behavior with the decomposed objective in Eq.~\ref{eq:decomposedLoss}.
\section{Hyperparameters}
\label{appen:hyperparams}

Table~\ref{tab:hyperparams} lists all hyperparameters used across the three LLM-based evaluation tasks.
The top block reports the shared pipeline settings applied uniformly across tasks; the middle blocks report task-specific training configurations for GSM-Symbolic and Symbolic Regression; the bottom block reports the per-task verifier and solver timeouts.

\begin{table}[h]
\centering
\small
\caption{Hyperparameter settings for all evaluation tasks.}
\label{tab:hyperparams}
\begin{tabular}{@{}llr@{}}
\toprule
\textbf{Category} & \textbf{Hyperparameter} & \textbf{Value} \\
\midrule
\multirow{3}{*}{Shared pipeline}
  & Planner search budget ($\Delta$)        & 10 \\
  & FGGM rejection-sample budget ($K$)      & 5 \\
  & In-context examples                     & 3 (2 for $\tau^2$-bench) \\
\midrule
\multirow{6}{*}{\shortstack[l]{GRPO / LoRA\\(GSM-Symbolic)}}
  & Base model                              & Qwen3-8B \\
  & LoRA rank                               & 16 \\
  & LoRA alpha ($\alpha$)                   & 32 \\
  & GRPO training epochs                    & 5 \\
  & GRPO batch size                         & 64 \\
  & GRPO learning rate                      & $1 \times 10^{-5}$ \\
  & GRPO generations per prompt ($G$)       & 8 \\
\midrule
\multirow{3}{*}{\shortstack[l]{Parameter tuning\\(Symbolic Regression)}}
  & Optimizer                               & Adam \\
  & Learning rate                           & 0.05 \\
  & Backpropagation steps                   & 40 \\
\midrule
\multirow{3}{*}{Verifier timeouts}
  & Dafny --- \texttt{dafny verify} timeout     & 240\,s per subprocess call \\
  & $\tau^2$-bench --- Z3 \texttt{solver.check()} timeout & 10\,s per call (up to 2 retries) \\
  & GSM-Symbolic --- Z3 \texttt{solver.check()} timeout & 5\,s per call \\
\bottomrule
\end{tabular}
\end{table}

\section{Inference Token Usage}
\label{appen:tokenUsage}

Table~\ref{tab:tokens} reports the average number of tokens consumed per instance during inference for \Tool{} and the baselines on each task. We report input (prompt) tokens, output (generated) tokens, and their total, averaged over the held-out test set. 

\begin{table}[h]
\centering
\small
\caption{Average inference token usage per instance (input / output / total), averaged over the test set.}
\label{tab:tokens}
\setlength{\tabcolsep}{6pt}
\begin{tabular}{@{}ll r r r@{}}
\toprule
\textbf{Task} & \textbf{Method}
  & \textbf{Input}~$\downarrow$
  & \textbf{Output}~$\downarrow$
  & \textbf{Total}~$\downarrow$ \\
\midrule
\multirow{3}{*}{HumanEvalDafny}
  & Baseline (Claude Sonnet 4.5) & 860 & 639 & 1,499 \\
  & DafnyBench                   & 2,254 & 1,836 & 4,090 \\
  & \Tool                        & 2,740 & 2,279 & 5,019 \\
\midrule
\multirow{3}{*}{DafnyBench}
  & Baseline (Claude Sonnet 4.5) & 925 & 836 & 1,761 \\
  & DafnyBench                   & 3,419 & 1,968 & 5,387 \\
  & \Tool                        & 4,123 & 2,302 & 6,425 \\
\midrule
\multirow{3}{*}{GSM-Symbolic}
  & Baseline (Qwen3-8B)          & 451 & 76 & 527 \\
  & CRANE                        & 1,261 & 152 & 1,413 \\
  & \Tool                        & 531 & 64 & 595 \\
\midrule
\multirow{3}{*}{$\tau^2$-Retail}
  & Baseline (Qwen3-8B)          & 396,752 & 1,801 & 398,553 \\
  & Agent-C                      & 512,376 & 3,465 & 515,841 \\
  & \Tool                        & 458,583 & 2,465 & 461,048 \\
\midrule
\multirow{3}{*}{$\tau^2$-Airline}
  & Baseline (Qwen3-8B)          & 160,256 & 825 & 161,081 \\
  & Agent-C                      & 278,741 & 1,322 & 280,063 \\
  & \Tool                        & 253,621 & 1,452 & 255,073 \\
\bottomrule
\end{tabular}
\end{table}

\section{Planner Feedback}
\label{appen:plannerFeedback}

After each search--verify--learn iteration, \Tool constructs structured feedback $I'$ from the execution traces of the current best candidate agent on the training data $\data$ (the \funcIt{collectFeedback} step of Algorithm~\ref{alg:tool}).
The planner LLM receives the task-performance score of the previous program along with per-example error diagnostics. Specifically, for each training example on which the previous candidate failed, \Tool provides the planner with the example input, the agent's output, and the corresponding error message (e.g., a verification error, a grammar violation, or an incorrect answer). The planner is then asked to generate a one-to-two sentence description and a suggested fix for each failed example. The resulting list of per-example descriptions and suggestions is collected and fed back to the planner as part of the context $I'$ when it samples the next candidate program. This feedback mechanism enables the planner to identify recurring failure patterns and adjust its synthesized FGGM definitions and prompting programs $\prompt$ accordingly in subsequent iterations.
\section{Planner LLM Prompt}
\label{appen:plannerPrompt}

The following is the prompt template used by the planner LLM to synthesize candidate agent programs. Placeholders in braces (e.g., \texttt{\{task\_description\}}) are populated at each iteration with the task specification, postcondition, available operator library, agent signature, and feedback from the previous iteration's execution traces (Appendix~\ref{appen:plannerFeedback}).

\begin{figure}[H]
\begin{tcolorbox}[
colback=lightgraybg,
colframe=lightgraybg,
boxrule=0pt,
arc=2pt,
left=4pt,right=4pt,top=4pt,bottom=4pt
]
\scriptsize
\begin{verbatim}
You are an expert programmer tasked with writing a Dafny agent function.

## TASK
{task_description}

## CRITICAL POSTCONDITION
{postcondition_description}

## AVAILABLE OPERATORS
You can use these operators in your agent:

```dafny
{library_functions}
```

## AGENT SIGNATURE AND POSTCONDITION
```dafny
{agent_signature}
```

## CONSTRAINTS
- Fully Typed: All variables must have type annotations
- No imports: All operators are already available in scope
- No recursion: Don't call agent() recursively
- Always satisfy postcondition: Every return path must satisfy the postcondition
- No f-strings: Do not use format strings (f-strings); use the addition operator for concatenation
{task_specific_constraints}

## FEW-SHOT EXAMPLES
{few_shot_examples}

## PREVIOUS AGENT CODE
{previous_agent_code}

## PREVIOUS ATTEMPT FEEDBACK
{feedback}

## SCORING INFORMATION
{scoring_info}

## OUTPUT
Write ONLY the agent function inside a ```dafny``` block.

Make sure to:
- Check the postcondition before every return
- If the task fails, return a safe fallback value
- LLM calls are stateless, calling the LLM again does not maintain previous context
- Predefine variables before using them to avoid syntax errors
{task_specific_output_tips}
\end{verbatim}
\end{tcolorbox}
\caption{Planner LLM prompt template. The planner receives this prompt at each search--verify--learn iteration, with placeholders filled from the task specification and feedback from the previous iteration.}
\label{fig:plannerPrompt}
\end{figure}

\section{Proof Details}
\label{appen:proofDetails}
\checkerComplete*
\begin{proof}[Proof of Lemma~\ref{lem:checkerComplete}]
We prove the claim by structural induction on the quantifier-free formula $\loutspec{}$.
Fix arbitrary $(x_1,\dots,x_n,y)$. We show:
\[
\eval(x_1,\dots,x_n,y) \iff \loutspec{}(x_1,\dots,x_n,y).
\]

\paragraph{Base case.}
Suppose $\loutspec{}$ is an atomic predicate $P(x_1, \dots, x_n, y)$ over terms constructed from computable library functions in $\fset{c}$. Under the substitution $(x_1,\dots,x_n,y)$, $\evalR{P}$ evaluates each term by executing the corresponding functions in $\fset{c}$. Since these functions are computable $\evalR{P}$ exactly computes the atomic predicate $P(x_1, \dots, x_n, y)$ on $(x_1, \dots, x_n, y)$ which reduces to $T$ or $F$.

\paragraph{Inductive step.}
Assume the claim holds for formulas $\phi$ and $\psi$ i.e. $\evalR{\phi}(x_1, \dots, x_n, y) \iff \phi(x_1, \dots, x_n, y)$ and $\evalR{\psi}(x_1, \dots, x_n, y) \iff \psi(x_1, \dots, x_n, y)$ We show it holds for compound formulas:

\begin{itemize}[leftmargin=*]
    \item If $\loutspec{} = \neg \phi$, then by definition of $\eval$,
    \[
    \eval(x_1,\dots,x_n,y) = \neg \evalR{\phi}(x_1,\dots,x_n,y).
    \]
    By the inductive hypothesis,
    \begin{align*}
    \eval(x_1,\dots,x_n,y) \iff   \neg\big( \evalR{\phi}(x_1,\dots,x_n,y)\big)\iff\\ \neg \big(\phi(x_1,\dots,x_n,y)\big) \iff \loutspec{}(x_1,\dots,x_n,y)
    \end{align*}

    \item If $\loutspec{} = \phi \wedge \psi$, then
    \[
    \eval(x_1,\dots,x_n,y)
    = \evalR{\phi}(x_1,\dots,x_n,y) \wedge \evalR{\psi}(x_1,\dots,x_n,y),
    \]
    which, by the inductive hypothesis,
    \begin{align*}
    \eval(x_1,\dots,x_n,y) \iff \big(\evalR{\phi}(x_1,\dots,x_n,y) \wedge \evalR{\psi}(x_1,\dots,x_n,y)\big) \\ \iff \big(\phi(x_1,\dots,x_n,y) \wedge \psi(x_1,\dots,x_n,y)\big) \iff \loutspec{}(x_1,\dots,x_n,y)    
    \end{align*}

    \item If $\loutspec{} = \phi \vee \psi$, the argument is analogous to $\loutspec{} = \phi \wedge \psi$.
\end{itemize}
\end{proof}
\localContract*
\begin{proof}
Let $r = \funcIt{id_{\fggm{}}}(x_1, \dots, x_n)$. Based on the rejection sampler with the checker $\eval$ (Fig~\ref{fig:rejectionSampler}). $r$ is returned from the if branch $\kw{if}\ 
\funcIt{\eval}(x_1,\dots,x_n,y)\{\kw{return}\ y; \}$ or from the fallback $r = \fallback(x_1,\dots,x_n,y)$. Using the semantics of an if block, the following condition always holds for the output $r$ where $y_f$ is the final rejected sample.
\begin{small}
\begin{align}
&\axiomEncode{\axiom{}}\implies (\forall \theta \in \Theta. (\forall x_i\in T_i). \linspec{}(x_1, \dots,x_n) \implies(\eval(x_1,\dots,x_n, r) \vee (r = \fallback(x_1,\dots,x_n, y_{f})))) \label{appeneq:FGGMsound1}
\end{align}
\end{small}

 Eq.~\ref{appeneq:FGGMsound2} follows from the validity of the fallback $\fallback$, and Eq.~\ref{appeneq:FGGMsound3} follows from the completeness of the checker $\eval$ (Lemma~\ref{lem:checkerComplete}). To simplify the notation, we used $(\forall x_i\in T_i)$ to denote $(\forall x_1 \in T_1, \dots, \forall x_n \in T_n)$.
\begin{small}
\begin{align}
&\axiomEncode{\axiom{}}\implies\forall \theta \in \Theta. (\forall x_i\in T_i). \linspec{}(x_1, \dots,x_n) \wedge (r = \fallback(x_1,\dots,x_n, y_f)) \implies \loutspec{}(x_1,\dots, x_n, r) \label{appeneq:FGGMsound2}\\
&\axiomEncode{\axiom{}}\implies(\forall \theta \in \Theta. (\forall x_i\in T_i). (\linspec{}(x_1, \dots,x_n) \wedge \eval(x_1,\dots,x_n, r)) \implies \loutspec{}(x_1,\dots, x_n, r)) \label{appeneq:FGGMsound3}\\
&\axiomEncode{\axiom{}}\implies(\forall \theta \in \Theta. (\forall x_i\in T_i). \linspec{}(x_1, \dots,x_n) \implies \loutspec{}(x_1, \dots, x_n, r))\;\;\;\; \text{Using Eqs.~(\ref{appeneq:FGGMsound1}, \ref{appeneq:FGGMsound2}, \ref{appeneq:FGGMsound3})} \nonumber
\end{align}    
\end{small}
\end{proof}

\searchVerify*
\begin{proof}
Let, $\fggmSet{} = \set{\fggm{}_1, \dots, \fggm{}_k}$ and $\axiom{}_{\fggmSet{}}$ denote the set of first-order sentences defining all the local contracts as shown below.
\begin{align}
&\axiom{}_{\fggmSet{}} = \set{\varphi_{\fggm{}_1}, \dots, \varphi_{\fggm{}_k}} \;\text{where $\varphi_{\fggm{}_i}$ defined below} \nonumber \\
&\varphi_{\fggm{}_{i}}: \forall x_1\in T_1,\dots,x_n\in T_n. \fggm{}_{i}^{\linspec{}}(x_1, \dots, x_n) \implies \fggm{}_{i}^{\loutspec{}}(x_1, \dots, x_n, \fggm{}_{i}(x_1, \dots, x_n)) \nonumber\\
& (\axiomEncode{\axiom{}} \implies (\forall \theta_1 \in \Theta_{1}.\varphi_{\fggm{}_{1}})) \wedge \cdots \wedge (\axiomEncode{\axiom{}} \implies (\forall \theta_k \in \Theta_{k}.\varphi_{\fggm{}_{k}}))\ \;\;\text{from Theorem~\ref{thm:FGGMsoundness}} \nonumber \\
& \big(\axiomEncode{\axiom{}} \implies \wedge_{i=1}^{k}(\forall \theta_i \in \Theta_{i}.\varphi_{\fggm{}_{i}})\big)\label{eq:appenParamIndepen}
\end{align}
If $\program{} \neq \bot$ then $\ver{\inspec{}}{\outspec{}}$ with $\context{} = (\dg, \fset{}\cup\fggmSet{}, \axiom{} \cup \axiom{}_{\fggmSet{}})$ ensures 
\begin{align}
(\program{} \neq \bot) \implies ((\axiomEncode{\axiom{}} \wedge (\wedge_{i=1}^{k}\forall \theta_i \in \Theta_{i}.\varphi_{\fggm{}_{i}})\big) \implies \nonumber \\ \forall x_1\in T_1\dots\forall x_n \in T_n. \inspec{}(x_1, \dots, x_n) \implies \outspec{}(x_1,\dots, x_n, \programSub{}(x_1, \dots, x_n)) \label{eq:tu}\\
(\program{} \neq \bot) \implies ((\axiomEncode{\axiom{}}  \implies \nonumber\\ \forall \theta_1 \in \Theta_1,\dots, \forall \theta_k \in \Theta_k, \forall x_1\in T_1\dots\forall x_n \in T_n. \nonumber\\
\inspec{}(x_1, \dots, x_n) \implies \outspec{}(x_1,\dots, x_n, \programSub{}(x_1, \dots, x_n)) \;\text{Using Eq.~\ref{eq:appenParamIndepen} and Eq.~\ref{eq:tu}} \nonumber
\end{align}
\end{proof}

\soundnessThm*
\begin{proof}
$\opt{f} \gets \arg\max_{f \in \mathcal{P}}
\sum_{(x,y)\in\data}\reward{}(x,y,f(x))$. If $\opt{f} \neq \bot$, then by definition $(\mathcal{P} \neq \set{}) \wedge(\opt{f} \in \mathcal{P})$. By Theorem~\ref{thm:verificationParamSet}, every program $f \in \mathcal{P}$ that satisfies $\axiomEncode{\axiom{}}$ also satisfies
\begin{align*}
\forall x_1 \in T_1, \dots, \forall x_n \in T_n,\;
\inspec{}(x_1, \dots, x_n) \implies \outspec{}(x_1, \dots, x_n, f(x_1, \dots, x_n)).
\end{align*}
In particular, since $\opt{f} \in \mathcal{P}$, instantiating the above with $f = \opt{f}$ yields $\axiomEncode{\axiom{}} \implies \big(\forall x_1 \in T_1, \dots, \forall x_n \in T_n.\; \inspec{}(x_1, \dots, x_n) \implies \outspec{}(x_1, \dots, x_n, \opt{f}(x_1, \dots, x_n))\big)$, which is the claim.
\end{proof}

\sufficientCond*
\begin{proof}
Let $f_n^{\theta} : T_1 \times \cdots \times T_n \to T_o$ be any type-correct generative model. By assumption (iii), there exists $\fallback \in \fspace{G}{\fset{c}}$ such that
\[
\axiomEncode{\axiom{}} \implies \forall x_1,\dots,x_n.\;
\inspec{}(x_1,\dots,x_n) \implies \outspec{}(x_1,\dots,x_n, \fallback(x_1,\dots,x_n)).
\]

\noindent\textbf{Construction.}
Define
\[
\fggm{} = (\funcIt{id}, f_n^{\theta}, \typs{}, \linspec{}, \loutspec{}, \prompt, \fallback, \fpost),
\]
with $(\linspec{}, \loutspec{}) := (\inspec{}, \outspec{})$ and any well-typed $\prompt$. Since $\fallback$ satisfies the contract and $\outspec{}$ is quantifier-free, $\fggm{}$ is valid and $\evalb$ is complete (Lemma~\ref{lem:checkerComplete}).

Define the program $f \in \fspace{G}{\fset{}}$ as:
\[
\kw{function}\ \funcIt{f}(x_1: T_1, \dots, x_n: T_n): T_o\ { \set{\kw{return}\ \funcIt{id}(x_1, \dots, x_n);}}.
\]

\noindent\textbf{Correctness.}
For any $(x_1,\dots,x_n)$ such that $\inspec{}(x_1,\dots,x_n)$ holds, by Theorem~\ref{thm:FGGMsoundness},
\[
\outspec{}(x_1,\dots,x_n, f(x_1,\dots,x_n)).
\]

\noindent\textbf{Output of $f$.}
Fix $(x_1,\dots,x_n)$ with $\inspec{}(x_1,\dots,x_n)$. By completeness,
\[
\evalb(x_1,\dots,x_n, y) \iff \outspec{}(x_1,\dots,x_n, y).
\]
Hence,
\[
f(x_1,\dots,x_n) =
\begin{cases}
f_n^{\theta}(x_1,\dots,x_n) & \text{if } \outspec{}(x_1,\dots,x_n, f_n^{\theta}(x_1,\dots,x_n)), \\
\fallback(x_1,\dots,x_n) & \text{otherwise.}
\end{cases}
\]

\noindent\textbf{Reward comparison.}
For any $(x,\_) \in \data$:

\emph{Case 1:} $\outspec{}(x, f_n^{\theta}(x))$. Then $f(x) = f_n^{\theta}(x)$ and
\[
\reward{}(x,\_, f(x)) = \reward{}(x,\_, f_n^{\theta}(x)).
\]

\emph{Case 2:} $\neg \outspec{}(x, f_n^{\theta}(x))$. Then $f(x) = \fallback(x)$ and $\outspec{}(x, f(x))$ holds. By assumption (i) (well-formedness of $\reward{}$, Def.~\ref{def:wellFormedReward}),
\[
\reward{}(x,\_, f(x)) > \reward{}(x,\_, f_n^{\theta}(x)).
\]

Thus,
\[
\forall (x,\_) \in \data.\quad \reward{}(x,\_, f(x)) \geq \reward{}(x,\_, f_n^{\theta}(x)),
\]
which implies
\[
R(f) \geq R(f_n^{\theta}).
\]

\noindent\textbf{Strict improvement.}
If $\exists (x,\_) \in \data$ such that $\neg \outspec{}(x, f_n^{\theta}(x))$, then Case 2 holds for at least one point, yielding
\[
R(f) > R(f_n^{\theta}).
\]

\noindent\textbf{Conclusion.}
The constructed program $f$ satisfies $(\inspec{}, \outspec{})$ and achieves no less reward than $f_n^{\theta}$, with strict improvement when violations occur. Note the program verification setup satisfies this sufficient condition.
\end{proof}

\rewardLowerBound*
\begin{proof}
Since $\outspec{}$ is quantifier-free if any of $K$ output of $\fggm{}$ say $y_l$ satisfies $\evalb{}$ then it is the final output, $\mathbb{I}(y = y_l) = 1$ for any $(x, \_) \in \data$. $\epsilon$ is arbitrarily small, and we drop it from the calculation. Thus
\begin{align*} 
    &\log(\lambda \times \reward{}(x, \_, f[\theta_t/ \Theta](x)) \geq \reward{}_a^{\theta_t}(p_x)\;\;\textit{Ignoring linear and higer order $\epsilon$-terms} \\
    &\sum_{(x, \_) \in \data}\log(\lambda \times \reward{}(x, \_, f[\theta_t/ \Theta](x)) \geq |D|\times \opt{\reward{}_a}(t)\\
    &\log(\lambda^{|D|} \times \Pi_{(x, \_) \in \data}\reward{}(x, \_, f[\theta_t/ \Theta](x)) \geq |D|\times \opt{\reward{}_a}(t) \\
    &\Pi_{(x, \_) \in \data}\reward{}(x, \_, f[\theta_t/ \Theta](x) \geq \frac{\exp{}(|D|\times \opt{\reward{}_a}(t))}{\lambda^{|D|}} \\
    &\frac{1}{|D|}\sum_{(x,\_) \in D}\reward{}(x, \_, f[\theta_t/ \Theta](x) \geq \frac{\exp{}( \opt{\reward{}_a}(t))}{\lambda} \;\;\text{using AM-GM inequality}
\end{align*}
\end{proof}

\section{FGGM Expressivity Preservation}
\label{appen:expressivity}

A natural concern with any post-hoc constraint enforcement mechanism is whether it implicitly limits the computational power of the underlying model. In particular, naive constrained decoding methods that restrict every generated token to conform to a finite output grammar can collapse the effective expressivity of constant-depth transformers to the circuit complexity class DLOGTIME-uniform $TC^0$ (constant-depth, polynomial-size threshold circuits) \cite{crane}. Since $TC^0$ is believed to be strictly weaker than $NL$ (nondeterministic log-space), such approaches provably cannot solve problems like graph reachability under standard complexity-theoretic assumptions. We show that the FGGM construction avoids this expressivity loss entirely.

\noindent\textbf{Expressivity of transformers with chain-of-thought.}
A series of works \cite{cotExpressivity, cotExpressivity2} establish that constant-depth, log-precision transformers without intermediate reasoning steps (i.e., producing a single autoregressive output) can be simulated by logspace-uniform $TC^0$ circuits. However, when allowed to produce $t(n)$ intermediate autoregressive steps, where $t(n)$ is polynomial in the input size $n$, the same class of transformers can simulate any Turing machine $M$ that halts within $t(n)$ steps. The construction from \cite{cotExpressivity} works as follows: each autoregressive step encodes one configuration of $M$, capturing the current state $q \in Q$, the symbols at each tape head, and the head movement directions. Since the set of states $Q$, the tape alphabet $\Gamma$, and the number of tapes are all constants, the total number of distinct configuration encodings is also constant, and hence representable by a finite vocabulary $V_M$. The LLM $\llm{M}$ then outputs the sequence of configuration encodings $\overline{\gamma}_1 \cdots \overline{\gamma}_{t(n)}$ followed by the final output $M(\vect{x})$, effectively unrolling the entire computation of $M$ across autoregressive steps. We restate this result as Lemma~\ref{lem:cotExpressivity}.

\begin{lemma}[Expressivity of CoT {\cite{cotExpressivity}}]
\label{lem:cotExpressivity}
For any Turing machine $M$ with tape alphabet $\Gamma$, there exists a constant-depth LLM $\llm{M}$ with finite vocabulary $V_M \supseteq \Gamma$ and log-precision that can simulate $t(n)$ steps of $M$ using $t(n)$ autoregressive steps, where $t(n)$ is polynomial in the input size $n$. For any input $\vect{x}$ with $|\vect{x}| = n$ on which $M$ halts within $t(n)$ steps, the unconstrained output takes the form $\llm{M}(\vect{x}) = \overline{\gamma}_1 \cdots \overline{\gamma}_{t(n)} \cdot M(\vect{x})$ where $\overline{\gamma}_1, \dots, \overline{\gamma}_{t(n)} \in V_M$ are encodings of the machine's configurations at each computational step.
\end{lemma}

\noindent\textbf{Why FGGM preserves expressivity.}
The key observation is that the underlying GM in an FGGM generates its output unconstrained — no token-level grammar mask is applied and no bound on generation length is imposed during the generation process. This allows the GM to use arbitrarily many intermediate chain-of-thought (CoT) reasoning steps. The post-processor $\fpost{}$ then extracts the final answer from the raw GM output using the built-in $\funcIt{split}$ function (\secRef~\ref{sec:searchSpace}) with an appropriate delimiter character $c$, and the rejection sampler checks only the extracted answer against the local contract $\loutspec{}$. Since the reasoning tokens are stripped before any constraint checking, they place no restriction on the model's intermediate computation. Consequently, the FGGM retains the full CoT expressivity established by Lemma~\ref{lem:cotExpressivity}. We formalize this as Theorem~\ref{thm:expressivityFGGM}.

\expressivityFGGM*
\begin{proof}
Let $M$ be an arbitrary Turing machine with tape alphabet $\Gamma$ that halts on all inputs satisfying $\linspec{}$ within $t(n)$ steps, where $t(n)$ is polynomial in $n = |\vect{x}|$.

\noindent\textbf{Step 1: Constructing the LLM.}
By Lemma~\ref{lem:cotExpressivity}, there exists a constant-depth LLM $\llm{M}$ with vocabulary $V_M \supseteq \Gamma$ and log-precision such that on any input $\vect{x}$, $\llm{M}(\vect{x}) = \overline{\gamma}_1 \cdots \overline{\gamma}_{t(n)} \cdot M(\vect{x})$.
Choose a fresh delimiter symbol $c \notin V_M$ and extend the vocabulary to $V_M' = V_M \cup \{c\}$. We modify $\llm{M}$ to produce $\llm{M}'(\vect{x}) = \overline{\gamma}_1 \cdots \overline{\gamma}_{t(n)} \cdot c \cdot M(\vect{x})$.
This requires one additional autoregressive step to output $c$ before the final answer, preserving the constant-depth property. Since all configuration encodings $\overline{\gamma}_i \in V_M$ and $c \notin V_M$, the delimiter $c$ does not appear in the reasoning prefix.

\noindent\textbf{Step 2: Defining the post-processor.}
Define $\fpost{}(y) = \funcIt{split}(c, y)[|\funcIt{split}(c, y)| - 1]$, which returns the substring after the last occurrence of $c$. Since $c$ appears exactly once in $\llm{M}'(\vect{x})$ (separating the reasoning from the final answer), $\fpost{}(\llm{M}'(\vect{x})) = M(\vect{x})$.

\noindent\textbf{Step 3: FGGM acceptance.}
Consider the FGGM $\fggm{}$ with local contract $(\linspec{}, \loutspec{})$, underlying GM $\llm{M}'$, post-processor $\fpost{}$, and any valid fallback $\fallback$. For any input $(x_1, \dots, x_n)$ satisfying $\linspec{}$, the rejection sampler (Fig.~\ref{fig:rejectionSampler}) receives the processed output $y = \fpost{}(\llm{M}'(\prompt(x_1, \dots, x_n))) = M(\vect{x})$.
By assumption, $\loutspec{}(x_1, \dots, x_n, M(\vect{x}))$ holds. Since $\loutspec{}$ is quantifier-free, by completeness of $\eval$ (Lemma~\ref{lem:checkerComplete}), $\eval(x_1, \dots, x_n, M(\vect{x}))$ returns true. Therefore, the rejection sampler accepts $M(\vect{x})$ on the first sample, the fallback is never triggered, and the FGGM returns $M(\vect{x})$.
\end{proof}

\end{document}